\documentclass[12pt]{article} 

\usepackage{smile}
\usepackage{graphicx}
\usepackage{booktabs} 
\usepackage{natbib}
\usepackage{tikz}

\usepackage{wrapfig}
\usepackage{lineno}
\usepackage{caption}
\usepackage{subcaption}
\usepackage{acronym}
\acrodef{icl}[ICL]{In-Context Learning}
\acrodef{gpt}[GPT]{Generative Pre-trained Transformer}
\acrodef{mle}[MLE]{Maximum Likelihood Estimate}
\acrodef{ff}[FF]{Feed-Forward}
\acrodef{mha}[MHA]{Multi-Head Attention}
\acrodef{mae}[MAE]{Masked AutoEncoders}
\acrodef{nlp}[NLP]{Natural Language Processing}
\acrodef{cv}[CV]{Computer Vision}
\acrodef{llm}[LLM]{Large Language Models}
\acrodef{hmm}[HMM]{Hiddn Markov Model}
\acrodef{rkhs}[RKHS]{Reproducing Kernel Hilbert Space}
\acrodef{bma}[BMA]{Bayesian Model Averaging}
\usepackage{geometry}

\usepackage[colorlinks,
linkcolor=red,
anchorcolor=blue,
citecolor=blue
]{hyperref}

\def\[#1\]{\begin{bmatrix}#1\end{bmatrix}}

\usepackage{hyperref}
\usepackage{url}

\def\cme{{\mathtt{CME}}}
\def\Tr{\operatorname{tr}}

\def\att{\mathtt{attn}}
\def\softmax{\mathtt{softmax}}

\def\noend{\notag \\}

\def\fk{\mathfrak{K}}
\def\fl{\mathfrak{L}}

\def\sm{{\mathtt{SM}}}

\def\tp{{\mathrm{p}}}

\def\relu{{\mathtt{ReLU}}}

\def\fX{{\mathfrak{X}}}

\def\regret{{\mathtt{regret}}}
\def\icl{{\mathtt{ICL}}}

\newcommand{\FF}{\mathtt{ffn}}
\newcommand{\MHA}{\mathtt{mha}}
\newcommand{\Att}{\mathtt{attn}}
\newcommand{\TV}{\mathtt{TV}}
\newcommand{\KL}{\mathtt{KL}}
\newcommand{\unif}{\mathtt{Unif}}
\newcommand{\pt}{\mathtt{prompt}}

\newcommand{\bbB}{\mathbb{B}}

\newcommand{\bbE}{\mathbb{E}}

\newcommand{\bbI}{\mathbb{I}}

\newcommand{\bbN}{\mathbb{N}}

\newcommand{\bbP}{\mathbb{P}}

\newcommand{\bbR}{\mathbb{R}}
\newcommand{\bbS}{\mathbb{S}}

\newcommand{\rmd}{\mathrm{d}}

\newcommand{\rmF}{\mathrm{F}}

\newcommand{\rmI}{\mathrm{I}}

\newcommand{\rmp}{\mathrm{p}}

\newcommand{\calA}{\mathcal{A}}
\newcommand{\calB}{\mathcal{B}}

\newcommand{\calD}{\mathcal{D}}

\newcommand{\calF}{\mathcal{F}}
\newcommand{\calG}{\mathcal{G}}

\newcommand{\calL}{\mathcal{L}}
\newcommand{\hcalL}{\hat{\mathcal{L}}}

\newcommand{\calS}{\mathcal{S}}

\newcommand{\calX}{\mathcal{X}}

\newcommand{\calZ}{\mathcal{Z}}

\newcommand{\hatA}{\hat{A}}

\newcommand{\tilA}{\tilde{A}}

\newcommand{\hatC}{\hat{C}}
\newcommand{\tilc}{\tilde{c}}

\newcommand{\tilD}{\tilde{D}}

\newcommand{\tilS}{\tilde{S}}

\newcommand{\hatW}{\hat{W}}

\newcommand{\tilW}{\tilde{W}}

\newcommand{\tilx}{\tilde{x}}
\newcommand{\tilX}{\tilde{X}}

\newcommand{\tilY}{\tilde{Y}}

\newcommand{\barz}{\bar{z}}

\newcommand{\barB}{\bar{B}}

\newcommand{\barD}{\bar{D}}

\newcommand{\tilgamma}{\tilde{\gamma}}

\newcommand{\htheta}{\hat{\theta}}
\newcommand{\tiltheta}{\tilde{\theta}}

\newcommand{\frakK}{\mathfrak{K}}

\newcommand{\frakQ}{\mathfrak{Q}}

\newcommand{\frakV}{\mathfrak{V}}

\newcommand{\frakX}{\mathfrak{X}}

\newcommand{\frakZ}{\mathfrak{Z}}

\title{What and How does In-Context Learning Learn? Bayesian Model Averaging, Parameterization, and Generalization}

\author{\small Yufeng Zhang\thanks{equal contribution} \thanks{Northwestern University; \texttt{yufengzhang2023@u.northwestern.edu}},~~Fengzhuo Zhang\footnotemark[1] \thanks{National University of Singapore; \texttt{fzzhang@u.nus.edu}},~~Zhuoran Yang\thanks{Yale University; \texttt{zhuoranyang.work@gmail.com}},~~Zhaoran Wang\thanks{Northwestern University; \texttt{zhaoranwang@gmail.com}}}
\date{}
\begin{document}

\maketitle

\begin{abstract}
In this paper, we conduct a comprehensive study of In-Context Learning (ICL) by addressing several open questions: (a) What type of ICL estimator is learned by large language models? (b) What is a proper performance metric for ICL and what is the error rate? (c) How does the transformer architecture enable ICL? 
To answer these questions, we adopt a Bayesian view and formulate ICL as a  problem of predicting the response corresponding to the current covariate, given a number of examples drawn from a latent variable model.    
To answer (a), we show that, without updating the neural network parameters, ICL implicitly implements the Bayesian model averaging algorithm, which is proven to be approximately parameterized by the attention mechanism. 
For (b), we analyze the ICL performance from an online learning perspective and establish a $\mathcal{O}(1/T)$ regret bound for perfectly pretrained ICL, where $T$ is the number of examples in the prompt.
To answer (c),  we show that, in addition to encoding  Bayesian model averaging via attention, the transformer architecture also enables 
a fine-grained statistical analysis of pretraining under realistic assumptions. 
In particular, we prove that the error of pretrained model is bounded by a sum of an approximation error and a generalization error, where the former decays to zero exponentially as the depth grows, and the latter decays to zero sublinearly with the number of tokens in the pretraining dataset. 
Our results provide a unified understanding of the transformer and its ICL ability with bounds on ICL regret, approximation, and generalization, which deepens our knowledge of these essential aspects of modern language models.

\end{abstract}

\section{Introduction}
    

With the ever-increasing sizes of model capacity and corpus, \ac{llm} have achieved tremendous successes across a wide range of tasks, including natural language understanding~\citep{dong2019unified,jiao2019tinybert}, symbolic reasoning~\citep{wei2022chain,kojima2022large}, and conversations~\citep{brown2020language,ouyang2022training}. Recent studies have revealed that these \ac{llm}s possess immense potential, as their large capacity allows for a series of \emph{emergent abilities}~\citep{wei2022emergent,liu2023pre}. 
One such ability is \ac{icl}, which enables an \ac{llm} to learn from just a few examples, without changing the network parameters. 
That is, after seeing a few examples in the prompt,  a pretrained language model seems to comprehend the underlying concept and is able to extrapolate the understanding to new data points.


Despite the tremendous empirical successes,  theoretical understanding of \ac{icl} remains limited. Specifically,
existing works fail to explain why \ac{llm}s  the ability for \ac{icl}, how the attention mechanism is related to the \ac{icl} ability, and how pretraining influences \ac{icl}. 
Although the optimality of \ac{icl} is investigated in \cite{xie2021explanation} and \cite{wies2023learnability}, these works both make unrealistic assumptions on the pretrained models, and their results cannot demystify the particular role played by the attention mechanism in \ac{icl}.

In this work, we focus on the scenario where a  transformer is first pretrained on a large dataset and then prompted to perform \ac{icl}. 
Our goal is to rigorously understand 
why the practice of ``pretraining + prompting'' unleashes the power of \ac{icl}.
To this end, we aim to answer the following three questions:
{\color{blue} \bf (a)} What type of \ac{icl} estimator is learned by \ac{llm}s? 
{\color{blue} \bf (b)} What are suitable performance metrics to evaluate \ac{icl} accurately and what are the error rates? 
{\color{blue} \bf (c)} What is the role played by the transformer architecture during the pretraining and prompting stages? The first and the third questions demand 
scrutinizing the transformer architecture to understand how \ac{icl} happens during transformer prompting. 
The second question then requires statistically analyzing the extracted \ac{icl} process. 
Moreover, the third question necessitates a holistic understanding beyond prompting --- we also need to characterize the statistical error of pretraining and how this error affects prompting.  


To address these questions, we adopt a Bayesian view and 
assume that the examples fed into a pretrained \ac{llm} are sampled from a latent variable model parameterized by a hidden concept 
$z_{*}\in\frakZ$.
Moreover, the pretrained dataset contains sequences of examples from the same latent variable model, but with the concept parameter $z \in\frakZ$ itself randomly distributed according to a prior distribution. 
We mathematically formulate \ac{icl} as the  problem of predicting the response of the response corresponding to the current covariate,
where the prompt contains $t$ examples of covariate-response pairs and the current covariate. 

Under such a setting, 
to answer {\bf \color{blue}{(a)}}, we show that the perfectly pretrained \ac{llm}s perform \ac{icl} in the form of \ac{bma}. 
That is,   \ac{llm} first computes a posterior distribution of $z_* \in\frakZ$ given the first $t$ examples, and then predicts the response of the $(t+1)$-th covariate by aggregating over the posterior (Proposition \ref{th:bma}). 

In addition, to answer {\bf \color{blue}{(b)}}, we adopt the online learning framework and define a notion called \ac{icl} regret, which is the averaged prediction error of \ac{icl} on a sequence of covariate-response examples. 
We prove that the \ac{icl} regret after prompting $t$ examples is $\cO(1/t)$ up to the statistical error of the pretrained model (Theorem \ref{th:iclpretrain}).

Finally, to answer {\bf \color{blue}{(c)}}, we elucidate the role played by the transformer architecture in prompting and pretraining respectively. 
In particular, we show that a variant of attention mechanism encodes \ac{bma} in its architecture, which enables the transformer to perform \ac{icl} via prompting. 
Such an attention mechanism can be viewed as an extension of linear attention and coincides with the standard softmax attention \citep{garnelo2023exploring} when the length of the prompt goes to infinity. And thus we show that softmax attention \cite{vaswani2017attention} approximately encodes \ac{bma} (Proposition~\ref{prop:cme-att-limit}). 
Besides,  the transformer architecture enables a fine-grained analysis of the statistical error incurred by pretraining. In particular, applying the  PAC-Bayes framework, we prove that the error of the pretrained language model, measured via total variation, is bounded by a sum of approximation error and generalization error (Theorem \ref{thm:pretrainmle}). 
The approximation error decays to zero exponentially fast as the depth of the transformer increases (Proposition \ref{prop:approxerrmle}), while the generalization error decays to zero sublinearly with the number of tokens in the pretraining dataset. 
This features the first pretraining analysis of transformers in total variation distance, which also takes the approximation error into account. 
Furthermore, as an interesting extension, we also study the misspecified case where the response variables of the examples fed into the \ac{llm} are perturbed. 
We provide sufficient conditions for \ac{icl} to be robust to the perturbations and establish the finite-sample statistical error (Proposition \ref{prop:wronglabel}).

In sum, by addressing questions {\bf \color{blue} (a)}--{\bf \color{blue}(c)}, we provide a unified understanding of the  \ac{icl} ability of \ac{llm}s and the particular role played by the attention mechanism. 
Our theory provides a holistic theoretical understanding of the regret,  approximation, and generalization errors of \ac{icl}.

\section{Related Work}
\textbf{In-Context Learning.} After \cite{brown2020language} showcased the in-context learning (\ac{icl}) capacity of GPT-3, there has been a notable surge in interest towards enhancing and comprehending this particular ability~\citep{dong2022survey}. 
The \ac{icl} ability has seen enhancements through the incorporation of extra training stages~\citep{min2021metaicl,wei2021finetuned,iyer2022opt}, carefully selecting and arranging informative demonstrations~\citep{liu2021makes,kim2022self,rubin2021learning,lu2021fantastically}, giving explicit instructions~\citep{honovich2022instruction,zhou2022large,wang2022self}, and prompting a chain of thoughts~\citep{wei2022chain,zhang2022automatic,zhou2022least}. In efforts to comprehend the mechanisms of \ac{icl} ability, researchers have also conducted extensive work. 
Empirically, \cite{chan2022data} demonstrated that the distributional properties, including the long-tailedness, are important for \ac{icl}. \cite{garg2022can} investigated the function class that \ac{icl} can approximate. \cite{min2022rethinking} showed that providing wrong mappings between the input-output pairs in examples does not degrade the \ac{icl}. 
Theoretically, \cite{akyurek2022learning}, \cite{von2022transformers}, \cite{bai2023transformers}, and \cite{dai2022can}  indicated that \ac{icl} implicitly implements the gradient descent or least-square algorithms from the function approximation perspective. 
However, the first three works only showed that transformers are able to approximate these two algorithms, which may not align with the pretrained model. 
The last work ignored the softmax module, which turns out to be important in practical implementation. \cite{feng2023towards} derived the impossibility results of \ac{icl} and the advantage of chain-of-thought for the function approximation. \cite{litransformers} viewed \ac{icl} from the multi-task learning perspective and derived the generalization bound. 
\cite{hahn2023theory} built the linguistic model for sentences and used the description length to bound the \ac{icl} error with this model. 
\cite{xie2021explanation} analyzed \ac{icl} within the Bayesian framework, assuming the access to the nominal language distribution and that the tokens are generated from \ac{hmm}s. 
However, the first assumption hides the relationship between pretraining and \ac{icl}, and the second assumption is restrictive. 
Following this thread, \cite{wies2023learnability} relaxed the \ac{hmm} assumption and assumed access to a pretrained model that is close to the nominal distribution conditioned on any token sequence, which is also unrealistic. 
Two recent works \cite{wang2023large}, and \cite{jiang2023latent} also provide the Bayesian analysis of \ac{icl}. Unfortunately, these Bayesian works cannot explain the importance of the attention mechanism for \ac{icl} and clarify how pretraining is related to \ac{icl}. 
In contrast, we prove that the attention mechanism enables \ac{bma} by encoding it in the network architecture and we relate the pretraining error of transformers to the \ac{icl} regret.

\vspace{-2mm}


\section{Preliminary}\label{sec:prelim}
\vspace{-2mm}
\textbf{Notation.} We denote $\{1,\cdots,N\}$ as $[N]$. For a Polish space $\calS$, we denote the collection of all the probability measures on it as $\Delta(\calS)$. The total variation distance  between two distributions $P,Q\in\Delta(\calS)$ is $\TV(P,Q)=\sup_{A\subseteq\calS}|P(A)-Q(A)|$. The $i^{\rm{th}}$ entry of a vector $x$ is denoted as $x_{i}$ or $[x]_{i}$. For a matrix $X\in\bbR^{T\times d}$, we index its $i^{\rm th}$ row and column as $X_{i,:}$ and $X_{:,i}$ respectively. The $\ell_{p,q}$ norm of $X$ is defined as $\|X\|_{p,q}= (\sum_{i=1}^{d}\|X_{:,i}\|_{p}^{q})^{1/q}$, and the {\em Frobenius norm} of it is defined as $\|X\|_{\rmF}= \|X\|_{2,2}$.

\vspace{-1mm}
\noindent{\bf Attention and Transformers.} Attention mechanism has been the most powerful and popular neural network module in both \ac{cv} and \ac{nlp} communities, and it is the backbone of the \ac{llm}s~\citep{devlin2018bert,brown2020language}. Assume that we have a query vector $q\in \bbR^{d_{k}}$. With $T$ key vectors in $K\in\bbR^{T\times d_{k}}$ and $T$ value vectors in $V\in\bbR^{T\times d_{v}}$, the attention mechanism maps the query vector $q$ to $\att(q, K, V) = V^\top\softmax(Kq)$, where $\softmax$ normalizes a vector via the exponential function, i.e., for $x\in\bbR^{d}$, $[\softmax(x)]_{i}=\exp(x_{i})/\sum_{j=1}^{d}\exp(x_{j})$ for $i\in[d]$. 
The output is a weighted sum of $V$, and the weights reflect the closeness between $W$ and $q$. For $t$ query vectors, we stack them into $Q\in\bbR^{t\times d_{k}}$. Attention maps these queries using the function $\Att(Q,K,V)=\softmax(QK^\top)V \in \bbR^{t\times d_{v}}$, where $\softmax$ is applied row-wisely.
In the practical design of transformers, practitioners usually use \ac{mha} instead of single attention to express sophisticated functions, which forwards the inputs through $h$ attention modules in parallel and outputs the sum of these sub-modules. 
Here $h\in\bbN$ is a hyperparameter.
Taking $X\in\bbR^{T\times d}$ as the input, \ac{mha} outputs $\MHA(X,W)=\sum_{i=1}^{h}\Att(XW_{i}^{Q},XW_{i}^{K},XW_{i}^{V})$, where $W=(W_{i}^{Q},W_{i}^{K},W_{i}^{V})_{i=1}^{h}$ is the parameters set of $h$ attention modules, $W_{i}^{Q}\in\bbR^{d\times d_{h}}$, $W_{i}^{K}\in\bbR^{d\times d_{h}}$, and $W_{i}^{V}\in\bbR^{d\times d}$ for $i\in[h]$ are weight matrices for queries, keys, and values, and $d_{h}$ is usually set to be $d/h$~\citep{michel2019sixteen}. 
The transformer is the concatenation of the attention modules and the fully-connected layers, which is widely adopted in \ac{llm}s~\citep{devlin2018bert,brown2020language}.

\vspace{-1mm}
\textbf{Large Language Models and In-Context Learning.}
Many \ac{llm}s are \emph{autoregressive}, such as GPT \citep{brown2020language}. It means that the model continuously predicts future tokens based on its own previous values. 
For example, starting from a token $x_{1}\in\frakX$, where $\frakX$ is the alphabet of tokens, a \ac{llm} $\PP_{\theta}$ with parameter $\theta\in\Theta$ continuously predicts the next token according to $x_{t+1}\sim \PP_{\theta}(\cdot\,|\,S_{t})$ based on the past $S_{t}=(x_{1},\cdots,x_{t})$ for $t\in\NN$. Here, each token represents a word and the position of the word~\citep{ke2020rethinking}, and the token sequences $S_{t}$ for $t\in\NN$ live in the sequences space $\frakX^{*}$. \ac{llm}s are first \emph{pretrained} on a huge body of corpus, making the prediction $x_{t+1}\sim \PP_{\theta}(\cdot\,|\,S_{t})$ accurate, and then prompted to perform downstream tasks. 
During the pretraining phase, we aim to maximize the conditional probability $\PP_{\theta}(x\,|\,S)$ over the nominal next token $x$~\citep{brown2020language}.

After pretraining, \ac{llm}s are prompted to perform downstream tasks without tuning parameters. Different from the finetuned models that learn the task explicitly~\citep{liu2023pre}, \ac{llm}s can implicitly learn from the examples in the \emph{prompt}, which is known as \ac{icl}~\citep{brown2020language}. 
Concretely, pretrained \ac{llm}s are provided with a prompt $\pt_t = (\tilc_1, r_1, \ldots, \tilc_t, r_t, \tilc_{t+1})$ with $t$ examples and a query as inputs, where each pair $(\tilc_i, r_i)\in \frakX^{*}\times\frakX$ is an example of the task, and $\tilc_{t+1}$ is the query, as shown in Figure~\ref{fig:pipeline} in Appendix~\ref{app:fig}. 
For example, the $\pt_{t}$ with $t=2$ can be ``Cats are animals, pineapples are plants, mushrooms are''. Here $\tilc_1\in\frakX^{*}$ is a token sequence ``Cats are'', while $r_{1}$ is the response ``animals''. The query $\tilc_{t+1}$ is ``mushrooms are'', and the desired response is ``fungi''.  The prompts are generated from a hidden concept $z_{*}\in\frakZ$, e.g., $z_{*}$ can be the classification of biological categories, where $\frakZ$ is the concept space. The generation process is $\tilc_{i}\sim\PP(\cdot\,|\,\tilc_{1},r_{1},\cdots,\tilc_{i-1},r_{i-1},z_{*})$ and $r_{i}\sim \PP(\cdot\,|\,\pt_{i-1},z_{*})$ for the nominal distribution $\PP$ and $i\in[t]$. 
Thus, when performing \ac{icl}, 
\ac{llm}s aim to estimate the conditional distribution  $\PP(r_{t+1} | \pt_t, z_*)$. 
It is widely conjectured and experimentally found that the pretrained \ac{llm}s can implicitly identify the hidden concept $z_{*}\in\frakZ$ from the examples, and then perform \ac{icl} by outputting from $\PP(r_{t+1} | \pt_t, z_*)$.
In the following, we will provide theoretical justifications for this claim. 
We note that delimiters are omitted in our work, and our results can be generalized to handle this case. Since \ac{llm}s are autoregressive, the definition of the notation $\PP(\cdot\,|\,S)$ with $S\in\frakX^{*}$ may be ambiguous because the length of the subsequent tokens is not specified.  Unless explicitly specified, we let $\PP(\cdot\,|\,S)$  denote the distribution of the next single token conditioned on $S$.



\section{In-Context Learning via Bayesian Model Averaging}\label{sec:icl_bma}

In this section, we show that \ac{llm}s perform \ac{icl} implicitly via \ac{bma}. Given a sequence $S = \{(\tilc_t, r_t)\}_{t = 1}^T$ with $T$ examples generated from a hidden concept $z_{*}\in\frakZ$, we use $S_t = \{(\tilc_i, r_i)\}_{i = 1}^t$ to represent the first $t$ \ac{icl} examples in the sequence. Here $\tilc_t$ and $r_t$ respectively denote the \ac{icl} covariate and response. 
During the \ac{icl} phase, a \ac{llm} is sequentially prompted with $\pt_t = (S_t, \tilc_{t+1})$ for $t\in[T-1]$, 
i.e., the first $t$ examples and the $(t+1)$-th covariate. 
The prompted \ac{llm} aims to predict the response $r_{t+1}$ based on $\pt_t = (S_t, \tilc_{t+1})$ whose true distribution is $r_{t+1}\sim\PP(\cdot \given \pt_t,z_{*})$. 
For the analysis of \ac{icl}, 
we focus on the following latent variable model  
\begin{align}
    \label{eq:model}
    r_t = f(\tilc_t, h_t, \xi_t), \quad \forall t \in [T],
\end{align}
where the hidden variable $h_t\in \cH$ determines the relation between $c_t$ and $r_t$, $\xi_t\in \Xi$ for $t\in[T]$ are i.i.d.\ random noises, and  $f: \cX \times \cH \times \Xi \rightarrow \fX$ is a function that relates response $r_{t}$ to $\tilc_t, h_t$, and $\xi_t$.  
In the data generation process, a hidden concept $z_{*}\in\frakZ$ is first generated from $\bbP(z)$. The hidden variables $\{h_t\}_{t=1}^T$ are then a stochastic process whose distribution is determined by the hidden concept $z_{*}$, that is 
\begin{align*}
    \PP(h_t = \cdot \,|\, \tilde c_t, \{ r_{\ell}, h_{\ell}, \tilde c_{\ell}\}_{\ell < t}) = g_{z_*} (  h_1, \ldots, h_{t-1}, \zeta_t)
\end{align*}
for some function $g_{z_*}$ parameterized by   $z^*$, where $\{\zeta_t\}_{t=1}^{T}$ are exogenous noises. 
The response $r_{t}$ is then generated according to \eqref{eq:model}.  The model in \eqref{eq:model} essentially assumes that the hidden concept $z_{*}$ implicitly determines the transition of the conditional distribution $ \PP(r_t = \cdot \given \tilc_t)$ by affecting the evolution of the latent variables $\{h_t\}_{t\in [T]}$, and it does not impose any assumption on the distribution of $\tilc_t$. This model is quite general, and it subsumes the models in previous works. When $f$ is the emission function in \ac{hmm} and $h_{t}=h$ for $t\in[T]$ is the values of hidden states that depend on $z$, model in \eqref{eq:model} recovers the \ac{hmm} assumption in \cite{xie2021explanation}. When $h_{t}=z$ for $t\in[T]$ degenerate to the hidden concept, this recovers the casual graph model in \cite{wang2023large} and the \ac{icl} model in \cite{jiang2023latent}. 

Assuming that the tokens follow the statistical model given in  \eqref{eq:model}, during pretraining, we collect $N_p$ independent trajectories by sampling from \eqref{eq:model} with concept $z$ randomly sampled from $\bbP(z)$.
Intuitively, during pretraining, by training in an autoregressive manner, the \ac{llm} approximates the conditional distribution $\PP(r_{t+1} \given \pt_t) = \EE_{z\sim \bbP(z)} [\PP(r_{t+1} \given \pt_t, z)] $, which is the conditional distribution of $r_{t+1}$ given $\pt_t$, aggregated over the randomness of the concept $z_*$.

Under the model in \eqref{eq:model}, we will show that   pretrained  \ac{llm}s are able to perform \ac{icl} because they secretly implement \ac{bma} \citep{wasserman2000bayesian} during prompting.
For ease of presentation, we first consider the setting where the \ac{llm} is  \emph{perfectly pretrained}, i.e., the conditional distribution induced by the \ac{llm} is given by $\PP(r_{t+1} \given \pt_t)$. 
We relax this condition by analyzing the pretraining error in Section~\ref{sec:pretrain}.

\begin{proposition}[\ac{llm}s Perform \ac{bma}]
\label{th:bma}
Under the model in \eqref{eq:model}, it holds that
\vspace{-2mm}
\begin{align}
    \label{eq:model-bma}
    \PP(r_{t+1} = \cdot \given \pt_t)  = \int \PP(r_{t+1} = \cdot  \given \tilc_{t+1}, S_t, z) \PP(z \given S_t) \ud z.
\end{align}
\end{proposition}
\vspace{-4.5mm}
We note that the left-hand side of \eqref{eq:model-bma}  is the prediction of the pretrained \ac{llm} given a prompt $\pt_t$. Meanwhile,  the right-hand side 
is exactly the prediction given by the \ac{bma} algorithm that 
infers the posterior belief of the concept $z_{*}$ based on 
$S_t$ and predicts $r_{t+1}$ by aggregating the likelihood in \eqref{eq:model} with respect to the posterior $\PP(z_{*} = \cdot  \given S_t)$.
  Thus, this proposition 
shows that 
perfectly pretrained   \ac{llm}s are able to perform   \ac{icl} because they {\bf implement BMA during prompting}. As mentioned, Proposition \ref{th:bma} is proved under a more general model than the previous works and thus serves as a generalized result of some claims in the previous works. We note that the claim of Proposition~\ref{th:bma} is independent of the network structure. This partially explains why LSTMs demonstrate \ac{icl} ability in \cite{xie2021explanation}. In the next section, we will demonstrate how the attention mechanism helps to implement \ac{bma}. The proof of Proposition~\ref{th:bma} is in Appendix~\ref{sec:pf-th-bma}.

Next, we study the performance of \ac{icl} from an online learning perspective. Recall that \ac{llm}s are continuously prompted with $S_{t}$ and aim to predict the $(t+1)$-th covariate $r_{t+1}$ for $t\in[T-1]$. This can be viewed as an online learning problem. 
For any algorithm that generates a sequence of density estimators $\{\hat \PP(r_t)\}_{t=1}^T$ for predicting $\{ r_t\}_{t\in [T]}$, we consider the following  \ac{icl} regret as its performance metric: 
\vspace{-2mm}
\begin{align}
    \label{eq:regret}
    \regret_t = t^{-1}\sup_z \sum_{i=1}^t \log \PP(r_i \given \pt_{i-1}, z) - t^{-1}\sum_{i=1}^t \log \hat \PP( r_i).
\end{align}
This \ac{icl} regret measures the performance of the estimator $\hat \PP$ compared with the best hidden concept in hindsight. For the perfectly trained \ac{llm}s, the estimator is exactly $\hat \PP(r_t) = \PP(r_{t+1} \given \pt_t)$. 
By building the equivalence of pretrained \ac{llm} and \ac{bma}, we have the following corollary, which shows that predicting $\{r_t\}_{t\in[T]}$ by iteratively prompting the  \ac{llm} incurs a $\mathcal{O}(1/T) $ regret.
\begin{corollary}[\ac{icl} Regret of Perfectly Pretrained Model]
\label{th:bma_reg}
Under the model in \eqref{eq:model}, we have for any $ t\in [T]$ that
\vspace{-3mm}
\begin{align*}
    t^{-1}\sum_{i= 1}^t \log \PP(r_{i} \given \pt_{i-1}) \ge \sup_{z\in \cZ} \Bigl( t^{-1}\sum_{i=1}^{t} \log \PP(r_{i} \given  \pt_{i-1}, z) +t^{-1}\log \PP_{\calZ}(z) \Bigr).
\end{align*}
Here $\PP_{\calZ}$ is the prior of the hidden concept $z\in \cZ$. When the hidden concept space $\frakZ$ is finite and the prior $\PP_{\calZ}(z)$ is the uniform distribution on $\frakZ$, we have that $\regret_t \le \log | \frakZ |/t.$ When the nominal concept $z_*$ satisfies that $\sup_z \sum_{i=1}^t \PP(r_{i} \given z, \pt_{i-1}) = \sum_{i=1}^t\PP(r_{i} \given z_*, \pt_{i-1})$ for any $t \in [T]$, the regret is bounded as $\regret_t \le \log (1/\PP_{\calZ}(z_*))/t.$
\end{corollary}
\vspace{-1mm}
This theorem states that the \ac{icl} regret of the perfectly pretrained model is bounded by $\log (1/\PP_{\calZ}(z_*))/t$. This is intuitive since the regret is relatively large if the concept $z_{*}$ rarely appears according to the prior distribution. 
This corollary shows that, when given sufficiently many examples, predicting $\{r_{t}\}_{t\in[T]}$ via \ac{icl} is almost as good as the oracle method which knows true concept $z_*$ and the likelihood function $\PP(r_{i} \given  \pt_{i-1}, z_*)$. The practical relevance of this result is discussed in Appendix~\ref{app:prac_discuss}.
The proof of Corollary~\ref{th:bma_reg} is in Appendix~\ref{sec:pf-th-bma-reg}. 
In Section~\ref{sec:pretrain}, we characterize the deviation between the learned model and the underlying true model. Next, we show how transformers parameterize \ac{bma}.

\subsection{Attention Parameterizes Bayesian Model Averaging} \label{sec:approx}


In the following, we explore the role played by the attention mechanism in \ac{icl}.
To simplify the presentation, we consider the case where the covariate $\tilde c_t\in\frakX^{*}$ is a single token $c_t\in\frakX$ in this subsection. During the \ac{icl} phase, pretrained \ac{llm}s are prompted with $\pt_t = (S_t, c_{t+1})$ and tasked with predicting the $(t+1)$-th response $r_{t+1}$. The transformers first separately map the covariates $\tilde c_i$ and responses $r_{i}$ for $i\in[t]$ to the corresponding feature spaces, which are usually realized by the fully connected layers. We denote these two learnable mappings as $k : \RR^d \rightarrow \RR^{d_k}$ and $v: \RR^d \rightarrow \RR^{d_v}$. Their nominal values are denoted as $k_{*}$ and $v_{*}$, respectively. The pretraining of the transformer essentially learns the nominal mappings $v_{*}$ and $k_{*}$ with sufficiently many data points. After these transformations, the attention module will take $v_i = v_*(r_i)$ and $k_i = k_*(c_i)$ for $i\in[t]$ as the value and key vectors to predict the result for the query $q_{t+1}=k_{t+1} = k_*(c_{t+1})$. To elucidate the role played by attention, we consider a Gaussian linear simplification of \eqref{eq:model}
\begin{align}
\label{eq:cr_lvm}
v_t = z_{*} \phi(k_t) + \xi_t, \quad \forall t \in [T],
\end{align}
where $\phi: \RR^{d_k} \rightarrow \RR^{d_\phi}$ refers to the feature mapping in some \ac{rkhs}, $z_{*} \in \RR^{d_v \times d_\phi}$ corresponds to the hidden concept, and $\xi_t \sim N(0, \sigma^2 I), t\in[T]$ are i.i.d.  Gaussian noises with covariance $\sigma^2 I$.   
Besides, we assume the prior of $z_{*} $ is $\bbP(z)$ is a Gaussian distribution $N(0, \lambda I)$.
Note that \eqref{eq:cr_lvm} can be written as
\begin{align}
    \label{eq:cr_lvm2}
    r_t = v_*^{-1}\Bigl( z_{*} \phi\bigl(k_*(c_t) \bigr)+ \xi_t \Bigr),
\end{align}
which is a realization of \eqref{eq:model} with $h_t = z$, $\xi_t = \epsilon_t$, and $f(c, h, \xi) = v_*^{-1}( h \phi(k_*(c) )+ \xi)$. 
In other words,  \eqref{eq:cr_lvm}, or equivalently \eqref{eq:cr_lvm2}, specifies a specialization of  \eqref{eq:model} where 
  in the feature space, the hidden concept $z_{*}$ represents a transformation between the value $v$ and the key $k$. Here, we simply take this as the transformation by a matrix, which can be easily generalized by building a bijection between concepts $z$ and complex transformations. In the following, to simplify the notation, let $\fk: \RR^{d_k} \times \RR^{d_k} \rightarrow \RR$.  denote the kernel function of the \ac{rkhs} induced by $\phi$.  The stacks of the values and keys are denoted as $K_t = (k_1, \ldots, k_t)^\top \in \RR^{t\times d_k}$ and $V_t = (v_1, \ldots, v_t)^\top \in \RR^{t \times d_v}$, respectively. Consequently, the model in \eqref{eq:cr_lvm} implies that
\begin{align}\label{eq:bma}
	\PP(v_{t+1} \given \pt_t) \!= \!\int \PP(v_{t+1} \given z, q_{t+1}) \PP(z \given S_t) \ud z \propto \exp\Bigl( -\norm[\big]{v_{t+1} - \bar z_t \phi(q_{t+1}) }_{\Sigma_t^{-1}}^2 \big/ 2  \Bigr),
\end{align}
where we denote by $\Sigma_t$ the covariance of $v_{t+1}\sim\PP(\cdot \given S_t, q_{t+1})$, and the mean concept $\bar z_t$ is
\begin{align}
	\label{eq:cme-att}
	\bar z_t & = V_t \bigl( \phi(K_t) \phi(K_t)^\top + \lambda I \bigr)^{-1}\phi(K_t) = V_t \bigl( \fk(K_t, K_t) + \lambda I \bigr)^{-1}\phi(K_t).
\end{align}

Combining \eqref{eq:bma} and \eqref{eq:cme-att},  we can see that $\bar z_t \phi(q_{t+1})$ essentially measures the similarity between the query and keys, which is quite similar to the attention mechanism defined in Section~\ref{sec:prelim}. However, here the similarity is normalization according to \eqref{eq:cme-att}, not by softmax. This motivates us to define a new structure of attention and explore the relationship between the newly defined attention and the original one. For any $q \in \RR^{d_k}$, $K \in \RR^{t\times d_k}$, and $V \in \RR^{t \times d_v}$, we define a variant of the attention mechanism as follows,
\begin{align}\label{eq:att-cme}
    \att_\dagger(q, K, V) = V^\top \bigl( \fk(K, K) + \lambda I \bigr)^{-1}\fk(K, q).
\end{align}
From  \eqref{eq:bma}, \eqref{eq:cme-att}, and \eqref{eq:att-cme}, it holds that the response $v_{t+1}$ for $(t+1)$-th query is distributed as $v_{t+1} \sim N(\att_\dagger(q_{t+1}, K_t, V_t), \Sigma_t)$. We note that {\bf {$\att_\dagger$ bakes the \ac{bma} algorithm}} for the Gaussian linear model {\bf {in its architecture}}, by first estimating $\barz_{t}$ via \eqref{eq:cme-att} and deriving the final estimate from the inner product between $\barz_{t}$ and $q_{t+1}$. Here $\att_\dagger (\cdot)$ is an instance of the \emph{intention mechanism} studied in \cite{garnelo2023exploring} and can be viewed as a generalization of linear attention. 
Recall that we define the softmax attention \citep{vaswani2017attention} for any $q \in \RR^{d_k}$, $K \in \RR^{t\times d_k}$, and $V \in \RR^{t \times d_v}$ as $\att(q, K, V) = V^\top\softmax(Kq)$.
In the following proposition, we show that the attention in \eqref{eq:att-cme} coincides with the softmax attention as the sequence length goes to infinity.

\begin{proposition}\label{prop:cme-att-limit}
We assume that the key-value pairs $\{(k_t, v_t)\}_{t=1}^T$  are independent and identically distributed, and we adopt Gaussian RBF kernel $\frakK_{\mathtt{ RBF}}$. In addition, we assume that $\norm{k_t}_2 = \norm{v_t} = 1$. Then, it holds for an absolute constant $C>0$ and any $q \in \RR^{d_k}$ with $\norm{q} = 1$ that $\lim_{T\rightarrow \infty} \att_\dagger(q, K_T, V_T) = C \cdot  \lim_{T\rightarrow \infty} \att(q, K_T, V_T).
$
\end{proposition}
\vspace{-3mm}
The proof is in Appendix~\ref{sec:pf-prop-cme}. Combined with the conditional probability of $v_{t+1}$ in \eqref{eq:bma}, this proposition shows that {\bf softmax attention approximately encodes \ac{bma}} in long token sequences~\citep{wasserman2000bayesian}, and thus is able to perform \ac{icl} when prompted after pretraining.





\section{Theoretical Analysis of Pretraining}\label{sec:pretrain}
\vspace{-2mm}
\subsection{Pretraining Algorithm}\label{sec:pretrainalgo}
In this section, we describe the pretraining setting.   We largely follow the transformer structures in \cite{brown2020language}.   The whole network is a composition of $D$ sub-modules, and each sub-module consists of a \ac{mha} and a \ac{ff} fully connected layer. Here, $D>0$ is the depth of the network. The whole network takes $X^{(0)}=X\in\bbR^{L\times d}$ as its input. In the $t$-th layer for $t\in[D]$, it first takes the output $X^{(t-1)}$ of the $(t-1)$-th layer as the input and forwards it through \ac{mha} with a residual link and a layer normalization $\Pi_{\rm{norm}}(\cdot)$ to output $Y^{(t)}$, which projects each row of the input into the unit $\ell_{2}$-ball. Here we take $d_{h}=d$ in \ac{mha}, and the generalization of our result to general cases is trivial. Then the intermediate output $Y^{(t)}$ is forwarded to the \ac{ff} module. It maps each row of the input $Y^{(t)}\in\bbR^{L\times d}$ through the same single-hidden layer neural network with $d_{F}$ neurons, that is $\FF(Y^{(t)},A^{(t)})=\relu(Y^{(t)}A_{1}^{(t)})A_{2}^{(t)}$, where $A_{1}^{(t)}\in\bbR^{d\times d_{F}}$, and $A_{2}^{(t)}\in\bbR^{d_{F}\times d}$ are the weight matrices.  Combined with a residual link and layer normalization, it outputs the output of layer $t$ as $X^{(t)}$, that is
\begin{align}
    Y^{(t)}\!\!=\!\Pi_{\rm{norm}}\big[\MHA(X^{(t-1)}\!,\!W^{(t)})+\!\gamma_{1}^{(t)}X^{(t-1)}\big],\,
    X^{(t)}\!\!=\!\Pi_{\rm{norm}}\big[\FF(Y^{(t)}\!,\!A^{(t)})+\!\gamma_{2}^{(t)}Y^{(t)}\big].\label{eq:netdef}
\end{align}
Here we allocate weights $\gamma_{1}^{(t)}$ and $\gamma_{2}^{(t)}$ to residual links only for the convenience of theoretical analysis. 
In the last layer, the network outputs the probability of the next token via a softmax module, that is $Y^{(D+1)}=\softmax(\bbI_{L}^\top X^{(D)}A^{(D+1)}/(L\tau))\in\bbR^{d_{y}}$, where $\bbI_{L}\in\bbR^{L}$ is the vector with all ones, $A^{(D+1)}\in\bbR^{d\times d_{y}}$ is the weight matrix, $\tau\in(0,1]$ is the fixed temperature parameter, and $d_{y}$ is the output dimension. The parameters of each layer are denoted as $\theta^{(t)}=(\gamma_{1}^{(t)},\gamma_{2}^{(t)},W^{(t)},A^{(t)})$ for $t\in[D]$ and $\theta^{(D+1)}=A^{(D+1)}$, and the parameter of the whole network is the concatenation of these parameters, i.e., $\theta=(\theta^{(1)},\cdots,\theta^{(D+1)})$. We consider the transformers with bounded parameters. The set of parameters~is 
\begin{align*}
    \Theta&=\Big\{\theta\,|\, \big\|A^{(D+1),\top}\big\|_{1,2}\leq B_{A},\max\big\{\big|\gamma_{1}^{(t)}\big|,\big|\gamma_{2}^{(t)}\big|\big\}\leq 1, \big\|A_{1}^{(t)}\big\|_{\rmF}\leq B_{A,1},\big\|A_{2}^{(t)}\big\|_{\rmF}\leq B_{A,2},\\
    &\big\|W_{i}^{Q,(t)}\big\|_{\rmF}\leq B_{Q},\big\|W_{i}^{K,(t)}\big\|_{\rmF}\leq B_{K},\big\|W_{i}^{V,(t)}\big\|_{\rmF}\leq B_{V} \text{ for all }t\in[D],i\in[h]\Big\},
\end{align*}
where $B_{A}$, $B_{A,1}$, $B_{A,2}$, $B_{Q}$, $B_{K}$, and $B_{V}$ are the bounds of parameter. Here we only consider the non-trivial case where these bounds are larger than $1$, otherwise, the magnitude of the output in $D^{\rm th}$ layer decreases exponentially with growing depth. The probability induced by the transformer with parameter $\theta$ is denoted as $\PP_{\theta}$.

The pretraining dataset consists of $N_{\rmp}$ independent trajectories. For the $n$-th trajectory with $n\in[N_{\rmp}]$, a hidden concept $z^{n}\sim \PP_{\calZ}(z)\in\Delta(\frakZ)$ is first sampled, which is the hidden variables of the token sequence to generate, e.g., the theme, the sentiment, and the style. Then the tokens are sequentially sampled from the Markov chain induced by $z^{n}$ as $x_{t+1}^{n}\sim \PP(\cdot\,|\,S_{t}^{n},z^{n})$ and $S_{t+1}^{n}=(S_{t}^{n},x_{t+1}^{n})$, where $x_{t+1}^{n}\in\frakX$, and $S_{t}^{n},S_{t+1}^{n}\in\frakX^{*}$. Here the Markov chain is defined with respect to the state $S_{t}^{n}$, which obviously satisfies the Markov property since $S_{i}^{n}$ for $i\in[t-1]$ are contained in $S_{t}^{n}$. The pretraining dataset is $\calD_{N_{\rmp},T_{\rmp}}=\{(S_{t}^{n},x_{t+1}^{n})\}_{n,t=1}^{N_{\rmp},T_{\rmp}}$ where the concepts $z^{n}$ is hidden from the context and thus unobserved. Here each token sequence is divided into $T_{\rmp}$ pieces $\{(S_{t}^{n},x_{t+1}^{n})\}_{t=1}^{T_{\rmp}}$. We highlight that this pretraining dataset collecting process subsumes those for GPT, and \ac{mae}~\citep{radford2021learning}. For GPT, each trajectory corresponds to a paragraph or an article in the pretraining dataset, and $z^{n}\sim \PP_{\calZ}(z)$ is realized by the selection process of these contexts from the Internet. For \ac{mae}, we take $T_{\rmp}=1$, and $S_{1}^{n}$ and $x_{2}^{n}$ respectively correspond to the image and the masked token.

To pretrain the transformer, we adopt the cross-entropy as the loss function, which is widely used in the training of BERT and GPT. The corresponding pretraining algorithm is
\begin{align}
    \htheta=\argmin_{\theta\in\Theta} -\frac{1}{N_{\rmp}T_{\rmp}}\sum_{n=1}^{N_{\rmp}}\sum_{t=1}^{T_{\rmp}}\log \PP_{\theta}(x_{t+1}^{n}\,|\,S_{t}^{n}).\label{algo:pretrainmle}
\end{align}
We first analyze the population version of \eqref{algo:pretrainmle}. In the training set, the conditional distribution of $x_{t+1}^{n}$ conditioned on $S_{t}^{n}$ is $\PP(x_{t+1}^{n}\,|\,S_{t}^{n})=\int_{\frakZ}\PP(x_{t+1}^{n}\,|\,S_{t}^{n},z)\PP_{\calZ}(z\,|\,S_{t}^{n})\rmd z$, where the unobserved hidden concept is weighed via its posterior distribution. Thus, the population risk of \eqref{algo:pretrainmle} is $\EE_{t}[\bbE_{S_{t}}[\KL(\PP(\cdot\,|\,S_{t})\|\PP_{\theta}(\cdot\,|\,S_{t}))+H(\PP(\cdot\,|\,S_{t}))]]$, where $t\sim\unif([T_{\rmp}])$, $H(p)=-\langle p,\log p\rangle$ is the entropy, and $S_{t}$ is distributed as the pertaining distribution. Thus, we expect that $\PP_{\theta}$ will converge to $\bbP$. For \ac{mae}, the network training adopts $\ell_{2}$-loss, and we defer the analysis of this case to Appendix~\ref{app:ell2}. 

\subsection{Performance Guarantee for Pretraining}\label{sec:pretrainanalysis}
We first state the assumptions for the pretraining setting.
\begin{assumption}\label{assump:boundedmle}
    There exists a constant $R>0$ such that for any $z\in\frakZ$ and $S_{t}\sim\PP(\cdot\,|\,z)$, we have $\|S_{t}^\top\|_{2,\infty}\leq R$ almost surely.
\end{assumption}
This assumption states that the $\ell_2$-norm of the magnitude of each token in the token sequence is upper bounded by $R>0$. This assumption holds in most machine learning settings. For BERT and GPT, each token consists of word embedding and positional embedding. For \ac{mae}, each token consists of a patch of pixels. The $\ell_2$-norm of each token is bounded in these cases.
\begin{assumption}\label{assump:positive}
    There exists a constant $c_{0}>0$ such that for any $z\in\frakZ$, $x\in\frakX$ and $S\in\frakX^*$, we have $\PP(x\,|\,S,z)\geq c_{0}$.
\end{assumption}
This assumption states that the conditional probability of $x$ conditioned on $S$ and $z$ is lower bounded. This comes from the ambiguity of language, that is, a sentence can take lots of words as its next word. Similar regularity assumptions are also widely adopted in \ac{icl} literature~\citep{xie2021explanation,wies2023learnability}. To state our result, we respectively use $\bbE_{S\sim\calD}$ and $\PP_{\calD}$ to denote the expectation and the distribution of the average distribution of $S_{t}^{n}$ in $\calD_{N_{\rmp},T_{\rmp}}$, i.e., $\bbE_{S\sim\calD}[f(S)]=\sum_{t=1}^{T_{\rmp}}\bbE_{S_{t}}[f(S_{t})]/T_{\rmp}$ for any function $f:\frakX^{*}\rightarrow\bbR$.
\begin{theorem}\label{thm:pretrainmle}
    Let $\barB=\tau^{-1}RhB_{A}B_{A,1}B_{A,2}B_{Q}B_{K}B_{V}$ and $\barD=D^{2} d (d_{F}+d_{h}+d)+d\cdot d_{y}$. Under Assumptions~\ref{assump:boundedmle} and ~\ref{assump:positive}, the pretrained model $\PP_{\htheta}$ by the algorithm in \eqref{algo:pretrainmle} satisfies
    \begin{align*}
        &\bbE_{S\sim\calD}\Big[\TV\big(\PP( \cdot \,|\,S),\PP_{\htheta}( \cdot \,|\,S)\big)\Big]\\
        &\quad\!=\!O\!\bigg(\!\!\underbrace{\inf_{\theta^{*}\in\Theta}\!\!\sqrt{\bbE_{S\sim\calD}\KL\big(\PP(\cdot|S)\|\PP_{\theta^{*}}(\cdot|S)\big)}\!+\!\frac{t_{\rm mix }^{1/4}\log 1/\delta}{(N_{\rmp}T_{\rmp})^{1/4}}}_{\displaystyle \text{approximation error}}\!+\!\underbrace{\frac{\sqrt{t_{\rm mix}}}{\sqrt{N_{\rmp}T_{\rmp}}}\Big(\!\barD\log(1\!+\!N_{\rmp}T_{\rmp}\barB)\!+\!\log\!\frac{1}{\delta}}_{\displaystyle \text{generalization error}}\!\Big)\!\bigg)
    \end{align*}
    with probability at least $1-\delta$, where $t_{\rm mix}$ is the mixing time of the Markov chains induced by $\PP$, formally defined in Appendix~\ref{app:MC}.
\end{theorem}
We define the right-hand side of the equation as $\Delta_{\rm pre}(N_{\rmp}, T_{\rmp}, \delta)$. The first and the second terms in the bound are the {\bf approximation error}. It measures the distance between the nominal distribution $\PP$ and the distributions induced by transformers with respect to KL divergence. If the nominal model $\PP$ can be represented by transformers exactly, i.e., the realizable case, these two terms will vanish. The third term is the {\bf generalization error}, and it does not increase with the growing sequence length $T_{\rmp}$. This is proved via the PAC-Bayes framework.

This pretraining analysis is missing in most existing theoretical works about \ac{icl}. \cite{xie2021explanation}, \cite{wies2023learnability}, and \cite{jiang2023latent} all assume access to an arbitrarily precise pretraining model. Although the generalization bound in \cite{litransformers} can be adapted to the pretraining analysis, the risk definition therein can not capture the approximation error in our result. Furthermore, their analysis cannot fit the maximum likelihood algorithm in \eqref{algo:pretrainmle}. Concretely, their result can only show that the convergence rate of KL divergence is $O((N_{\rmp}T_{\rmp})^{-1/2})$ with a realizable function class. Combined with Pinsker's inequality, this gives the convergence rate for total variation as $O((N_{\rmp}T_{\rmp})^{-1/4})$ even in the realizable case.

The deep neural networks are shown to be universal approximators for many function classes~\citep{cybenko1989approximation,hornik1991approximation,yarotsky2017error}. Thus, the approximation error in Theorem~\ref{thm:pretrainmle} should vanish with the increasing size of the transformer. To achieve this, we slightly change the structure of the transformer by admitting a bias term in feed-forward modules, taking $A_{2}^{(t)}\in\bbR^{d_{F}\times d_{F}}$, and admitting $d_{F}$ to vary across layers. This mildly affects the generalization error by replacing $D\cdot d_{F}$ by the sum of $d_{F}$ of all the layers in Theorem~\ref{thm:pretrainmle}. We derive the approximation error bound when the dimension of each word is equal to one, i.e., $\frakX\subseteq\bbR$. Our method can carry over the case $d>1$.
\begin{proposition}[Informal]\label{prop:approxerrmle}
    Under certain smoothness conditions, if $d_{F}\geq 16d_{y}$, $B_{A,1}\geq 16 R d_{y}$, $B_{A,2}\geq d_{F}$ $B_{A}\geq \sqrt{d_{y}}$, and $B_{V}\geq \sqrt{d}$, then for some constant $C>0$, we have 
    \begin{align*}
        \inf_{\theta^{*}\in\Theta}\max_{\|S^\top\|_{2,\infty}\leq R}\KL\big(\PP(\cdot\,|\,S)\,\|\,\PP_{\theta^{*}}(\cdot\,|\,S)\big)=O\bigg( d_{y}\exp\bigg(-\frac{C\cdot D^{1/4}}{\sqrt{\log B_{A,1}}}\bigg)\bigg).
    \end{align*}
\end{proposition}
The formal statement and proof are deferred to Appendix~\ref{app:formalapprox}. This proposition states that the {\bf approximation error decays exponentially with the increasing depth}. Combined with this result, Theorem~\ref{thm:pretrainmle} provides the full description of the pretraining performance.

\section{ICL Regret under Practical Settings}\label{sec:comb}

\subsection{ICL Regret with an Imperfectly Pretrained Model}

xIn Section~\ref{sec:icl_bma}, we study the \ac{icl} regret with a  perfect pretrained model. In what follows, we characterize the ICL regret when the  pretrained model has an error. Note that the distribution $\cD_\icl$ of the prompts of \ac{icl} tasks can be different from that of pretraining. We impose the following assumption on their relation.
\begin{assumption}
    \label{asp:coverage}
    We assume that there exists an absolute constant $\kappa > 0$ such that for any ICL prompt, it holds that $\PP_{\cD_\icl}(\pt) \le \kappa \cdot \PP_{\cD}(\pt)$.
\end{assumption}
This assumption states that the prompt distribution is covered by the pretraining distribution. Intuitively, the pretrained model cannot precisely inference on the datapoint that is outside the support of the pretraining distribution. For example, if the pretraining data does not contain any mathematical symbols and numbers, it is difficult for the pretrained model to calculate $2\times 3$ in \ac{icl} precisely. We then have the following theorem characterizing the ICL regret of the pretrained model.
\begin{theorem}[ICL Regret of Pretrained Model]
    \label{th:iclpretrain}
    We assume that the underlying hidden concept $z_*$ maximizes $\sum_{i=1}^t \log \PP(r_i \given \pt_{i-1}, z)$ for any $t\in [T]$ and there exists an absolute constant $\beta >0$ such that $\log (1/p_0(z_*)) \le \beta$.
    Under Assumptions \ref{assump:boundedmle}, \ref{assump:positive}, and \ref{asp:coverage}, we have with probability at least $1- \delta$ that
    \begin{align*}
        & \EE_{\pt\sim \cD_\icl} \Bigl[ T^{-1} \cdot\sum_{t=1}^{T} \log \PP(r_t \given z^*, \pt_{t-1})- T^{-1} \cdot\sum_{t= 1}^T \log \PP_{\hat \theta} (r_{t} \given \pt_{t-1})  \Bigr]  \\
        & \quad \le \cO\bigl(\beta / T + \kappa \cdot b^*\cdot \Delta_{\rm pre}(N_{\rmp}, T_{\rmp}, \delta) \bigr).
    \end{align*}
    Here we denote by $\Delta_{\rm pre}(N_{\rmp}, T_{\rmp}, \delta)$ the pretraining error in Theorem \ref{thm:pretrainmle}.
\end{theorem}
Theorem \ref{th:iclpretrain} shows that the expected ICL regret for the pretrained model is upper bounded by the sum of two terms: {\bf (a) the ICL regret for the underlying true model}  and {\bf(b) the pretraining error}. These two terms are separately bounded in Sections~\ref{sec:icl_bma}~and~\ref{sec:pretrain}. 

\vspace{-2mm}
\subsection{Prompting with Wrong Input-Output Mappings}

In the real-world implementations of \ac{icl}, the provided input-output examples may not conform to the nominal distribution induced by $z_{*}$, and the outputs in examples can be \emph{perturbed}. We temporarily take concept space $\frakZ$ as a finite space,
and our results can be generalized with a  covering number argument. We denote the prompt considered in Section~\ref{sec:icl_bma} as $\pt_t=(S_{t},\tilc_{t+1})$, $S_{t}=(\tilc_{1},r_{1},\cdots,\tilc_{t},r_{t})\in\frakX^{*}$, and $(\tilc_{i+1},r_{i+1})\sim \PP(\cdot\,|\,S_{i},z_{*})$ for $i\in[t-1]$. Here, each input $\tilc_{i}\in\frakX^{l}$ is a $l$-length token sequence, and each output $r_{i}\in\frakX$ is a single token. The perturbed prompt is then denoted as $\pt^{\prime}=(S_{t}^{\prime},\tilc_{t+1})$, where $S_{t}^{\prime}=(\tilc_{1},r_{1}^{\prime},\cdots,\tilc_{t},r_{t}^{\prime})\in\frakX^{*}$, and $r_{i}^{\prime}$ for $i\in[t]$ is the modified output. We denote the perturbed prompt distribution as $\PP^{\prime}$. Then the performance of \ac{icl} with wrong input-output mappings can be stated as follows.

\begin{proposition}[Informal]\label{prop:wronglabelinf}
    Under certain assumptions, including the distinguishability assumption ($\min_{z\neq z^{*}}\KL_{\rm pair}\big(\PP(\cdot\,|\,z^{*})\,\|\,\PP(\cdot\,|\,z)\big)>2\log 1/c_{0}$), the pretrained model $\PP_{\htheta}$ in \eqref{algo:pretrainmle} predicts the outputs with the prompt containing wrong mappings as
    \begin{align*}
        &\bbE_{\pt^{\prime}}\Big[\KL\big(\PP(\cdot\,|\,\tilc_{t+1},z_{*})\|\PP_{\htheta}(\cdot\,|\,S_{t}^{\prime},\tilc_{t+1})\big)\Big]\\
        &\quad=\!\cO\bigg(\!\Delta_{\rm pre}(N_{\rmp},T_{\rmp},\delta)\!+\!\exp\bigg(\!-\!\frac{\sqrt{t}}{2(1+l)\log 1/c_0}\!\bigg(\min_{z\neq z^{*}}\KL_{\rm pair}\big(\PP(\cdot\,|\,z^{*})\,\|\,\PP(\cdot\,|\,z)\big)+2\log c_{0}\bigg)\!\bigg)\!\bigg)
    \end{align*}
    with probability at least $1-\delta$.
\end{proposition}
The first term is the pretraining error in Theorem~\ref{thm:pretrainmle}, which is related to the size of the pretraining set and the capacity of the neural networks. The second term is the \ac{icl} error. Intuitively, this term represents the concept identification error. If the considered task $z_{*}$ is distinguishable, i.e., satisfying Assumption~\ref{assump:distinguish}, this term decays to $0$ exponentially in $\sqrt{t}$. The required assumptions and formal statement are in Appendix~\ref{app:wonglabel}.

\newpage
\bibliographystyle{ims}
\bibliography{ref.bib}

\newpage
\appendix
\begin{center}
{\Large {\bf Appendix for \\ ``What and How does In-Context Learning Learn? Bayesian Model Averaging, Parameterization, and Generalization''}}
\end{center}

\setcounter{equation}{0}
\counterwithin*{equation}{section}

\renewcommand{\theequation}{\thesection.\arabic{equation}}

\section{Conclusion}
In this paper, we investigated the theoretical foundations of \ac{icl} for the pretrained language models. We proved that the perfectly pretrained \ac{llm}s implicitly implements \ac{bma} with regret $\cO(1/t)$ over a general response generation modeling, which subsumes the models in previous works. Based on this, we showed that the attention mechanism parameterizes the \ac{bma} algorithm. Analyzing the pretraining process, we demonstrated that the total variation between the pretrained model and the nominal distribution consists of the approximation error and the generalization error. The combination of the \ac{icl} regret and the pretraining performance gives the full description of \ac{icl} ability of pretrained \ac{llm}s. We mainly focus on the prompts that comprise several examples in this work and leave the analysis of instruction-based prompts for future works.
\section{More Related Works}
\textbf{Transformers.} Our work is also related to the works that theoretically analyze the performance of transformers. For the analytic properties of transformers, \cite{vuckovic2020mathematical} proved that attention is Lipschitz-continuous via the view of interacting particles. \cite{noci2022signal} provided the theoretical justification of the rank collapse phenomenon in transformers. \cite{yun2019transformers} demonstrated that transformers are universal approximators. For the statistical properties of transformers, \cite{malladi2022kernel}, \cite{hron2020infinite}, and \cite{yang2020tensor} analyzed the training of transformers within the neural tangent kernel framework. \cite{wei2022statistically} presented the approximation and generalization bounds for learning boolean circuits and Turing machines with transformers. \cite{edelman2021inductive} and \cite{litransformers} derived the generalization error bound of transformers. In our work, we analyze transformers from both the analytic and statistical sides. We show that attention essentially implements the \ac{bma} algorithm in the \ac{icl} setting. Furthermore, we derive the approximation and generalization bounds for transformers in the pretraining phase.

\textbf{Generalization.} Our analysis of the pretraining is also related to the generalization analysis of the neural networks. This topic has attracted a lot of interests for a long time. \cite{anthony1999neural} derived the uniform generalization bound for fully-connected neural networks with the help pf VC dimension. \cite{bartlett2017spectrally} sharpened this generalization bound for classification problem by adopting the Dudley's integral and calculating of the covering number of neural network class. At the same time, \cite{neyshabur2017pac} derived a similar as \cite{bartlett2017spectrally} from PAC-Bayes framework. Following this line, \cite{liao2020pac} , \cite{ledent2021norm} and \cite{lin2019generalization} built the generalization bound for graph neural networks and convolutional neural network. These results respected the underlying graph structure and the translation-invariance in the networks. \cite{edelman2021inductive} established the generalization bound for transformer, but this result did not reflect the permutation-invariance, still depending on the channel number. Our work focuses on the analysis of \ac{mle} with transformer function class, which is not covered by previous works. Our bounds are sharper than that of \cite{edelman2021inductive} on the channel number dependency.

\section{Figure for Pretraining and \ac{icl}}\label{app:fig}
\begin{figure}[H] 
    \centering 
    \includegraphics[width=0.9\textwidth]{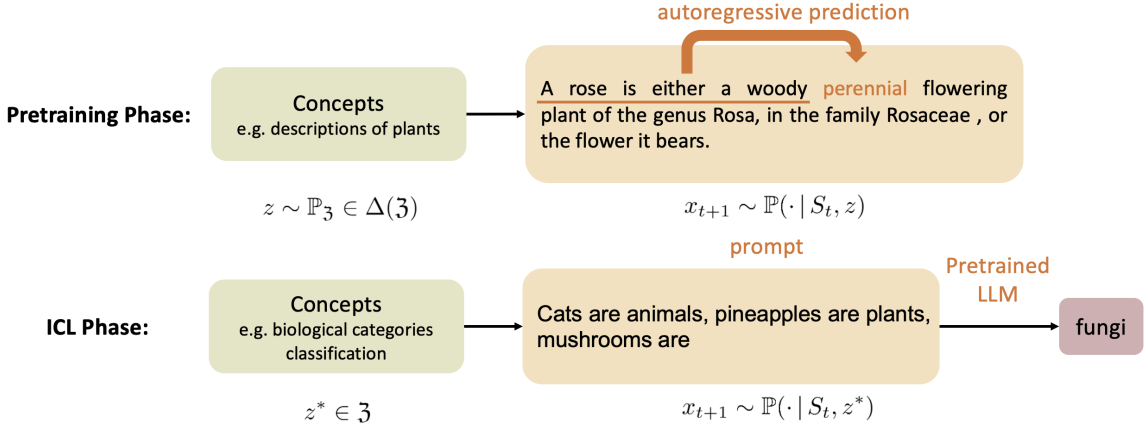} 
    \caption{To form the pretraining dataset, a hidden concept $z$ is first sampled according to $\PP_{\frakZ}$, and a document is generated from the concept. Taking the token sequence $S_{t}$ up to position $t\in[T]$ as the input, the \ac{llm} is pretrained to maximize the next token $x_{t+1}$. During the \ac{icl} phase, the pretrained \ac{llm} is prompted with several examples to predict the response of the query.} 
    \label{fig:pipeline}
\end{figure}

\section{Discussion About The Experimental Results in Existing Works}\label{app:prac_discuss}

We note that \cite{akyurek2022learning} trains the transformer from scratch with the ICL distribution. This is a special case of the pretraining setting in our work, where they set the pretraining distribution as the ICL distribution. Figure 2 in \cite{akyurek2022learning} indicates that the ICL behavior of the transformer matches the behavior of the Bayesian predictor. This is exactly the result we prove in Theorem 4.1.

We note that that dependency on the size of the hidden variable space is also verified in the experiments in \cite{garg2022can}. Figure 2 in \cite{garg2022can} indicates that the ICL of LLMs only has a significant error when $T\leq d$, where $d$ is the dimension of the hidden variable. This implies that the regret of the ICL by LLMs is at most linear in $O(d/T)$. From the view of our theoretical result, discretizing the set $\{z\in\mathbb{R}^{d} \,|\, \|z\|_{2}\leq d \}$ with approximation error $\delta>0$ will result in a set with $(C/\delta)^{d}$ elements, where $C>0$ is an absolute constant. Corollary 4.2 implies that the regret is $\log |\mathfrak{Z}|/T = d\log (C/\delta)/T$, which matches Figure 2 in \cite{garg2022can}.


\section{Proofs for Section~\ref{sec:approx}}
\subsection{Introduction of Conditional Mean Embedding}\label{app:cme}
Let $\cH_k$ and $\cH_v$ be the two \ac{rkhs}s over the spaces $\frakQ$ and $\frakV$ with the kernels $\fk$ and $\fl$, respectively. We denote by $\phi: \frakQ \rightarrow \ell_2$ and $\varphi: \frakV \rightarrow \ell_2$ the feature mappings associated with $\cH_k$ and $\cH_v$, respectively. Here $\l_{2}$ is the space of the square-integrable function class. Then it holds for any $k, k' \in \frakQ$ and $v, v' \in \frakV$ that
\begin{align}
	\label{eq:feature-kernel}
	\phi(k)^\top \phi(k') = \fk(k, k'), \qquad \varphi(v)^\top \varphi(v) = \fl(v, v').
\end{align}
Let $\PP_{\cK, \cV}$ be the joint distribution of the two random variables $\cK$ and $\cV$ taking values in $\frakQ$ and $\frakV$, respectively. Then the conditional mean embedding $\cme(q, \PP_{\cK, \cV}) \in \cH_v$ of the conditional distribution $\PP_{\cV \given \cK}$ is defined as 
\begin{align*}
	\cme(q, \PP_{\cK, \cV}) = \EE \bigl[\fl(\cV, \cdot) \biggiven \cK= q\bigr].
\end{align*}
The conditional mean embedding operator $C_{\cV \given \cK}: \cH_k \rightarrow \cH_v$ is a linear operator such that
\begin{align*}
	C_{\cV \given \cK} \fk(q, \cdot) = \cme(q, \PP_{\cK, \cV}),
\end{align*}
for any $q \in \frakQ$.
We define the (uncentered) covariance operator $C_{\cK\cK}: \cH_k \rightarrow \cH_k$ and the (uncentered) cross-covariance operator $C_{\cV\cK} : \cH_k \rightarrow \cH_v$ as follows, 
\begin{align*}
	C_{\cK\cK} = \EE \bigl[ \fk(\cK, \cdot) \otimes \fk(\cK, \cdot)\bigr], \qquad C_{\cV\cK} = \EE \bigl[ \fl(\cV, \cdot) \otimes \fk(\cK, \cdot)\bigr].
\end{align*}
Here $\otimes$ is the tensor product. \cite{song2009hilbert} shows that $C_{\cV \given \cK} = C_{\cV\cK} C_{\cK\cK}^{-1}$. Thus, we have that
\begin{align}\label{eq:cme-cov}
	\cme(c, \PP_{\cK, \cV}) = C_{\cV\cK} C_{\cK\cK}^{-1} \fk(c, \cdot).
\end{align}

For i.i.d.\ samples $\{(k^\ell, v^\ell)\}_{\ell \in [L]}$ of $\PP_{\cK, \cV}$, $\norm{\cdot}_{\mathrm{HS}}$ denotes the Hilbert-Schmidt norm, we write $\phi(K) = (\phi(k^1), \ldots, \phi(k^L))^\top \in \RR^{L \times d_\phi}$ and $\varphi(V) = (\phi(v^1), \ldots, \phi(v^L))^\top \in \RR^{L \times d_\varphi}$. Then the empirical covariance operator $\hat C_{\cK\cK}$ and empirical cross-covariance operator $\hat C_{\cV\cK}$  are defined as
\begin{align}
	\label{eq:emp-cov}
	\hat C_{\cK\cK} &= L^{-1} \sum_{\ell=1}^{L}\phi(k^\ell) \phi(k^\ell)^\top = L^{-1}\phi(K)^\top \phi(K) \in \RR^{d_\phi\times d_\phi} \nonumber \\
	\hat C_{\cV\cK} &= L^{-1} \sum_{\ell=1}^{L} \varphi(v^\ell)\varphi(k^\ell)^\top = L^{-1}\varphi(V)\phi(K)^\top \in \RR^{d_\varphi\times d_\phi}.
\end{align}
The empirical version of the conditional operator is
\begin{align*}
	\hat C_{\cV \given \cK}^\lambda = \varphi(Y)^\top \phi(X)\bigl(\phi(X)^\top \phi(X) + \lambda \cI \bigr)^{-1} =  \hat C_{\cV\cK}(\hat C_{\cK\cK} + L^{-1}\lambda \cI)^{-1} \in \RR^{d_\varphi \times d_\phi}.
\end{align*}

\subsection{Proof of Proposition \ref{th:bma}} \label{sec:pf-th-bma}
\begin{proof}
By \eqref{eq:model}, we have that
\begin{align}
    \label{eq:model00}
    \PP(r_{t+1}\given \pt_t)  &= \int \PP(r_{t+1} \given c_{t+1}, h_{t+1}) \PP(h_{t+1} \given S_t) \ud h_{t+1}  \noend
    & = \int \PP(r_{t+1} \given c_{t+1}, h_{t+1}) \PP(h_{t+1} \given S_t, z) \PP(z \given S_t) \ud h_{t+1} \ud z \noend 
    &= \int \PP(r_{t+1} \given c_{t+1}, S_t, z) \PP(z \given S_t) \ud z,
\end{align}
where the first and the last equalities results from model \eqref{eq:model-bma}, and the second equality results from Bayes' theorem.
\end{proof}

\subsection{Proof of Corollary \ref{th:bma_reg}} \label{sec:pf-th-bma-reg}
\begin{proof}
Note that 
\begin{align*}
	\PP(z \given S_t) &= \frac{\PP(S_t\given z) \PP_{\calZ}(z) }{\int \PP(S_t\given z') \PP_{\calZ}(z') \ud z'}  = \frac{\prod_{i=1}^{t} \PP(r_i \given z, S_t, c_i ) \PP_{\calZ}(z) }{\int \prod_{i=1}^{t} \PP(r_i \given z', S_{i-1}, c_i) \PP_{\calZ}(z') \ud z'}.
\end{align*}
Then, by Bayesian model averaging, we have the following density estimation,
\begin{align*}
	\PP(r_{t+1} \given S_t, c_{t+1}) &= \int \PP(r_{t+1} \given z, S_t, c_{t+1}) \PP(z \given S_t) \ud z \\
	& = \frac{ \int \prod_{i=1}^{t +1} \PP(r_i \given z, S_{i-1}, c_i) \PP_{\calZ}(z) \ud z }{\int \prod_{i=1}^{t} \PP(r_i \given z', S_{i-1}, c_i) \PP_{\calZ}(z') \ud z'}.
\end{align*}
Thus, it holds that
\begin{align*}
	-\sum_{t= 0}^T \log \PP(r_{t+1} \given c_{t+1}, S_t) & = -\sum_{i = 1}^{t} \biggl(\log \int \prod_{i=1}^{t +1} \PP(r_i \given z, S_{i-1}, c_i) \PP_{\calZ}(z) \ud z - \log \int \prod_{i=1}^{t} \PP(r_i \given z, S_{i-1}, c_i) \PP_{\calZ}(z) \ud z \biggr) \\
	& = -\log \int \prod_{i=1}^{T +1} \PP(r_i \given z, S_{i-1}, c_i) \PP_{\calZ}(z) \ud z \\
	& = \inf_q \EE_{z\sim q} \biggl[ -  \sum_{i=1}^{T +1} \log \PP(r_i \given z, S_{i-1}, c_i)\biggr] +   \EE_{z\sim q} \biggl[\log\frac{q(z)}{\PP_{\calZ}(z)} \biggr].
\end{align*}
We consider $q$ to be in the class of all Dirac measures. Then, we have that
\begin{align*}
    -\frac{1}{T}\sum_{t= 1}^T \log \PP(r_{t} \given c_{t}, S_{t-1}) \le \frac{1}{T} \inf_z\Bigl( -  \sum_{t=1}^{T} \log \PP(r_t \given z, S_{t-1}, c_t) - \log \PP_{\calZ}(z) \Bigr).
\end{align*}
Thus, the statistical convergence rate of the Bayesian posterior averaging is $\cO(1/T)$.
\end{proof}

\subsection{Proof of Proposition \ref{prop:cme-att-limit}} \label{sec:pf-prop-cme}
\begin{proof}
    The proof of Proposition~\ref{prop:cme-att-limit} mainly involves two steps
    \begin{itemize}
        \item Build the relationship between $\att_\dagger$ and conditional mean embedding.
        \item Build the relationship between the $\att$ and conditional mean embedding.
    \end{itemize}
    
    \textbf{Step 1: Build the relationship between $\att_\dagger$ and conditional mean embedding.}
    
    In the following, we adopt $\cH_k$ and $\cH_v$ to denote the \ac{rkhs}s for the key and the value with the kernel functions $\fk$ and $\fl$, respectively. Also, we use $\|\cdot\|$ to denote the norm of \ac{rkhs} for an element in the corresponding \ac{rkhs} and the operator norm of the operators that transform elements between \ac{rkhs}s. For the value space, we adopt the Euclidean kernel $\fl(v, v') = v^\top v'$, and the feature mapping $\varphi$ is the identity mapping. Recall the definition of the empirical covariance operator and the empirical cross-covariance operator in Appendix~\ref{app:cme}. For keys and values, we correspondingly define them as
    \begin{align*}
	    \hat C_{\cK\cK} = L^{-1}\phi(K)^\top \phi(K) , \, \hat C_{\cV\cK} = L^{-1}\varphi(V)^\top\phi(K) , \, \hat C_{\cV\cV} = L^{-1}\varphi(V)^\top \varphi(V),
    \end{align*}
    where $\phi(K) = (\phi(k^1), \ldots, \phi(k^L))^\top \in \RR^{L\times d_\phi}$ and $\varphi(V) = (\varphi(v^1), \ldots, \varphi(v^L))^\top \in \RR^{L\times d_\varphi}$
    By the definition of the newly defined attention in Section~\ref{sec:approx}, we have that
    \begin{align*}
	    \att_\dagger(q, K, V) = \hatC_{\cV\cK} (\hat C_{\cK\cK} + L^{-1}\lambda \cI)^{-1} \phi(q),
    \end{align*}
    which implies that $\att_\dagger$ recovers the empirical conditional mean embedding.
    By \eqref{eq:cme-cov}, it holds that
    \begin{align}\label{eq:tac1}
		&\norm[\big]{\att_\dagger(q, K, V) - \cme(q, \PP_{\cK, \cV})} \noend 
		&\quad \le \underbrace{\norm[\big]{ \hatC_{\cV\cK} (\hat C_{\cK\cK} + L^{-1}\lambda \cI)^{-1} \phi(q) - C_{\cV\cK} (C_{\cK\cK} + L^{-1}\lambda \cI)^{-1} \phi(q)}}_{\displaystyle\text{(i)}} \noend 
		&\qquad + \underbrace{\norm[\big]{ C_{\cV\cK} (C_{\cK\cK} + L^{-1}\lambda \cI)^{-1} \fk(q, \cdot) - C_{\cV\cK} C_{\cK\cK}^{-1} \fk(q, \cdot)}}_{\displaystyle\text{(ii)}}.
	\end{align}

\vskip4pt

\noindent{\bf Upper bounding term (i) of \eqref{eq:tac1}.} 
Following the proof from \cite{song2009hilbert}, we only need to upper bound $\norm{ \hatC_{\cV\cK} (\hat C_{\cK\cK} + L^{-1}\lambda \cI)^{-1} - C_{\cV\cK} (C_{\cK\cK} + L^{-1}\lambda \cI)^{-1}}$. It holds that 
\begin{align}\label{eq:tac11}
    &\norm[\big]{ \hatC_{\cV\cK} (\hat C_{\cK\cK} + L^{-1}\lambda \cI)^{-1} - C_{\cV\cK} (C_{\cK\cK} + L^{-1}\lambda \cI)^{-1}} \\
    & \quad \le \norm[\big]{ \hat C_{\cV\cK} (\hat C_{\cK\cK} + L^{-1}\lambda \cI)^{-1}(\hat C_{\cK\cK} - C_{\cK\cK} ) (C_{\cK\cK} + L^{-1}\lambda \cI)^{-1} } + \norm[\big]{ (\hatC_{\cV\cK} - C_{\cV\cK}) ( C_{\cK\cK} + L^{-1}\lambda \cI)^{-1} }. \nonumber
\end{align}
Considering the first term on the right-hand side of \eqref{eq:tac11}, we have the operator decomposition $\hatC_{\cV\cK} = \hatC_{\cV\cV}^{1/2} \cW \hat C_{\cK\cK}^{1/2}$ for $\cW$ such that  $\norm{\cW} \le 1$.
This decomposition implies that
\begin{align}\label{eq:tac12}
    & \norm[\big]{ \hat C_{\cV\cK} (\hat C_{\cK\cK} + L^{-1}\lambda \cI)^{-1}(\hat C_{\cK\cK} - C_{\cK\cK} ) (C_{\cK\cK} + L^{-1}\lambda \cI)^{-1} } \noend
    & \, \le \norm{\hatC_{\cV\cV}}^{1/2} \cdot \norm[\big]{ \hat C_{\cK\cK}^{1/2} (\hat C_{\cK\cK} + L^{-1}\lambda \cI)^{-1/2} } \cdot \norm[\big]{(\hat C_{\cK\cK} + L^{-1}\lambda \cI)^{-1/2}} \cdot \norm[\big]{ (\hat C_{\cK\cK} - C_{\cK\cK} ) (C_{\cK\cK} + L^{-1}\lambda \cI)^{-1} } \noend
    & \, \le (L^{-1}\lambda)^{-1/2} \cdot \norm[\big]{ (\hat C_{\cK\cK} - C_{\cK\cK} ) (C_{\cK\cK} + L^{-1}\lambda \cI)^{-1} },
\end{align}
where the last inequality follows from the fact that
\begin{align*}
	\norm{\hatC_{\cV\cV}}^2 = L^{-1} \sum_{\ell = 1}^{L} \norm{v^\ell}_2^2 \le 1, \quad 
	\hat C_{\cK\cK} (\hat C_{\cK\cK} + L^{-1}\lambda \cI)^{-1} \le \cI, \quad 
	(\hat C_{\cK\cK} + L^{-1}\lambda \cI)^{-1} \le (L^{-1}\lambda)^{-1} \cI.
\end{align*}
Combining \eqref{eq:tac12} and \eqref{eq:tac11}, we have
\begin{align}\label{eq:tac13}
    &\norm[\big]{ \hatC_{\cV\cK} (\hat C_{\cK\cK} + L^{-1}\lambda \cI)^{-1} - C_{\cV\cK} (C_{\cK\cK} + L^{-1}\lambda \cI)^{-1}} \\
    &\quad \le (L^{-1}\lambda)^{-1/2} \cdot \norm[\big]{ (\hat C_{\cK\cK} - C_{\cK\cK} ) (C_{\cK\cK} + L^{-1}\lambda \cI)^{-1} } + \norm[\big]{ (\hatC_{\cV\cK} - C_{\cV\cK}) ( C_{\cK\cK} + L^{-1}\lambda \cI)^{-1} }. \nonumber
\end{align}
In the following, we will upper bound the second term on the right-hand side of \eqref{eq:tac13} with Lemma \ref{lem:cme-concen}. For this purpose, we define $\xi: \RR^{d_\tp} \times \RR^d \rightarrow  \cH_k \otimes \cH_v$ as follows,
	\$
		\xi(k, v) = \varphi(v) \phi(k)^\top (C_{\cK\cK} + L^{-1}\lambda \cI)^{-1}.
	\$
	Since the operator norm of $(C_{\cK\cK} + L^{-1}\lambda \cI)^{-1}$ is upper bounded by $ (L^{-1}\lambda)^{-1}$, we have that
	\begin{align*}
		\norm[\big]{\xi(k, v)} = \norm[\big]{(C_{\cK\cK} + L^{-1}\lambda \cI)^{-1}} \cdot \norm[\big]{\varphi(v)} \cdot \norm[\big]{\phi(k)} \le C \cdot (L^{-1}\lambda)^{-1},
	\end{align*}
	where $C >0$ is an absolute constant.
	Additionally, we can bound the expectation of the squared norm of $\xi(k, v)$ as 
	\begin{align*}
		\EE\Bigl[\norm[\big]{\xi(k, v)}^2\Bigr] & = \EE\Bigl[ \norm[\big]{  \phi(k)^\top(  C_{\cK\cK} + L^{-1}\lambda \cI)^{-1}}^2 \cdot \norm[\big]{\varphi(v)}^2 \Bigr] \noend
		                                  & \le  \EE\Bigl[ \norm[\big]{(  C_{\cK\cK} + L^{-1}\lambda \cI)^{-1}  \phi(k)}^2 \Bigr] \noend
		                                  & \le (L^{-1}\lambda)^{-1} \cdot \EE\Bigl[ \inp[\big]{(  C_{\cK\cK} + L^{-1}\lambda \cI)^{-1}  \phi(k)}{  \phi(k)} \Bigr].
	\end{align*}
	Using the definition of the trace operator, we have
	\begin{align*}
		\EE\Bigl[\norm[\big]{\xi(k, v)}^2\Bigr] & \le \EE\Bigl[ \Tr\bigl((  C_{\cK\cK} + L^{-1}\lambda \cI)^{-2}  \phi(k)  \phi(k)^\top \bigr) \Bigr] \noend
		                                  & \le (L^{-1}\lambda)^{-1} \cdot  \Tr\bigl((  C_{\cK\cK} + L^{-1}\lambda \cI)^{-1}  C_{\cK\cK}\bigr) \\
		                                  &= (L^{-1}\lambda)^{-1} \cdot \Gamma(L^{-1}\lambda).
	\end{align*}
	Here $\Gamma(L^{-1}\lambda)$ is the effective dimension of $C_{\cK\cK}$ in \cite{caponnetto2007optimal}, which is defined as follows,
	\begin{align*}
		\Gamma(L^{-1}\lambda) =  \Tr\bigl((  C_{\cK\cK} + L^{-1}\lambda \cI)^{-1}  C_{\cK\cK}\bigr).
	\end{align*}
	We apply Lemma \ref{lem:cme-concen} with $B = C (L^{-1}\lambda)^{-1}$ and $\sigma^2 = (L^{-1}\lambda)^{-1} \cdot \Gamma(L^{-1}\lambda)$, then we have that with probability at least $1- \delta$, the following holds
	\begin{align}
		\label{eq:ce20}
		\norm[\big]{\hat C_{\cV\cK} (  C_{\cK\cK} + L^{-1}\lambda \cI)^{-1} -   C_{\cV\cK} (  C_{\cK\cK} + L^{-1}\lambda \cI)^{-1}} \le C \cdot \biggl( \frac{2}{\lambda} + \sqrt{\frac{\Gamma(L^{-1}\lambda)}{\lambda}} \biggr) \log\frac{2}{\delta},
	\end{align}
	where $C>0$ is an absolute constant. Similarly, we can prove that with probability at least $1-\delta$, the following holds
	\begin{align}
		\label{eq:ce221}
		\norm[\big]{\hat C_{\cK\cK} (  C_{\cK\cK} + L^{-1}\lambda \cI)^{-1} -   C_{\cK\cK} (  C_{\cK\cK} + L^{-1}\lambda \cI)^{-1}} \le C' \cdot \biggl( \frac{2}{\lambda} + \sqrt{\frac{\Gamma(L^{-1}\lambda)}{\lambda}} \biggr)  \log\frac{2}{\delta}.
	\end{align}
Here $C'>0$ is an absolute constant.
Combining \eqref{eq:tac13}, \eqref{eq:ce20}, and \eqref{eq:ce221}, we have with probability at least $1- \delta$ that
	\begin{align}\label{eq:tac1100}
	    &\norm[\big]{ \hatC_{\cV\cK} (\hat C_{\cK\cK} + L^{-1}\lambda \cI)^{-1} - C_{\cV\cK} (C_{\cK\cK} + L^{-1}\lambda \cI)^{-1}} \noend 
	    &\quad \le C'' \cdot \sqrt{\frac{L}{\lambda}} \cdot \biggl( \frac{2}{\lambda} + \sqrt{\frac{\Gamma(L^{-1}\lambda)}{\lambda}} \biggr)  \log\frac{2}{\delta}.
	\end{align}

\vskip4pt	

\noindent{\bf Upper bounding term (ii) of \eqref{eq:tac1}.}
We follow the procedures in the proof from \cite{fukumizu2015nonparametric}. 
For any $g \in \cH_k$, we have that
\begin{align*}
	\inp{C_{\cV\cK} (g)}{ C_{\cV\cK} (g)} & = \EE\bigl[\fl(\cV, \bar \cV) g(\cK) g(\bar \cK)\bigr] \\
	& = \inp[\Big]{(C_{\cK\cK} \otimes C_{\cK\cK})\EE\bigl[\fl(\cV, \bar \cV) \biggiven \cK = \cdot , \bar \cK = \ddagger \bigr]}{g \otimes g}.
\end{align*}
Similarly, for any $q \in \RR^{d_\tp}$ and any $g \in \cH_k$, we have that
\begin{align*}
	\inp[\Big]{ C_{\cV\cK}}{\EE\bigl[\fl(\cV, \cdot) \biggiven \cK = q\bigr] } &= \inp[\Big]{\EE\bigl[\fl(\cV, \bar\cV) \biggiven \cK = q, \cK = \ddagger\bigr]}{C_{\cK\cK} g} \\
	& = \inp[\Big]{(\cI \otimes C_{\cK\cK}) \EE\bigl[\fl(\cV, \bar \cV) \biggiven \cK = \cdot , \bar \cK = \ddagger \bigr]}{\fl(\cdot, q) \otimes g}.
\end{align*}
Taking $g = (C_{\cK\cK} + L^{-1}\lambda \cI)^{-1} \fk(q, \cdot)$, we have that
\begin{align*}
    &\norm[\big]{ C_{\cV\cK} (C_{\cK\cK} + L^{-1}\lambda \cI)^{-1} \fk(q, \cdot) - C_{\cV\cK} C_{\cK\cK}^{-1} \fk(q, \cdot)}^2 \noend
    & \quad = \bigg\langle \Big( (C_{\cK\cK} + L^{-1}\lambda \cI)^{-1} C_{\cK\cK} \otimes (C_{\cK\cK} + L^{-1}\lambda \cI)^{-1} C_{\cK\cK} - \cI \otimes  (C_{\cK\cK} + L^{-1}\lambda \cI)^{-1} C_{\cK\cK}    \noend
    &   \quad \qquad (C_{\cK\cK} + L^{-1}\lambda \cI)^{-1} C_{\cK\cK} \otimes \cI + \cI \otimes \cI \Big) \EE\bigl[\fl(\cV, \bar \cV) \biggiven \cK = \cdot , \bar \cK = \ddagger \bigr], \fk(q, \cdot) \otimes \fk(q, \dagger) \bigg\rangle. 
\end{align*}
We note that $\EE[ \fl(v, \bar v) \given k = \cdot , \bar k = \ddagger ] \in \cH_k \otimes \cH_k$ is in the range spanned by $C_{\cK\cK}\otimes C_{\cK\cK}$. Thus, we can define $ \tilde \cC \in \cH_k \times \cH_k$ such that $(C_{\cK\cK}\otimes C_{\cK\cK})\tilde \cC = \EE[ \fl(v, \bar v) \given k = \cdot , \bar k = \ddagger ]$. Let $\{\lambda_i\}_{i = 1}^\infty$ and $\{\varphi_i\}_{i =1}^\infty$ be the eigenvalues and eigenvectors of $C_{\cK\cK}$, respectively. We then have that
\begin{align*}
	&\norm[\big]{ C_{\cV\cK} (C_{\cK\cK} + L^{-1}\lambda \cI)^{-1} \fk(q, \cdot) - C_{\cV\cK} C_{\cK\cK}^{-1} \fk(q, \cdot)}^4 \\
    & \quad \le \bigg\|\Big( (C_{\cK\cK} + L^{-1}\lambda \cI)^{-1} C_{\cK\cK} \otimes (C_{\cK\cK} + L^{-1}\lambda \cI)^{-1} C_{\cK\cK} - \cI \otimes  (C_{\cK\cK} + L^{-1}\lambda \cI)^{-1} C_{\cK\cK}    \noend
    &   \quad \qquad (C_{\cK\cK} + L^{-1}\lambda \cI)^{-1} C_{\cK\cK} \otimes \cI + \cI \otimes \cI \Big) \EE\bigl[\fl(\cV, \bar \cV) \biggiven \cK = \cdot , \bar \cK = \ddagger \bigr] \bigg\|^2 \noend
    & \quad = \sum_{i,j} \biggl( \frac{\lambda_i \lambda_j (L^{-1}\lambda)^2}{(\lambda_i + L^{-1}\lambda)(\lambda_j + L^{-1}\lambda)} \biggr)^2 \cdot  \inp{\varphi_i \otimes \varphi_j}{ \tilde \cC}^2 \noend
    & \quad \le (L^{-1}\lambda)^4 \cdot \norm{\tilde \cC}^2.
\end{align*}
Thus, we have 
\#\label{eq:tac223}
    \norm[\big]{ C_{\cV\cK} (C_{\cK\cK} + \lambda \cI)^{-1} \fk(q, \cdot) - C_{\cV\cK} C_{\cK\cK}^{-1} \fk(q, \cdot)}_2 \le C\cdot \lambda L^{-1},
\#
where $C > 0$ is an absolute constant.

Combining \eqref{eq:tac1}, \eqref{eq:tac1100}, and \eqref{eq:tac223}, we have with probability at least $ 1- \delta$, the following holds
\begin{align}
	\norm[\big]{\att_\dagger(q, K, V) - \cme(q, \PP_{\cK, \cV})} & \le \cO\biggl( \sqrt{\frac{L}{\lambda}} \cdot \biggl( \frac{2}{\lambda} + \sqrt{\frac{\Gamma(L^{-1}\lambda)}{\lambda}} \biggr)  \log\frac{1}{\delta} + \lambda L^{-1}\biggr).\label{ieq:attdaggercme}
\end{align}
Since $\frakK$ is Gaussian RBF kernel, we have that $\Gamma(L^{-1}\lambda)=\cO(L/\lambda)$.

\textbf{Step 2: Build the relationship between the $\att$ and conditional mean embedding.}

We achieve our goal in two sub-steps. In the first step, we prove that there exists a constant $C>0$ such that
\begin{align}
    \att_\sm(q, K, V)=C\int_{\bbS^{d-1}} v \hat\PP^{\fk}_{\cV\given \cK}(v\given q) \ud v,\label{eq:smattreformulate}
\end{align}
where $\SSS^{d-1}$ is the $(d-1)$-dimensional unit sphere. Here $\hat\PP^{\fk}_{\cV\given \cK}$ is the kernel conditional density estimation of $\PP_{\cV\given \cK}$ defined as follows,
\begin{align*}
    \hat\PP^{\fk}_{\cV\given \cK}(v\given q) = \frac{\sum^L_{\ell=1}\fk(k^\ell, q)\cdot\fk(v^\ell, v)}{\sum^L_{\ell=1}\fk(k^\ell, q)},
\end{align*}
where $\iota = 1/\int_{\SSS^{d - 1}} \fk(k, q) \rd q $ is a normalization constant. Note that $\iota$ does not depend on the value of $k$ by symmetry.
We transform the right-hand side of this equality as
\#\label{eq::pf_lem_kernel_story1_eq1}
\int v \hat\PP^{\fk}_{\cV\given \cK}(v\given q) \ud v 
&= \iota \cdot \int_{\SSS^{d - 1}} v\cdot \frac{\sum^L_{\ell=1}\fk(k^\ell, q)\cdot\fk(v^\ell, v)}{\sum^L_{\ell=1}\fk(k^\ell, q)} \ud v \notag\\
&= \frac{\iota \cdot \sum^L_{\ell = 1} \fk(k^\ell, q)\cdot\int_{\SSS^{d - 1}} v \cdot\fk(v^\ell, v) \ud v }{\sum^L_{\ell=1}\fk(k^\ell, q)}.
\#
Thus, it suffices to calculate the integration term$\int_{\SSS^{d - 1}} v \cdot\fk(v^\ell, v) \ud v$.
To this end, we have the following lemma.
\begin{proposition}
\label{prop:calculation_of_integral}
Let $\fk(a, b) = \exp(a^\top b / \gamma)$ be the exponential kernel with a fixed $\gamma >0$. It holds for any $b\in\SSS^{d - 1}$ that
\$
\int_{\SSS^{d - 1}} a\cdot \fk(a, b)\ud a = C_1 \cdot b,
\$
where $C_1>0$ is an absolute constant.

\end{proposition}
\begin{proof}
See Section~\ref{app:calculation_of_integral} for a detailed proof.
\end{proof}
Thus, it holds for the right-hand side of \eqref{eq::pf_lem_kernel_story1_eq1} that
\$
\iota \cdot C_1 \cdot \frac{\sum^L_{\ell = 1} \fk(k^\ell, q)\cdot v^\ell}{\sum^L_{\ell=1}\fk(k^\ell, q)} = \iota \cdot C_1  \cdot V^\top\softmax(Kq / \gamma)=\iota \cdot C_1 \cdot \att_\sm(q, K, V),
\$
where the first equality follows from the definition of the softmax function and the second equality follows from the definition of the softmax attention.

The second step is to relate the right-hand side of \eqref{eq:smattreformulate} to conditional mean embedding. In fact, under the condition that $\hat\PP^{\fk}_{\cV\given \cK}(v\given q) \rightarrow \PP(v \given q)$ uniformly for any $q \in \SSS^{d_\tp -1}$ as $L\rightarrow \infty$, we have 
\$
\int v \hat\PP^{\fk}_{\cV\given \cK}(v\given q) \ud v \rightarrow \EE[\cV \given \cK = q] \qquad \text{as} \quad L\rightarrow \infty.
\$
Thus, we have that
\begin{align}
    \att_\sm(q, K, V)\rightarrow C\cdot\EE[\cV \given \cK = q] \qquad \text{as} \quad L\rightarrow \infty \label{eq:attcmeasy}
\end{align}
for some constant $C>0$. Combining \eqref{eq:attcmeasy} and \eqref{ieq:attdaggercme} and choosing $\lambda=L^{3/4}$, we complete the proof of Proposition~\ref{prop:cme-att-limit}.

\end{proof}

\section{Appendix for Section~\ref{sec:pretrain}}

\subsection{Supplemental Definitions for Markov Chains}\label{app:MC}
We follow the notations in \cite{paulin2015concentration}. Let $\Omega$ be a Polish space. The transition kernel for a time-homogeneous Markov chain $\{X_{i}\}_{i=1}^{\infty}$ supported on $\Omega$ is a probability distribution $\PP(x,\rmd y)$ for every $x\in\Omega$. Given $X_{1}=x_{1},\cdots,X_{t-1}=x_{t-1}$, the conditional distribution of $X_{t}$ equals $\PP(x_{t-1},\rmd y)$. A distribution $\pi$ is said to be a stationary distribution of this Markov chain if $\int_{x\in\Omega}\PP(x,\rmd y)\pi(\rmd x)=\pi(\rmd y)$. We adopt $\PP^{t}(x,\cdot)$ to denote the distribution of $X_{t}$ conditioned on $X_{1}=x$. The \emph{mixing time} of the chain is defined by
\begin{align*}
    d(t)=\sup_{x\in\Omega}\TV\big(P^{t}(x,\cdot),\pi\big), \quad t_{\rm mix}(\varepsilon)=\min\{t\,|\,d(t)\leq\varepsilon\},\quad t_{\rm mix}=t_{\rm mix}(1/4).
\end{align*}

\subsection{Proof of Theorem~\ref{thm:pretrainmle}}
\begin{proof}[Proof of Theorem~\ref{thm:pretrainmle}]
    Our proof mainly involves three steps.
    \begin{itemize}
        \item Error decomposition with the PAC-Bayes framework.
        \item Control each term in the error decomposition.
        \item Conclude the proof.
    \end{itemize}
    
    \textbf{Step 1: Error decomposition with the PAC-Bayes framework.}
    
    For ease of notation, we temporarily write $T_{\rmp}$ and $N_{\rmp}$ as $T$ and $N$, respectively. Recall that the pretraining dataset is $\calD=\{(S_{t}^{n},x_{t+1}^{n})\}_{n,t=1}^{N,T}$, which consists of $N$ trajectories (essays), and each essay have $T+1$ words. Given $S_{t}^{n}$, the next word is generated as $x_{t+1}^{n}\sim \PP(\cdot\,|\,S_{t}^{n})$, and $S_{t+1}^{n}=(S_{t}^{n},x_{t+1}^{n})$. Here, we construct a ghost sample $\tilde{\calD}=\{(\tilS_{t}^{n},\tilx_{t+1}^{n})\}_{n,t=1}^{N,T}$ as $\tilS_{t}^{n}=S_{t}^{n}$ and $\tilx_{t+1}^{n}\sim \PP(\cdot\,|\,\tilS_{t}^{n})$ independently from $\calD$. We define function $g(\theta)=L(\theta,D)-\log \bbE_{\tilde{\calD}}[\exp(L(\theta,\tilde{\calD}))\,|\,\calD]$, where 
    \begin{align*}
        L(\theta,\tilD)=-\frac{1}{4}\sum_{n=1}^{N}\sum_{t=1}^{T}\log\frac{\PP(\tilx_{t+1}^{n}\,|\,S_{t}^{n})}{\PP_{\theta}(\tilx_{t+1}^{n}\,|\,S_{t}^{n})}.
    \end{align*}
    For distributions $Q,P\in\Delta(\Theta)$, where $P$ can potentially depends on $\calD$, Lemma~\ref{lem:DV} shows that
    \begin{align*}
        \bbE_{P}\big[g(\theta)\big]\leq \KL(P\|Q)+\log\bbE_{Q}\big[\exp\big(g(\theta)\big)\big].
    \end{align*}
    Substituting the definition of $g(\theta)$ and taking expectation with respect to the distribution of $\calD$ on the both sides of the inequality, we can derive that
    \begin{align*}
        \bbE_{\calD}\bigg[\exp\Big\{\bbE_{P}\Big[L(\theta,\calD)-\log\bbE_{\tilde{\calD}}\big[ \exp\big( L(\theta,\tilde{\calD})\big)\,|\,\calD\big]\Big]-\KL(P\,\|\,Q)\Big\}\bigg]\leq 1.
    \end{align*}
    With Chernoff inequality, we can show that with probability at least $1-\delta$, the following holds
    \begin{align}
        -\bbE_{\theta\sim P}\Big[\log\bbE_{\tilde{\calD}}\big[ \exp\big( L(\theta,\tilde{\calD})\big)\,|\,\calD\big]\Big]\leq -\bbE_{P}\big[L(\theta,\calD)\big]+ \KL(P\,\|\,Q)+\log\frac{1}{\delta}.\label{ieq:1}
    \end{align}
    We first cope with the left-hand side of \eqref{ieq:1}.
    \begin{align}
        &-\bbE_{P}\Big[\log\bbE_{\tilde{\calD}}\big[ \exp\big( L(\theta,\tilde{\calD})\big)\,|\,\calD\big]\Big]\nonumber\\
        &\quad\geq -\frac{1}{2}\log\bbE_{\tilde{\calD}}\bigg[ \exp\bigg(-\frac{1}{2}\sum_{n=1}^{N}\sum_{t=1}^{T}\log\frac{\PP(\tilx_{t+1}^{n}\,|\,S_{t}^{n})}{\PP_{\htheta}(\tilx_{t+1}^{n}\,|\,S_{t}^{n})}\bigg)\,\bigg|\,\calD\bigg]\nonumber\\
        &\quad\qquad-\frac{1}{2}\bbE_{\theta\sim P}\bigg[\log\bbE_{\tilde{\calD}}\bigg[ \exp\bigg(-\frac{1}{2}\sum_{n=1}^{N}\sum_{t=1}^{T}\log\frac{\PP_{\htheta}(\tilx_{t+1}^{n}\,|\,S_{t}^{n})}{\PP_{\theta}(\tilx_{t+1}^{n}\,|\,S_{t}^{n})}\bigg)\,\bigg|\,\calD\bigg]\bigg]\nonumber\\
        &\quad=-\frac{1}{2}\sum_{n=1}^{N}\sum_{t=1}^{T}\log\bbE_{\tilx_{t+1}^{n}\sim \PP(\cdot\,|\,S_{t}^{n})}\bigg[ \exp\bigg(-\frac{1}{2}\log\frac{\PP(\tilx_{t+1}^{n}\,|\,S_{t}^{n})}{\PP_{\htheta}(\tilx_{t+1}^{n}\,|\,S_{t}^{n})}\bigg)\,\bigg|\,\calD\bigg]\nonumber\\
        &\quad\qquad-\frac{1}{2}\bbE_{\theta\sim P}\bigg[\log\bbE_{\tilde{\calD}}\bigg[ \exp\bigg(-\frac{1}{2}\sum_{n=1}^{N}\sum_{t=1}^{T}\log\frac{\PP_{\htheta}(\tilx_{t+1}^{n}\,|\,S_{t}^{n})}{\PP_{\theta}(\tilx_{t+1}^{n}\,|\,S_{t}^{n})}\bigg)\,\bigg|\,\calD\bigg]\bigg]\nonumber\\
        &\quad\geq \frac{1}{4}\sum_{n=1}^{N}\sum_{t=1}^{T}\TV\big(\PP(\cdot\,|\,S_{t}^{n}),\PP_{\htheta}(\cdot\,|\,S_{t}^{n})\big)^{2}-\frac{1}{2}\bbE_{\theta\sim P}\bigg[\log\bbE_{\tilde{\calD}}\bigg[ \exp\bigg(-\frac{1}{2}\sum_{n=1}^{N}\sum_{t=1}^{T}\log\frac{\PP_{\htheta}(\tilx_{t+1}^{n}\,|\,S_{t}^{n})}{\PP_{\theta}(\tilx_{t+1}^{n}\,|\,S_{t}^{n})}\bigg)\,\bigg|\,\calD\bigg]\bigg], \label{ieq:2}
    \end{align}
    where the first inequality results from the definition of $L(\theta,\calD)$ and Cauchy-Schwarz inequality, the equality results from that the transitions of $\tilx_{t+1}^{n}$ are independent given $\calD$, and the last inequality results from Lemma~\ref{lem:tvbound}. The second term in the right-hand side of \eqref{ieq:2} can be controlled if the distribution $P$ is chosen to concentrate around $\htheta$. This will be done in Step 2. Now we consider the right-hand side of \eqref{ieq:1}. For any $\theta^{*}\in\Theta$, we can decompose it as
    \begin{align}
        &-\bbE_{P}\big[L(\theta,\calD)\big]\nonumber\\
        &\quad = \bbE_{P}\bigg[\frac{1}{4}\sum_{n=1}^{N}\sum_{t=1}^{T}\log\frac{\PP(x_{t+1}^{n}\,|\,S_{t}^{n})}{\PP_{\theta^{*}}(x_{t+1}^{n}\,|\,S_{t}^{n})}+\log\frac{\PP_{\theta^{*}}(x_{t+1}^{n}\,|\,S_{t}^{n})}{\PP_{\htheta}(x_{t+1}^{n}\,|\,S_{t}^{n})}+\log\frac{\PP_{\htheta}(x_{t+1}^{n}\,|\,S_{t}^{n})}{\PP_{\theta}(x_{t+1}^{n}\,|\,S_{t}^{n})}\bigg]\nonumber\\
        &\quad\leq \frac{1}{4}\sum_{n=1}^{N}\sum_{t=1}^{T}\log\frac{\PP(x_{t+1}^{n}\,|\,S_{t}^{n})}{\PP_{\theta^{*}}(x_{t+1}^{n}\,|\,S_{t}^{n})}+\frac{1}{4}\sum_{n=1}^{N}\sum_{t=1}^{T}\bbE_{P}\bigg[\log\frac{\PP_{\htheta}(x_{t+1}^{n}\,|\,S_{t}^{n})}{\PP_{\theta}(x_{t+1}^{n}\,|\,S_{t}^{n})}\bigg],\label{ieq:3}
    \end{align}
    where the inequality results from the fact that $\htheta$ maximizes the likelihood. We will choose $\theta^{*}$ as the projection of $\PP$ onto $\{\PP_{\theta}\,|\,\theta\in\Theta\}$, i.e., $\PP_\theta^{*}$ is the best approximation of $\PP$ with respect to the KL divergence. Thus, the first term in the right-hand side of \eqref{ieq:3} is the approximation error. The second term in the right-hand side of \eqref{ieq:3} can be controlled in the same way as the second term in the right-hand side of \eqref{ieq:2}. Combining inequalities~\eqref{ieq:1}, \eqref{ieq:2}, and \eqref{ieq:3}, we have that
    \begin{align}
        &\frac{1}{4}\sum_{n=1}^{N}\sum_{t=1}^{T}\TV\big(\PP(\cdot\,|\,S_{t}^{n}),\PP_{\htheta}(\cdot\,|\,S_{t}^{n})\big)^{2}\nonumber\\
        &\quad\leq \underbrace{\frac{1}{2}\bbE_{\theta\sim P}\bigg[\log\bbE_{\tilde{\calD}}\bigg[ \exp\bigg(-\frac{1}{2}\sum_{n=1}^{N}\sum_{t=1}^{T}\log\frac{\PP_{\htheta}(x_{t+1}^{n}\,|\,S_{t}^{n})}{\PP_{\theta}(x_{t+1}^{n}\,|\,S_{t}^{n})}\bigg)\,\bigg|\,\calD\bigg]\bigg]+\frac{1}{4}\sum_{n=1}^{N}\sum_{t=1}^{T}\bbE_{P}\bigg[\log\frac{\PP_{\htheta}(x_{t+1}^{n}\,|\,S_{t}^{n})}{\PP_{\theta}(x_{t+1}^{n}\,|\,S_{t}^{n})}\bigg]}_{(\rmI)}\nonumber\\
        &\quad\qquad+\underbrace{\frac{1}{4}\sum_{n=1}^{N}\sum_{t=1}^{T}\log\frac{\PP(x_{t+1}^{n}\,|\,S_{t}^{n})}{\PP_{\theta^{*}}(x_{t+1}^{n}\,|\,S_{t}^{n})}}_{(\rm II)}+ \underbrace{\KL(P\,\|\,Q)}_{(\rm III)}+\log\frac{1}{\delta},\label{ieq:4}
    \end{align}
    where term (I) is the fluctuation error induced by $\theta\sim P$, term (II) is the approximation error, and term (III) is the KL divergence between $P$ and $Q$.
    
    \textbf{Step 2: Control each term in the error decomposition.}
    
    We first consider term (I). We need to quantify the fluctuation of $\PP_{\theta}$ when $\theta$ is changing.
    \begin{proposition}\label{prop:parafluc}
        For any input $X\in\bbR^{L\times d}$ and $\theta,\tiltheta\in\Theta$, we have that
        \begin{align*}
            &\TV\big(\PP_{\theta}(\cdot\,|\,X),\PP_{\tiltheta}(\cdot\,|\,X)\big)\\
            &\quad\leq \frac{2}{\tau}\big\|A^{(D+1),\top}-\tilA^{(D+1),\top}\big\|_{1,2}+\sum_{t=1}^{D}\alpha_{t}(\beta_{t}+\iota_{t}+\kappa_{t}+\rho_{t}),
        \end{align*}
        where 
        \begin{align*}
            \alpha_{t}&=\frac{2}{\tau}B_{A}(1+B_{A,1}\cdot B_{A,2})\big(1+hB_{V}(1+4 B_{Q}B_{K})\big)^{D-t}\\
            \beta_{t}&=|\gamma_{2}^{(t)}-\tilgamma_{2}^{(t)}|+(1+B_{A,1}\cdot B_{A,2})\cdot\big(1+(\|X^\top\|_{2,\infty}-1)\bbI_{t=1}\big)\cdot|\gamma_{1}^{(t)}-\tilgamma_{1}^{(t)}|\\
            \iota_{t}&=B_{A,2}\cdot\|A_{1}^{(t)}-\tilA_{1}^{(t)}\|_{\rmF}+B_{A,1}\cdot\|A_{2}^{(t)}-\tilA_{2}^{(t)}\|_{\rmF}\\
            \kappa_{t}&=(1+B_{A,1}\cdot B_{A,2})\cdot\big(1+(\|X^\top\|_{2,\infty}-1)\bbI_{t=1}\big)\cdot\sum_{i=1}^{h}\big\|W_{i}^{V,(t)}-\tilde{W}_{i}^{V,(t)}\|_{\rmF}\\
            \rho_{t}&=2(1+B_{A,1}\cdot B_{A,2})\cdot\big(1+(\|X^\top\|_{2,\infty}-1)\bbI_{t=1}\big)\cdot B_{V}\\
            &\quad\qquad\cdot\sum_{i=1}^{h}B_{K}\cdot\|W_{i}^{Q,(t+1)}-\tilde{W}_{i}^{Q,(t+1)}\|_{\rmF}+B_{Q}\cdot\|W_{i}^{K,(t+1)}-\tilde{W}_{i}^{K,(t+1)}\|_{\rmF}
        \end{align*}
        for all $t\in[D]$.
    \end{proposition}
    \begin{proof}[Proof of Proposition~\ref{prop:parafluc} ]
        See Appendix~\ref{app:parafluc}.
    \end{proof}
    With the help of Proposition~\ref{prop:parafluc}, we set the distribution $P$ as 
    \begin{align}
        P&=\prod_{t=1}^{D+1}\calL_{P}\big(\theta^{(t)}\big)\label{eq:distP}\\
        \calL_{P}\big(\theta^{(D+1)}\big)&=\unif\Big(\bbB\big(\hatA^{(D+1)},r^{(D+1)},\|\cdot\|_{1,2}\big)\Big)\nonumber\\
        \calL_{P}\big(\theta^{(t)}\big)&=\unif\Big(\bbB\big(\hat\gamma_{1}^{(t)},r_{\gamma,1}^{(t)},|\cdot|\big)\Big)\cdot\unif\Big(\bbB\big(\hat\gamma_{2}^{(t)},r_{\gamma,2}^{(t)},|\cdot|\big)\Big)\cdot\calL_{P}(A^{(t)})\cdot\calL_{P}(W^{(t)})\nonumber\\
        \calL_{P}(A^{(t)})&=\unif\Big(\bbB\big(\hatA_{1}^{(t)},r_{A,1}^{(t)},\|\cdot\|_{\rmF}\big)\Big)\cdot\unif\Big(\bbB\big(\hatA_{2}^{(t)},r_{A,2}^{(t)},\|\cdot\|_{\rmF}\big)\Big)\nonumber\\
        \calL_{P}(W^{(t)})&=\prod_{i=1}^{h}\unif\Big(\bbB\big(\hatW_{i}^{Q,(t)},r_{Q}^{(t)},\|\cdot\|_{\rmF}\big)\Big)\cdot\unif\Big(\bbB\big(\hatW_{i}^{K,(t)},r_{K}^{(t)},\|\cdot\|_{\rmF}\big)\Big)\cdot\unif\Big(\bbB\big(\hatW_{i}^{V,(t)},r_{V}^{(t)},\|\cdot\|_{\rmF}\big)\Big)\nonumber
    \end{align}
    for $t\in[D]$, where $\unif$ denotes the uniform distribution on the set, $\bbB(a,r,\|\cdot\|)=\{x\,|\, \|x-a\|\leq r\}$ denotes the ball centered in $a$ with radius $r$, the radius is set as
    \begin{alignat*}{2}
        r_{\gamma,1}^{(t)}&=R^{-1}(1+B_{A,1}\cdot B_{A,2})^{-1}\alpha_{t}^{-1}/NT,&\quad r_{\gamma,2}^{(t)}&=R^{-1}\alpha_{t}^{-1}/NT\\
        r_{A,1}^{(t)}&=R^{-1}B_{A,2}^{-1}\alpha_{t}^{-1}/NT,&\quad r_{A,2}^{(t)}&=R^{-1}B_{A,1}^{-1}\alpha_{t}^{-1}/NT,\\
        r_{V}^{(t)}&=R^{-1}h^{-1}(1+B_{A,1}\cdot B_{A,2})^{-1}\alpha_{t}^{-1}/NT,&\quad r_{Q}^{(t)}&=R^{-1}h^{-1}(1+B_{A,1}\cdot B_{A,2})^{-1}B_{V}^{-1}B_{K}^{-1}\alpha_{t}^{-1}/NT\\
        r_{K}^{(t)}&=R^{-1}h^{-1}(1+B_{A,1}\cdot B_{A,2})^{-1}B_{V}^{-1}B_{Q}^{-1}\alpha_{t}^{-1}/NT, &\quad r^{(D+1)}&=\tau B_{A}^{-1}/NT.
    \end{alignat*}
    Under this assignment, we now bound $|\log \PP_{\htheta}(x\,|\,S)/\PP_{\theta}(x\,|\,S)|$ for any $S\in\bbR^{L\times d}$ and $x\in\bbR^{d_{y}}$. We first note that 
    \begin{align}
        \PP_{\htheta}(x\,|\,S)\geq b_{y}=(1+d_{y}\exp (B_{A}/\tau))^{-1} \label{ieq:probbd}
    \end{align}
    for any $S$ and $x$. This results from the fact that
    \begin{align*}
        \bigg\|\frac{1}{L\tau}\bbI_{L}^\top X^{(D)}A^{(D+1)}\bigg\|_{1}\leq \big\|A^{(D+1),\top}\big\|_{1,2}\leq B_{A}.
    \end{align*}
    If $\TV(\PP_{\theta}(\cdot\,|\,S),\PP_{\tiltheta}(\cdot\,|\,S))=\varepsilon\leq b_{y}/2$, some basic calculations show that
    \begin{align*}
        \frac{b_{y}}{b_{y}+\varepsilon}\leq\frac{\PP_{\htheta}(x\,|\,S)}{\PP_{\theta}(x\,|\,S)}\leq 1+\frac{2\varepsilon}{b_{y}}.
    \end{align*}
    Thus, we have
    \begin{align*}
        \bigg|\log \frac{\PP_{\htheta}(x\,|\,S)}{\PP_{\theta}(x\,|\,S)}\bigg|\leq \frac{2\varepsilon}{b_{y}}=\cO\bigg(\frac{1}{NT}\bigg)\quad  \text{ for } P\text{ a.s. }
    \end{align*}
    Based on this, we conclude that 
    \begin{align}
        (\rmI)=\cO(1)\label{ieq:I}.
    \end{align}
    Next, we control term (III) in \eqref{ieq:4}. We take $Q$ as
    \begin{align}
        Q&=\prod_{t=1}^{D+1}\calL_{Q}\big(\theta^{(t)}\big)\label{eq:distQ}\\
        \calL_{Q}\big(\theta^{(D+1)}\big)&=\unif\Big(\bbB\big(0,B_{A},\|\cdot\|_{1,2}\big)\Big)\nonumber\\
        \calL_{Q}\big(\theta^{(t)}\big)&=\unif\Big(\bbB\big(1/2,1/2,|\cdot|\big)\Big)\cdot\unif\Big(\bbB\big(1/2,1/2,|\cdot|\big)\Big)\cdot\calL_{Q}(A^{(t)})\cdot\calL_{Q}(W^{(t)})\nonumber\\
        \calL_{Q}(A^{(t)})&=\unif\Big(\bbB\big(0,B_{A,1},\|\cdot\|_{\rmF}\big)\Big)\cdot\unif\Big(\bbB\big(0,B_{A,2},\|\cdot\|_{\rmF}\big)\Big)\nonumber\\
        \calL_{Q}(W^{(t)})&=\prod_{i=1}^{h}\unif\Big(\bbB\big(0,B_{Q},\|\cdot\|_{\rmF}\big)\Big)\cdot\unif\Big(\bbB\big(0,B_{K},\|\cdot\|_{\rmF}\big)\Big)\cdot\unif\Big(\bbB\big(0,B_{V},\|\cdot\|_{\rmF}\big)\Big).\nonumber
    \end{align}
    Then the KL divergence between $P$ and $Q$ is
    \begin{align}
        \KL(P\,\|\,Q)=\cO\Big((D^{2}\cdot d\cdot (d_{F}+d_{h}+d)+d\cdot d_{y})\cdot\log\big( 1+NT\tau^{-1}RhB_{A}B_{A,1}B_{A,2}B_{Q}B_{K}B_{V}\big)\Big).\label{ieq:III}
    \end{align}
    
    Finally, we control term (II) in \eqref{ieq:4}. This term can be controlled as
    \begin{align*}
        &\frac{1}{NT}\sum_{n=1}^{N}\sum_{t=1}^{T}\log\frac{\PP(x_{t+1}^{n}\,|\,S_{t}^{n})}{\PP_{\theta^{*}}(x_{t+1}^{n}\,|\,S_{t}^{n})}\\
        &\quad = \frac{1}{NT}\sum_{n=1}^{N}\sum_{t=1}^{T}\log\frac{\PP(x_{t+1}^{n}\,|\,S_{t}^{n})}{\PP_{\theta^{*}}(x_{t+1}^{n}\,|\,S_{t}^{n})}-\frac{1}{NT}\sum_{n=1}^{N}\sum_{t=1}^{T}\bbE_{S_{t}^{n}}\KL\big(\PP(\cdot\,|\,S_{t}^{n})\,\|\,\PP_{\theta^{*}}(\cdot\,|\,S_{t}^{n})\big)\\
        &\quad\qquad +\frac{1}{NT}\sum_{n=1}^{N}\sum_{t=1}^{T}\bbE_{S_{t}^{n}}\KL\big(\PP(\cdot\,|\,S_{t}^{n})\,\|\,\PP_{\theta^{*}}(\cdot\,|\,S_{t}^{n})\big).
    \end{align*}
    The first two terms in the right-hand side of the equality is the generalization error, which can be bounded with Lemma~\ref{lem:MCconcen}. With Assumption~\ref{assump:positive}, we note that
    \begin{align}\label{eq:logbound}
        \bigg|\log \frac{\PP(x\,|\,S)}{\PP_{\theta^{*}}(x\,|\,S)}\bigg|\leq b^{*}= \log\max\{c_{0}^{-1},b_{y}^{-1}\},
    \end{align}
    so the function satisfies the condition in Lemma~\ref{lem:MCconcen} with $c_{i}=2b^{*}$. Using the moment generating function bound in Lemma~\ref{lem:MCconcen} and Chernoff bound, we have that
    \begin{align}
        \frac{1}{NT}\sum_{n=1}^{N}\sum_{t=1}^{T}\log\frac{\PP(x_{t+1}^{n}\,|\,S_{t}^{n})}{\PP_{\theta^{*}}(x_{t+1}^{n}\,|\,S_{t}^{n})}-\frac{1}{NT}\sum_{n=1}^{N}\sum_{t=1}^{T}\bbE_{S_{t}^{n}}\KL\big(\PP(\cdot\,|\,S_{t}^{n})\,\|\,\PP_{\theta^{*}}(\cdot\,|\,S_{t}^{n})\big)\leq \sqrt{\frac{t_{\rm min }b^{*,2}}{2NT}}\log\frac{1}{\delta}\label{ieq:II}
    \end{align}
    with probability at least $1-\delta$.
    
    \textbf{Step 3: Conclude the proof.}
    
    Combining inequalities~\eqref{ieq:4}, \eqref{ieq:I}, \eqref{ieq:III}, and \eqref{ieq:II}, we have that
    \begin{align*}
        &\frac{1}{NT}\sum_{n=1}^{N}\sum_{t=1}^{T}\TV\big(\PP(\cdot\,|\,S_{t}^{n}),\PP_{\htheta}(\cdot\,|\,S_{t}^{n})\big)\nonumber\\
        &\quad\leq\sqrt{\frac{1}{NT}\sum_{n=1}^{N}\sum_{t=1}^{T}\TV\big(\PP(\cdot\,|\,S_{t}^{n}),\PP_{\htheta}(\cdot\,|\,S_{t}^{n})\big)^{2}}\nonumber\\
        &\quad=\cO\bigg(\frac{t_{\rm min }^{1/4}}{(NT)^{1/4}}\log\frac{1}{\delta}+\frac{\sqrt{D^{2} d (d_{F}+d_{h}+d)+d\cdot d_{y}}}{\sqrt{NT}}\cdot\log\big( 1+NT\bar{B}\big)\\
        &\quad\qquad+\inf_{\theta^{*}\in\Theta}\sqrt{\frac{1}{NT}\sum_{n=1}^{N}\sum_{t=1}^{T}\bbE_{S_{t}^{n}}\KL\big(\PP(\cdot\,|\,S_{t}^{n})\,\|\,\PP_{\theta^{*}}(\cdot\,|\,S_{t}^{n})\big)}\bigg),
    \end{align*}
    where we take $\theta^{*}$ as the best approximation parameters. Finally, we will change the left-hand side of this inequality to the expectation of it. In fact, we have that
    \begin{proposition}\label{prop:pacbayes}
        Let $\calF$ be the collection of functions of $f:\bbR^{n}\rightarrow\bbR$, and we assume that $|f|\leq b $ for any function $f\in\calF$. For a Markov chain $X=(X_{1},\cdot,X_{N})$, we define $f(X)=\sum_{i=1}^{N}f(X_{i})/N$. The mixing time of this Markov chain is denoted as $t_{\rm mix}(\varepsilon)$. Given a distribution $Q$ on $\calF$, with probability at least $1-\delta$, we have
        \begin{align}
            \Bigl|\bbE_{P}\Bigl[\bbE_{X}\big[f(X)\big]-f(X)\Bigr]\Bigr|\leq \sqrt{\frac{b^{2}\cdot t_{\rm min}}{2\log 2 N}}\biggl[\KL(P\,\|\,Q)+\log\frac{4}{\delta}\biggr],\nonumber
        \end{align}
        for any distribution $P$ on $\calF$ simultaneously with probability at least $1-\delta$, where 
        \begin{align*}
            t_{\rm min}=\inf_{0\leq \varepsilon<1}t_{\rm mix}(\varepsilon)\cdot\bigg(\frac{2-\varepsilon}{1-\varepsilon}\bigg)^{2}.
        \end{align*}
    \end{proposition}
    \begin{proof}[Proof of Proposition~\ref{prop:pacbayes}]
        See Appendix~\ref{app:pacbayes}.
    \end{proof}
    We note that Proposition~\ref{prop:pacbayes} is indeed an uniform convergence bound, since it holds simultaneously for all $P$. Thus, we can set $P$ and $Q$ as those in equalities~\eqref{eq:distP} and \eqref{eq:distQ}, then we have that
    \begin{align*}
        &\frac{1}{NT}\sum_{n=1}^{N}\sum_{t=1}^{T}\bbE_{S_{t}^{n}}\Big[\TV\big(\PP(\cdot\,|\,S_{t}^{n}),\PP_{\htheta}(\cdot\,|\,S_{t}^{n})\big)\Big]-\frac{1}{NT}\sum_{n=1}^{N}\sum_{t=1}^{T}\TV\big(\PP(\cdot\,|\,S_{t}^{n}),\PP_{\htheta}(\cdot\,|\,S_{t}^{n})\big)\\
        &\quad=\cO\bigg(\frac{\sqrt{t_{\rm min}}}{\sqrt{NT}}\Big(\barD\log(1+NT\barB)+\log\frac{1}{\delta}\Big)\bigg).
    \end{align*}
    Thus, we have that
    \begin{align*}
        &\frac{1}{NT}\sum_{n=1}^{N}\sum_{t=1}^{T}\bbE_{S_{t}^{n}}\Big[\TV\big(\PP(\cdot\,|\,S_{t}^{n}),\PP_{\htheta}(\cdot\,|\,S_{t}^{n})\big)\Big]\\
        &\quad=\cO\bigg(\frac{t_{\rm min }^{1/4}}{(NT)^{1/4}}\log\frac{1}{\delta}+\frac{\sqrt{t_{\rm min}}}{\sqrt{NT}}\Big(\barD\log(1+NT\barB)+\log\frac{1}{\delta}\Big)\\
        &\quad\qquad+\inf_{\theta^{*}\in\Theta}\sqrt{\frac{1}{NT}\sum_{n=1}^{N}\sum_{t=1}^{T}\bbE_{S_{t}^{n}}\KL\big(\PP(\cdot\,|\,S_{t}^{n})\,\|\,\PP_{\theta^{*}}(\cdot\,|\,S_{t}^{n})\big)}\bigg).
    \end{align*}
    We conclude the proof of Theorem~\ref{thm:pretrainmle}.
    
\end{proof}
\subsection{Formal Statement and Proof of Proposition~\ref{prop:approxerrmle}}\label{app:formalapprox}
Denote the alphabet of the language as $\frakX\subseteq\bbR$ ($d=1$), then the conditional distribution $\PP^{*}$ can be viewed as a function $g^{*}:\frakX^{L}\rightarrow \bbR^{d_{y}}$, where $L$ is the maximal length of a sentence, and the output is the distribution of the next word. Since $\calA$ is finite, Theorem 2 in \cite{zaheer2017deep} shows that there exist $\rho^{*}:\bbR\rightarrow\bbR^{d_{y}}$ and $\phi^{*}:\frakX\rightarrow\bbR$ such that
\begin{align*}
    g^{*}(X)=\rho^{*}\bigg(\frac{1}{L}\sum_{i=1}^{L}\phi^{*}(x_{i})\bigg),
\end{align*}
where $X=[x_{1},\cdots,x_{L}]$. The $i^{\rm th}$ component of $\rho^{*}$ is denoted as $\rho_{i}^{*}$ for $i\in[d_{y}]$. For a function $f$ defined on $\Omega$, the $L^{\infty}$ norm of it is defined as $\|f\|_{\infty}=\sup_{x\in\Omega}|f(x)|$. The set of the real-valued smooth functions on it is denoted as $\calS^{\infty}(\Omega,\bbR)$, Then we denote the set of the smooth functions with bounded derivatives as
\begin{align*}
    \calS_{B}=\Big\{f\in\calS^{\infty}([-B,B],\bbR)\,|\, \big\|f^{(n)}(x)\big\|\leq n! \text{ for all }n\in\bbN\Big\},
\end{align*}
where $f^{(n)}$ is the $n^{\rm th}$-order derivative of $f$.

\begin{assumption}\label{assump:smoothmle}
    There exists $B>0$ such that $\phi^{*},\tau\log\rho_{i}^{*}\in\calS_{B}$ for $i\in[d_{y}]$.
\end{assumption}
This assumption states that the function $g^{*}$ is smooth enough for transformers to approximate.
\begin{proposition}\label{prop:fapproxerrmle}
    Under Assumptions~\ref{assump:positive} and \ref{assump:smoothmle}, if $d_{F}\geq 16d_{y}$, $B_{A,1}\geq 16 R d_{y}$, $B_{A,2}\geq d_{F}$ $B_{A}\geq \sqrt{d_{y}}$, and $B_{V}\geq \sqrt{d}$, then
    \begin{align*}
        \max_{\|S^\top\|_{2,\infty}\leq R}\KL\big(\PP^{*}(\cdot\,|\,S)\,\|\,\PP_{\theta^{*}}(\cdot\,|\,S)\big)=\cO\bigg( d_{y}\exp\bigg(-\frac{D^{1/4}}{\sqrt{C^{2}B^{2}\log B_{A,1}}}\bigg)\bigg),
    \end{align*}
    for some constant $C>0$.
\end{proposition}

\begin{proof}[Proof of Proposition~\ref{prop:fapproxerrmle}]
    Our proof mainly involves three steps.
    \begin{itemize}
        \item Build the high-level transformer approximator for $g^{*}$.
        \item Build the approximators in the transformer for $\phi^{*}$ and $\rho_{i}^{*}$ separately.
        \item Conclude the proof.
    \end{itemize}
    
    \textbf{Step 1: Build the high-level transformer approximator for $g^{*}$}
    
    Without loss of generality, we assume that $B>1$ in Assumption~\ref{assump:smooth}. To approximate $\phi^{*}$, we ignore the attention module in the transformer by setting $W_{i}^{V,(t)}=0$, $\gamma_{1}^{(t)}=1$, $\gamma_{2}^{(t)}=0$ for all $i\in[h]$. We further set $A_{2}^{(t)}=I_{d_{F}}\in\bbR^{d_{F}\times d_{F}}$, which is the identity matrix. The network structure now is
    \begin{align*}
        X^{(t+1)}=\Pi_{\rm{norm}}\big[\relu(X^{(t)}A_{1}^{(t+1)}+b^{(t+1)}\cdot\bbI_{L})\big],
    \end{align*}
    where $b^{(t+1)}\in\bbR$ is the bias term. In Step 2, we will use this fully-connected network to approximate $\phi^{*}$. To approximate the average $\frac{1}{L}\sum_{i=1}^{L}\phi^{*}(x_{i})$, we take $W_{i}^{Q,(t)}=0$, $W_{i}^{K,(t)}=0$, and $W_{i}^{V,(t)}=\bbI_{d}$, $\gamma_{1}^{(t)}=0$, $\gamma_{2}^{(t)}=1$, $A_{2}^{(t)}=0$. After this average aggregation, we still take $W_{i}^{V,(t)}=0$, $\gamma_{1}^{(t)}=1$, $\gamma_{2}^{(t)}=0$ for all $i\in[h]$ and $A_{2}^{(t)}=I_{d_{F}}\in\bbR^{d_{F}\times d_{F}}$ to approximate $\rho_{i}^{*}$ for $i\in[d_{y}]$. We stack the approximators for $\rho_{i}^{*}$ to approximate $\rho^{*}$, multiplying the width of the networks by $d_{F}$.
    
    \textbf{Step 2: Build the approximators in the transformer for $\phi^{*}$ and $\rho_{i}^{*}$ separately.}
    
    In the first and the $D^{\rm th}$ layer, we take $A_{1}^{(1),\prime}=A_{1}^{(1)}/R$ and $A_{1}^{(D),\prime}=A_{1}^{(D)}\cdot R$ to normalize and retrieve the magnitudes of inputs, where $R$ is the range of the inputs. This will keep the magnitudes of the intermediate outputs small. Next, we will use Lemma~\ref{lem:approx} to construct the networks. In the proof of Lemma~\ref{lem:approx}, the norm of the outputs of the intermediate layers do not excess the range of the inputs, so the layer normalization in our networks will not influence the constructed approximators. In this case, we can respectively approximate $\phi^{*}$ and $\rho_{i}^{*}$ with fully-connected networks $\Psi_{\phi^{*}}$ and $\Psi_{\rho_{i}^{*}}$ for $i\in [d_{y}]$ as
    \begin{align*}
        \|\phi^{*}-\Psi_{\phi^{*}}\|_{\infty}\leq\varepsilon_{\phi},\quad \|\rho_{i}^{*}-\Psi_{\rho_{i}^{*}}\|_{\infty}\leq\varepsilon_{\rho} \text{ for }i\in[d_{y}],
    \end{align*}
    where the depth $D(\cdot)$, the width $W(\cdot)$, and the maximal weight $B(\cdot)$ of the networks satisfy that
    \begin{align*}
        &D^{\prime}=D(\Psi_{\phi^{*}})\leq C\cdot B\cdot(\log \varepsilon_{\phi}^{-1})^{2}+\log B,\quad D^{\prime\prime}=\max_{i\in [d_{y}]}D(\Psi_{\rho_{i}^{*}})\leq C\cdot B\cdot(\log \varepsilon_{\rho}^{-1})^{2}+\log B,\\
        &W(\Psi_{\phi^{*}})\leq 16,\quad W(\Psi_{\rho_{i}^{*}})\leq 16,\quad B(\Psi_{\phi^{*}})\leq 1,\quad B(\Psi_{\rho_{i}^{*}})\leq 1
    \end{align*}
    for some constant $C>0$. The bounds for width and maximal weight require that $d_{F}\geq 16d_{y}$ and $B_{A,1}\geq \sqrt{d_{F}\cdot d_{F}}\geq 16d_{y}$. Then we have that for any $X=(x_{1},\cdots,x_{L})$
    \begin{align*}
        &\bigg\|\rho^{*}\bigg(\frac{1}{L}\sum_{i=1}^{L}\phi^{*}(x_{i})\bigg)-\Psi_{\rho^{*}}\bigg(\frac{1}{L}\sum_{i=1}^{L}\Psi_{\phi^{*}}(x_{i})\bigg)\bigg\|_{1}\\
        &\quad\leq \bigg\|\rho^{*}\bigg(\frac{1}{L}\sum_{i=1}^{L}\phi^{*}(x_{i})\bigg)-\Psi_{\rho^{*}}\bigg(\frac{1}{L}\sum_{i=1}^{L}\phi^{*}(x_{i})\bigg)\bigg\|_{1}+\bigg\|\Psi_{\rho^{*}}\bigg(\frac{1}{L}\sum_{i=1}^{L}\phi^{*}(x_{i})\bigg)-\Psi_{\rho^{*}}\bigg(\frac{1}{L}\sum_{i=1}^{L}\Psi_{\phi^{*}}(x_{i})\bigg)\bigg\|_{1}\\
        &\quad \leq d_{y}\varepsilon_{\phi}+d_{y}\cdot(B_{A,1})^{D^{\prime\prime}}\cdot \varepsilon_{\phi},
    \end{align*}
    where the first inequality results from the triangle inequality, $(B_{A,1})^{D^{\prime\prime}}$ in the second inequality results from the error propagation through a depth-$D^{\prime\prime}$ network. For the whole network, we have that
    \begin{align*}
        D^{\prime}+D^{\prime\prime}\leq D.
    \end{align*}
    We take that $D^{\prime}=D/2+D^{3/4}$ and $D^{\prime\prime}=\sqrt{D}/(\sqrt{C\cdot B}\log B_{A,1})$ for the constant $C$ in Lemma~\ref{lem:approx}. Then for $D>3$, we have that
    \begin{align*}
        &\bigg\|\rho^{*}\bigg(\frac{1}{L}\sum_{i=1}^{L}\phi^{*}(x_{i})\bigg)-\Psi_{\rho^{*}}\bigg(\frac{1}{L}\sum_{i=1}^{L}\Psi_{\phi^{*}}(x_{i})\bigg)\bigg\|_{1}= \cO\bigg( d_{y}\exp\bigg(-\frac{D^{1/4}}{\sqrt{C^{2}B^{2}\log B_{A,1}}}\bigg)\bigg).
    \end{align*}
    
    \textbf{Step 3: Conclude the proof.}
    
    We denote $\Psi_{\rho^{*}}(\sum_{i=1}^{L}\Psi_{\phi^{*}}(x_{i})/L)$ as $\PP_{\theta^{*}}$. Then if $\TV(\PP(\cdot\,|\,X),\PP_{\theta^{*}}(\cdot\,|\,X))=\varepsilon\leq c_{0}/2$, some basic calculations show that
    \begin{align*}
        \frac{c_{0}}{c_{0}+\varepsilon}\leq\frac{\PP(x\,|\,S)}{\PP_{\theta^{*}}(x\,|\,S)}\leq 1+\frac{2\varepsilon}{c_{0}}.
    \end{align*}
    Thus, we have
    \begin{align*}
        \max_{\|S^\top\|_{2,\infty}\leq R}\KL\big(\PP(\cdot\,|\,S)\,\|\,\PP_{\theta^{*}}(\cdot\,|\,S)\big)\leq \frac{2\varepsilon}{c_{0}}=\cO\bigg( d_{y}\exp\bigg(-\frac{D^{1/4}}{\sqrt{C^{2}B^{2}\log B_{A,1}}}\bigg)\bigg).
    \end{align*}
    
\end{proof}

\subsection{Pretraining Results for $\ell_{2}$ Loss}\label{app:ell2}
\subsubsection{Pretraining Algorithm with $\ell_{2}$ Loss}
Training with $\ell_{2}$ loss is common in the \ac{cv} community, e.g. \cite{radford2021learning}. The network structure is largely similar to those in \cite{brown2020language} and \cite{devlin2018bert}. Here, we modify the network structure of the last layer. The network derives the final output as $Y^{(D+1)}=\frac{1}{L}\bbI_{L}^\top X^{(D)}A^{(D+1)}$, where $\bbI_{L}\in\bbR^{L}$ is the vector with all ones, $A^{(D+1)}\in\bbR^{d\times d_{y}}$. The parameters in each layer are $\theta^{(t)}=(\gamma_{1}^{(t)},\gamma_{2}^{(t)},W^{(t)},A^{(t)})$ for $t\in[D]$, and $\theta^{(D+1)}=A^{(D+1)}$, and the parameters of the whole network is $\theta=(\theta^{(1)},\cdots,\theta^{(D+1)})$. Similar to Section~\ref{sec:pretrainalgo}, we consider the transformer with bounded weights. The set of parameters is
\begin{align*}
    \Theta&=\Big\{\theta\,|\, \big\|A^{(D+1)}\big\|_{\rmF}\leq B_{A},\max\big\{\big|\gamma_{1}^{(t)}\big|,\big|\gamma_{2}^{(t)}\big|\big\}\leq 1, \big\|A_{1}^{(t)}\big\|_{\rmF}\leq B_{A,1},\big\|A_{2}^{(t)}\big\|_{\rmF}\leq B_{A,2},\\
    &\big\|W_{i}^{Q,(t)}\big\|_{\rmF}\leq B_{Q},\big\|W_{i}^{K,(t)}\big\|_{\rmF}\leq B_{K},\big\|W_{i}^{V,(t)}\big\|_{\rmF}\leq B_{V} \text{ for all }t\in[D],i\in[h]\Big\},
\end{align*}
where $B_{A}$, $B_{A,1}$, $B_{A,2}$, $B_{Q}$, $B_{K}$, and $B_{V}$ are the bounds of parameter. We only consider the non-trivial case where these bounds are larger than $1$, otherwise the magnitude of the output in $D^{\rm th}$ layer decades exponentially with growing depth. We denote the transformer with parameter $\theta$ as $f_{\theta}$.

In such case, we focus on the pretraining setting in \ac{cv} tasks, i.e., the pretraining set $\calD=\{(S^{i},x^{i})\}_{i=1}^{N}$ consists of i.i.d. pairs. The underlying distribution is denoted as $(S,x)\sim \mu\in\Delta(\frakX^{*}\times\frakX)$. In such case, $d=d_{y}$, i.e., the transformer directly predicts the musked token. The training algorithm is
\begin{align}
    \htheta=\argmin_{\theta\in\Theta} \frac{1}{N}\sum_{i=1}^{N}\big\|x^{i}-f_{\theta}(S^{i})\big\|_{2}^{2}\label{algo:pretrain}
\end{align}
From the population version of \eqref{algo:pretrain}, it is easy to see that the function $f^{*}(S)=\bbE[x\,|\,S]$ achieves the minimal population error, where the conditional expectation is defined from $\mu$. In the following, we will quantify the error between $f_{\htheta}$ and $f^{*}$.

\subsubsection{Performance Guarantee for Pretraining with $\ell_{2}$ Loss}
We first state the assumptions for the pretraining setting.
\begin{assumption}\label{assump:bounded}
    There exists a constant $R>0$ such that for $(S,x)\sim\mu$, we have $\|S^\top\|_{2,\infty}\leq R$ and $\|x\|_{2}\leq B_{x}$ almost surely.
\end{assumption}
Then the performance guarantee for the pretraining result $\htheta$ can be derived as following.
\begin{theorem}\label{thm:pretrain}
    Let $\barB=B_{x}RhB_{A}B_{A,1}B_{A,2}B_{Q}B_{K}B_{V}$ and $\barD=D^{2} d (d_{F}+d_{h}+d)+d\cdot d_{y}$. If Assumption~\ref{assump:bounded} holds, the pretrained model $f_{\htheta}$ by the algorithm in \eqref{algo:pretrain} satisfies
    \begin{align*}
        \bbE_{S,x}\Big[\big\|f^{*}(S)-f_{\htheta}(S)\big\|_{2}^{2}\Big]\leq\underbrace{\frac{3}{2}\min_{\theta\in\Theta}\bbE\Big[\big\|f^{*}(S)-f_{\theta}(S)\big\|_{2}^{2}\Big]}_{\text{approximation error}}+\underbrace{\cO\bigg(\frac{B_{x}^{2}}{N}\biggl[\barD\log(1+N\barB)+\log\frac{2}{\delta}\biggr]\bigg)}_{\text{generalization error}}
    \end{align*}
    with probability at least $1-\delta$.
\end{theorem}
The first term is the approximation error. It measures the proximity between the nominal function $f^{*}$ and the functions induced by the parameter set $\Theta$. The second term is the generalization error. Similar as Theorem~\ref{thm:pretrainmle}, the generalization error is independent of the token sequence length.

Since the neural networks are universal approximators, we will explicitly approximate $f^{*}$ from the transformer function class. Theorem 2 in \cite{zaheer2017deep} shows that there exist $\rho^{*}:\bbR\rightarrow\bbR^{d_{y}}$ and $\phi^{*}:\bbR\rightarrow\bbR$ such that
\begin{align*}
    f^{*}(X)=\rho^{*}\bigg(\frac{1}{L}\sum_{i=1}^{L}\phi^{*}(x_{i})\bigg),
\end{align*}
where $X=[x_{1},\cdots,x_{L}]$. The $i^{\rm th}$ component of $\rho^{*}$ is denoted as $\rho_{i}^{*}$ for $i\in[d_{y}]$. For a function $f$ defined on $\Omega$, the $L^{\infty}$ norm of it is defined as $\|f\|_{\infty}=\sup_{x\in\Omega}|f(x)|$. The set of the real-valued smooth functions on it is denoted as $\calS^{\infty}(\Omega,\bbR)$, Then we denote the set of the smooth functions with bounded derivatives as
\begin{align*}
    \calS_{B}=\Big\{f\in\calS^{\infty}([-B,B],\bbR)\,|\, \big\|f^{(n)}(x)\big\|\leq n! \text{ for all }n\in\bbN\Big\},
\end{align*}
where $f^{(n)}$ is the $n^{\rm th}$-order derivative of $f$.

\begin{assumption}\label{assump:smooth}
    There exists $B>0$ such that $\phi^{*},\rho_{i}^{*}\in\calS_{B}$ for $i\in[d_{y}]$.
\end{assumption}
This assumption states that the function $f^{*}$ is smooth enough. Then we have that
\begin{proposition}\label{prop:fapproxierr}
    Under \ref{assump:smooth}, if $d_{F}\geq 16d_{y}$, $B_{A,1}\geq 16 R d_{y}$, $B_{A,2}\geq d_{F}$ $B_{A}\geq \sqrt{d_{y}}$, and $B_{V}\geq \sqrt{d}$, then
    \begin{align*}
        \max_{\|S^{\top}\|_{2,\infty}\leq R}\big\|f^{*}(S)-f_{\theta^{*}}(S)\big\|_{2}= \cO\bigg( d_{y}\exp\bigg(-\frac{D^{1/4}}{\sqrt{C^{2}B^{2}\log B_{A,1}}}\bigg)\bigg)
    \end{align*}
    for some constant $C>0$.
\end{proposition}

\subsubsection{Proof of Theorem~\ref{thm:pretrain}}
\begin{proof}[Proof of Theorem~\ref{thm:pretrain}]
    For ease of notation, we respectively define the empirical risk and the population risk as
    \begin{align*}
        \hcalL(f,\calD)=\frac{1}{N}\sum_{i=1}^{N}\big\|x^{i}-f_{\theta}(S^{i})\big\|_{2}^{2},\quad \calL(f)=\bbE_{S,x}\Big[\big\|x-f_{\theta}(S)\big\|_{2}^{2}\Big].
    \end{align*}
    The our proof mainly involves three steps.
    \begin{itemize}
        \item Error decomposition for the excess population risk.
        \item Control each term in the error decomposition.
        \item Conclude the proof.
    \end{itemize}
    
    \textbf{Step 1: Error decomposition for the excess population risk.}
    The excess population risk for the estimate $\htheta$ can be decomposed to the sum of the generalization error and the approximation error as
    \begin{align}
        &\calL(f_{\htheta})-\calL(f^{*})\nonumber\\
        &\quad=\calL(f_{\htheta})-\calL(f^{*})-2\big(\hcalL(f_{\htheta},\calD)-\hcalL(f^{*},\calD)\big)+2\big(\hcalL(f_{\htheta},\calD)-\hcalL(f_{\theta^{*}},\calD)\big)+2\big(\hcalL(f_{\theta^{*}},\calD)-\hcalL(f^{*},\calD)\big)\nonumber\\
        &\quad\leq\underbrace{\calL(f_{\htheta})-\calL(f^{*})-2\big(\hcalL(f_{\htheta},\calD)-\hcalL(f^{*},\calD)\big)}_{\text{generalization error}}+\underbrace{2\big(\hcalL(f_{\theta^{*}},\calD)-\hcalL(f^{*},\calD)\big)}_{\text{approximation error}},\label{eq:errordecomp}
    \end{align}
    where $\theta^{*}=\argmin_{\theta\in\Theta}\calL(f_{\theta})$, and the inequality results from that $\htheta$ achieves the minimal empirical risk.
    
    \textbf{Step 2: Control each term in the error decomposition.}
    
    We first consider the generalization error and will adapt Lemma~\ref{lem:fastpacbayes} to bound it. Define the function
    \begin{align*}
        g(S,x,\theta)=\big\|x-f_{\theta}(S)\big\|_{2}^{2}-\big\|x-f^{*}(S)\big\|_{2}^{2}.
    \end{align*}
    To verify the conditions in Lemma~\ref{lem:fastpacbayes}, we notice that $|g(S,x,\theta)|\leq (B_{x}+B_{f})^{2}$ and that
    \begin{align*}
        \bbE \big[g(S,x,\theta)\big]&=\bbE\Big[\big\|x-f_{\theta}(S)\big\|_{2}^{2}-\big\|x-f^{*}(S)\big\|_{2}^{2}\Big]\\
        &=\bbE\Big[\big\|f^{*}(S)-f_{\theta}(S)\big\|_{2}^{2}\Big]\\
        \bbE\Big[\big(g(S,x,\theta)-\bbE \big[g(S,x,\theta)\big]\big)^{2}\Big]&\leq \bbE \Big[\big(g(S,x,\theta)\big)^{2}\Big]\\
        &\leq \bbE\Big[\big\|2x-f^{*}(S)-f_{\theta}(S)\big\|_{2}^{2}\cdot\big\|f^{*}(S)-f_{\theta}(S)\big\|_{2}^{2}\Big]\\
        &\leq (3B_{x}+B_{f})^{2}\cdot\bbE\Big[\big\|f^{*}(S)-f_{\theta}(S)\big\|_{2}^{2}\Big],
    \end{align*}
    where the second equality results from the definition of $f^{*}$, the second inequality results from Cauchy–Schwarz inequality, and the last inequality result from the boundedness of $x$, $f^{*}$, and $f_{\theta}$. Then Lemma~\ref{lem:fastpacbayes} shows that for a distribution $Q\in\Delta(\Theta)$ and $0<\lambda\leq 1/(2(B_{x}+B_{f})^{2})$, the following holds with probability at least $1-\delta$ simultaneously for all $P\in\Delta(\Theta)$
    \begin{align*}
        &\biggl|\bbE_{\theta\sim P}\biggl[\bbE \big[g(S,x,\theta)\big]-\frac{1}{N}\sum_{i=1}^{N}g(S^{i},x^{i},\theta)\biggr]\biggr|\\
        &\quad\leq \lambda (3B_{x}+B_{f})^{2}\bbE_{\theta\sim P}\Big[\bbE \big[g(S,x,\theta)\big]\Big]+\frac{1}{N\lambda}\biggl[\KL(P\,\|\,Q)+\log\frac{2}{\delta}\biggr].
    \end{align*}
    Taking $\lambda=1/(2(3B_{x}+B_{f})^{2})$, we have
    \begin{align*}
        &\biggl|\bbE_{\theta\sim P}\Bigl[\calL(f_{\theta})-\calL(f^{*})-\big(\hcalL(f_{\theta},\calD)-\hcalL(f^{*},\calD)\big)\Bigr]\biggr|\\
        &\quad\leq \frac{1}{2}\bbE_{\theta\sim P}\big[\calL(f_{\theta})-\calL(f^{*})\big]+\frac{2(3B_{x}+B_{f})^{2}}{N}\biggl[\KL(P\,\|\,Q)+\log\frac{2}{\delta}\biggr].
    \end{align*}
    Next, we will take proper $P$ and $Q$ to relate this equation and the generalization error. For this purpose, we quantify how the perturbation of network parameters influence the output of the network.
    \begin{proposition}\label{prop:parafluc2}
        For any input $X\in\bbR^{L\times d}$ and $\theta,\tiltheta\in\Theta$, we have that
        \begin{align*}
            \|f_{\theta}(X)-f_{\tilde{\theta}}(X)\|_{2}\leq \big\|A^{(D+1)}-\tilA^{(D+1)}\big\|_{\rmF}+\sum_{t=1}^{D}\alpha_{t}(\beta_{t}+\iota_{t}+\kappa_{t}+\rho_{t}),
        \end{align*}
        where 
        \begin{align*}
            \alpha_{t}&=B_{A}(1+B_{A,1}\cdot B_{A,2})\big(1+hB_{V}(1+4 B_{Q}B_{K})\big)^{D-t}\\
            \beta_{t}&=|\gamma_{2}^{(t)}-\tilgamma_{2}^{(t)}|+(1+B_{A,1}\cdot B_{A,2})\cdot\big(1+(\|X^\top\|_{2,\infty}-1)\bbI_{t=1}\big)\cdot|\gamma_{1}^{(t)}-\tilgamma_{1}^{(t)}|\\
            \iota_{t}&=B_{A,2}\cdot\|A_{1}^{(t)}-\tilA_{1}^{(t)}\|_{\rmF}+B_{A,1}\cdot\|A_{2}^{(t)}-\tilA_{2}^{(t)}\|_{\rmF}\\
            \kappa_{t}&=(1+B_{A,1}\cdot B_{A,2})\cdot\big(1+(\|X^\top\|_{2,\infty}-1)\bbI_{t=1}\big)\cdot\sum_{i=1}^{h}\big\|W_{i}^{V,(t)}-\tilde{W}_{i}^{V,(t)}\|_{\rmF}\\
            \rho_{t}&=2(1+B_{A,1}\cdot B_{A,2})\cdot\big(1+(\|X^\top\|_{2,\infty}-1)\bbI_{t=1}\big)\cdot B_{V}\\
            &\quad\qquad \cdot\sum_{i=1}^{h}B_{K}\cdot\|W_{i}^{Q,(t+1)}-\tilde{W}_{i}^{Q,(t+1)}\|_{\rmF}+B_{Q}\cdot\|W_{i}^{K,(t+1)}-\tilde{W}_{i}^{K,(t+1)}\|_{\rmF}
        \end{align*}
        for all $t\in[D]$.
    \end{proposition}
    \begin{proof}[Proof of Proposition~\ref{prop:parafluc2} ]
        See Appendix~\ref{app:parafluc2}.
    \end{proof}
    With the help of Proposition~\ref{prop:parafluc2}, we set the distribution $P$ as 
    \begin{align}
        P&=\prod_{t=1}^{D+1}\calL_{P}\big(\theta^{(t)}\big)\label{eq:distP2}\\
        \calL_{P}\big(\theta^{(D+1)}\big)&=\unif\Big(\bbB\big(\hatA^{(D+1)},r^{(D+1)},\|\cdot\|_{\rmF}\big)\Big)\nonumber\\
        \calL_{P}\big(\theta^{(t)}\big)&=\unif\Big(\bbB\big(\hat\gamma_{1}^{(t)},r_{\gamma,1}^{(t)},|\cdot|\big)\Big)\cdot\unif\Big(\bbB\big(\hat\gamma_{2}^{(t)},r_{\gamma,2}^{(t)},|\cdot|\big)\Big)\cdot\calL_{P}(A^{(t)})\cdot\calL_{P}(W^{(t)})\nonumber\\
        \calL_{P}(A^{(t)})&=\unif\Big(\bbB\big(\hatA_{1}^{(t)},r_{A,1}^{(t)},\|\cdot\|_{\rmF}\big)\Big)\cdot\unif\Big(\bbB\big(\hatA_{2}^{(t)},r_{A,2}^{(t)},\|\cdot\|_{\rmF}\big)\Big)\nonumber\\
        \calL_{P}(W^{(t)})&=\prod_{i=1}^{h}\unif\Big(\bbB\big(\hatW_{i}^{Q,(t)},r_{Q}^{(t)},\|\cdot\|_{\rmF}\big)\Big)\cdot\unif\Big(\bbB\big(\hatW_{i}^{K,(t)},r_{K}^{(t)},\|\cdot\|_{\rmF}\big)\Big)\cdot\unif\Big(\bbB\big(\hatW_{i}^{V,(t)},r_{V}^{(t)},\|\cdot\|_{\rmF}\big)\Big)\nonumber
    \end{align}
    for $t\in[D]$, where $\unif$ denotes the uniform distribution on the set, $\bbB(a,r,\|\cdot\|)=\{x\,|\, \|x-a\|\leq r\}$ denotes the ball centered in $a$ with radius $r$, the radius is set as
    \begin{alignat*}{2}
        r_{\gamma,1}^{(t)}&=(B_{x}+B_{f})^{-1}R^{-1}(1+B_{A,1}\cdot B_{A,2})^{-1}\alpha_{t}^{-1}/N,&\  r_{\gamma,2}^{(t)}&=(B_{x}+B_{f})^{-1}R^{-1}\alpha_{t}^{-1}/N\\
        r_{A,1}^{(t)}&=(B_{x}+B_{f})^{-1}R^{-1}B_{A,2}^{-1}\alpha_{t}^{-1}/N,&\  r_{A,2}^{(t)}&=(B_{x}+B_{f})^{-1}R^{-1}B_{A,1}^{-1}\alpha_{t}^{-1}/N,\\
        r_{V}^{(t)}&=(B_{x}+B_{f})^{-1}R^{-1}h^{-1}(1+B_{A,1}\cdot B_{A,2})^{-1}\alpha_{t}^{-1}/N,&\  r^{(D+1)}&= (B_{x}+B_{f})^{-1}B_{A}^{-1}/N,\\
        r_{K}^{(t)}&=(B_{x}+B_{f})^{-1}R^{-1}h^{-1}(1+B_{A,1}\cdot B_{A,2})^{-1}B_{V}^{-1}B_{Q}^{-1}\alpha_{t}^{-1}/N, \\
        r_{Q}^{(t)}&=(B_{x}+B_{f})^{-1}R^{-1}h^{-1}(1+B_{A,1}\cdot B_{A,2})^{-1}B_{V}^{-1}B_{K}^{-1}\alpha_{t}^{-1}/N.
    \end{alignat*}
    Under this assignment, we now bound $\bbE_{\theta\sim P}[\|x-f_{\theta}(S)\|_{2}^{2}-\|x-f_{\htheta}(S)\|_{2}^{2}]$ as
    \begin{align*}
        &\bigg|\bbE_{\theta\sim P}\Big[\big\|x-f_{\theta}(S)\big\|_{2}^{2}-\big\|x-f_{\htheta}(S)\big\|_{2}^{2}\Big]\bigg|\leq 2(B_{x}+B_{f})\bigg|\bbE_{\theta\sim P}\Big[\big\|f_{\theta}(S)-f_{\htheta}(S)\big\|_{2}\Big]\bigg|= \cO\bigg(\frac{B_{x}+B_{f}}{N}\bigg),
    \end{align*}
    where the inequality results from Cauchy-Schwarz inequality, and the equality results from Proposition~\ref{prop:parafluc2}. Thus, we have that
    \begin{align}
        &\calL(f_{\htheta})-\calL(f^{*})-\big(\hcalL(f_{\htheta},\calD)-\hcalL(f^{*},\calD)\big)\nonumber\\
        &\quad\leq \frac{1}{2}\big(\calL(f_{\htheta})-\calL(f^{*})\big)+\cO\bigg(\frac{B_{x}+B_{f}}{N}\bigg)+\frac{2(3B_{x}+B_{f})^{2}}{N}\biggl[\KL(P\,\|\,Q)+\log\frac{2}{\delta}\biggr].\label{ieq:tI}
    \end{align}
    To access to the value of $\KL(P\,\|\,Q)$, we take $Q$ as the distribution in \eqref{eq:distQ} except that
    \begin{align}
        \calL_{Q}\big(\theta^{(D+1)}\big)&=\unif\Big(\bbB\big(0,B_{A},\|\cdot\|_{\rmF}\big)\Big).\label{eq:distQ2}
    \end{align}
    Then the KL divergence between $P$ and $Q$ is
    \begin{align*}
        \KL(P\,\|\,Q)=\cO\Big((D^{2}\cdot d\cdot (d_{F}+d_{h}+d)+d\cdot d_{y})\cdot\log\big( 1+NB_{x}RhB_{A}B_{A,1}B_{A,2}B_{Q}B_{K}B_{V}\big)\Big).
    \end{align*}
    Combining this equality with \eqref{ieq:tI}, we have that with probability at least $1-\delta$, the generalization error can be bounded as
    \begin{align}
        \calL(f_{\htheta})-\calL(f^{*})-2\big(\hcalL(f_{\htheta},\calD)-\hcalL(f^{*},\calD)\big)=\cO\bigg(\frac{B_{x}^{2}}{N}\biggl[\barD\log(1+N\barB)+\log\frac{2}{\delta}\biggr]\bigg).\label{ieq:tII}
    \end{align}
    
    Next we control the approximation error in \eqref{eq:errordecomp}.
    \begin{align}
        &\hcalL(f_{\theta^{*}},\calD)-\hcalL(f^{*},\calD)\nonumber\\
        &\quad=\hcalL(f_{\theta^{*}},\calD)-\hcalL(f^{*},\calD)-\frac{3}{2}\big(\calL(f_{\theta^{*}})-\calL(f^{*})\big)+\frac{3}{2}\big(\calL(f_{\theta^{*}})-\calL(f^{*})\big)\nonumber\\
        &\quad=\hcalL(f_{\theta^{*}},\calD)-\hcalL(f^{*},\calD)-\frac{3}{2}\big(\calL(f_{\theta^{*}})-\calL(f^{*})\big)+\frac{3}{2}\bbE\Big[\big\|f^{*}(S)-f_{\theta^{*}}(S)\big\|_{2}^{2}\Big],\label{eq:tIII}
    \end{align}
    where the second equality results from the definition of $f^{*}$. To bound the first two terms in the right-hand side of \eqref{eq:tIII}, we use Lemma~\ref{lem:fastpacbayes} and take $P$ and $Q$ as \eqref{eq:distP2} and \eqref{eq:distQ2}, replacing $\htheta$ by $\theta^{*}$. Then we have that
    \begin{align}
        \hcalL(f_{\theta^{*}},\calD)-\hcalL(f^{*},\calD)-\frac{3}{2}\big(\calL(f_{\theta^{*}})-\calL(f^{*})\big)=\cO\bigg(\frac{B_{x}^{2}}{N}\biggl[\barD\log(1+N\barB)+\log\frac{2}{\delta}\biggr]\bigg).\label{eq:tIV}
    \end{align}
    
    \textbf{Step 3: Conclude the proof.}
    
    Combining inequalities~\eqref{eq:errordecomp}, \eqref{ieq:tII}, \eqref{eq:tIII}, and \eqref{eq:tIV}, we have that
    \begin{align*}
        \calL(f_{\htheta})-\calL(f^{*})=\frac{3}{2}\bbE\Big[\big\|f^{*}(S)-f_{\theta^{*}}(S)\big\|_{2}^{2}\Big]+\cO\bigg(\frac{B_{x}^{2}}{N}\biggl[\barD\log(1+N\barB)+\log\frac{2}{\delta}\biggr]\bigg).
    \end{align*}
    Thus, we conclude the proof of Theorem~\ref{thm:pretrain}.
    
\end{proof}
\subsubsection{Proof of Proposition~\ref{prop:fapproxierr}}
\begin{proof}[Proof of Proposition~\ref{prop:fapproxierr}]
    Our proof mainly involves three steps.
    \begin{itemize}
        \item Build the high-level transformer approximator for $f^{*}$.
        \item Build the approximators in the transformer for $\phi^{*}$ and $\rho_{i}^{*}$ separately.
        \item Conclude the proof.
    \end{itemize}
    The first two steps follow the procedures of the proof of Proposition~\ref{prop:fapproxerrmle} exactly. Now we present the final step.
    
    \textbf{Step 3: Conclude the proof.}
    
    In the final layer, we just take $A^{(D+1)}=I_{d_{y}}$ as the identity matrix. Denoting the derived parameters as $\theta^{*}$ we have that
    \begin{align*}
        &\max_{\|X^{\top}\|_{2,\infty}\leq R}\bigg\|\rho^{*}\bigg(\frac{1}{L}\sum_{i=1}^{L}\phi^{*}(x_{i})\bigg)-f_{\theta^{*}}(X)\bigg\|_{2}= \cO\bigg( d_{y}\exp\bigg(-\frac{D^{1/4}}{\sqrt{C^{2}B^{2}\log B_{A,1}}}\bigg)\bigg).
    \end{align*}
    Thus, we conclude the proof of Proposition~\ref{prop:fapproxierr}.
\end{proof}

\section{Proofs and Formal Statements for \S\ref{sec:comb}}\label{app:comb}
\subsection{Proof of Theorem \ref{th:iclpretrain}} \label{sec:pf-th-iclpretrain}
\begin{proof}
By Corollary~\ref{th:bma_reg} and the fact that $\log (1/p_0(z_*)) \le \beta$, we have that
\begin{align}\label{eq:tk1}
     T^{-1} \cdot \EE_{\cD_\icl} \Bigl[ \sum_{t=1}^{T} \log \PP(r_t \given z^*, \pt_{t-1})-\sum_{t= 1}^T \log \PP(r_{t} \given \pt_{t-1})  \Bigr] \le \beta / T.
\end{align}
In addition, we have that
\begin{align}\label{eq:tk2}
    T^{-1} \cdot \EE_{\cD_\icl} \Bigl[\sum_{t= 1}^T \log \PP(r_{t} \given \pt_{t-1}) - \sum_{t= 1}^T \log \PP_{\hat \theta} (r_{t} \given \pt_{t-1}) \Bigr] =  \EE_{\cD_\icl} \Bigl[  \KL\bigl(\PP(\cdot \given \pt) \,\Big \|\, \PP_{\hat \theta}(\cdot \given \pt) \bigr) \Bigr].
\end{align}
Similar to \eqref{eq:logbound}, we have that
\begin{align*}
    \Bigl|\log \bigl(\PP(r \given \pt)  / \PP_{\hat \theta} (r \given \pt) \bigr)\Bigr| \le b^* = \log \max\{c_0^{-1}, b_y^{-1}\}.
\end{align*}
By Lemma \ref{lem:tv-kl}, we have that
\begin{align}\label{eq:tv-kl1}
    \KL\bigl(\PP(\cdot \given \pt) \,\|\, \PP_{\hat \theta}(\cdot \given \pt) \bigr)  \le (3 + b^*) / 2 \cdot \TV\big(\PP(\cdot\,|\,\pt),\PP_{\htheta}.(\cdot\,|\,\pt)\big).
\end{align}
By Assumption \ref{asp:coverage}, we have that $\PP_{\cD_\icl}(\pt) \le \kappa \PP_\cD(\pt)$. Thus, by Theorem \ref{thm:pretrainmle}, we have with probability at least $1- \delta$ that
\begin{align}\label{eq:tv-kl2}
    &\EE_{\cD_\icl}\Bigl[\KL\bigl(\PP(\cdot \given \pt) \,\|\, \PP_{\hat \theta}(\cdot \given \pt) \bigr) \Bigr] \noend
    & \quad \le C \cdot b^* \cdot \kappa \cdot \EE_{S\sim \cD}\Bigl[\TV\big(\PP(\cdot\,|\,S),\PP_{\htheta}.(\cdot\,|\,S)\big)\Bigr] \le C\cdot b^* \cdot \kappa \cdot  \Delta_{\rm pre}(N, T, \delta).
\end{align}
Combining \eqref{eq:tv-kl2}, \eqref{eq:tk1}, and \eqref{eq:tk2}, we have with probability at least $1- \delta$ that
\begin{align}
    \label{eq:tk3}
     &\EE_{\cD_\icl} \Bigl[ T^{-1} \cdot\sum_{t=1}^{T} \log \PP(r_t \given z^*, \pt_{t-1})- T^{-1} \cdot\sum_{t= 1}^T \log \PP_{\hat \theta} (r_{t} \given \pt_{t-1})  \Bigr] \noend
     & \quad \le \beta / T + \EE_{S \sim \cD} \Bigl[  \KL\bigl(\PP(\cdot \given S) \,\|\, \PP_{\hat \theta}(\cdot \given S) \bigr) \Bigr] \noend 
     & \quad \le \cO\bigl(\beta / T + b^*\cdot \kappa \cdot \Delta_{\rm pre}(N, T, \delta) \bigr),
\end{align}
which completes the proof of Theorem \ref{th:iclpretrain}.
\end{proof}

\subsection{Assumptions and Formal Statement for Prompting With Wrong Input-Output Mappings}\label{app:wonglabel}

We first state assumptions for this setting.
\begin{assumption}\label{assump:indep}
    Conditioned on any $z\in\frakZ$, the input-output pairs are independent, i.e., for any two input-output pair sequences $S_{t},S_{t^{\prime}}^{\prime}\in\frakX^{*}$, we have
    $
        \PP((S_{t},S_{t^{\prime}}^{\prime})\,|\,z)=\PP(S_{t}\,|\, z)\cdot \PP(S_{t^{\prime}}^{\prime}\,|\, z).
    $
\end{assumption}
\vspace{-3.5mm}
This assumption states that for any task $z\in\frakZ$, the input-output pairs are independently generated. This largely holds in realistic applications since the examples usually are independently produced. It can be relaxed when there are more structures in the token generation process, e.g. the hidden Markov model in \cite{xie2021explanation}.
\begin{assumption}\label{assump:plb}
    There exists a constant $c_{1}>0$ such that $\PP_{\calZ}(z_{*})\geq c_{1}$.
\end{assumption}
\vspace{-3.5mm}
This assumption states that the prior distribution of the hidden concept $z_{*}$ is strictly larger than $0$, otherwise this concept can never be deduced.
For two concepts $z,z^{\prime}\in\frakZ$, we define the KL divergence between the conditional distributions of input-output pair on them as $\KL_{\rm pair}(\PP(\cdot\,|\,z)\|\PP(\cdot\,|\,z^{\prime}))=\EE_{X,y\sim \PP(\cdot\,|\,z)}[\log(\PP(X,y\,|\,z)/\PP(X,y\,|\,z^{\prime}))]$. This divergence measures the distance between distributions of input-output pairs conditioned on different tasks $z$ and $z^{\prime}$.
\begin{assumption}\label{assump:distinguish}
    The concept $z_{*}$ satisfies that $\min_{z\neq z_{*}}\KL_{\rm pair}(\PP(\cdot\,|\,z_{*})\,\|\,\PP(\cdot\,|\,z))>2\log 1/c_{0}$, where $c_{0}$ is the constant in Assumption~\ref{assump:positive}.
\end{assumption}
\vspace{-3.5mm}
This distinguishability assumption requires that the divergence between $z_{*}$ and other concepts $z$ is large enough to infer the concept $z_{*}$ from the prompt. We denote the pretraining error in Theorem~\ref{thm:pretrainmle} as $\Delta_{\rm pre}(N_{\rmp}, T_{\rmp}, \delta)$, then we have the following result.
\begin{proposition}\label{prop:wronglabel}
    Under Assumptions~\ref{assump:positive}, \ref{assump:indep}, \ref{assump:plb}, and \ref{assump:distinguish}, the pretrained model $\PP_{\htheta}$ in \eqref{algo:pretrainmle} predicts the outputs with the prompt containing wrong mappings as
    \begin{align*}
        &\bbE_{\pt^{\prime}}\Big[\KL\big(\PP(\cdot\,|\,\tilc_{t+1},z_{*})\|\PP_{\htheta}(\cdot\,|\,S_{t}^{\prime},\tilc_{t+1})\big)\Big]\\
        &\quad=\!\cO\bigg(\!\Delta_{\rm pre}(N_{\rmp},T_{\rmp},\delta)\!+\!\exp\bigg(\!-\!\frac{\sqrt{t}}{2(1+l)\log 1/c_0}\!\bigg(\min_{z\neq z^{*}}\KL_{\rm pair}\big(\PP(\cdot\,|\,z^{*})\,\|\,\PP(\cdot\,|\,z)\big)+2\log c_{0}\bigg)\!\bigg)\!\bigg)
    \end{align*}
    with probability at least $1-\delta$.
\end{proposition}
\vspace{-2.5mm}

\subsection{Proof of Proposition~\ref{prop:wronglabel}}
\begin{proof}[Proof of Proposition~\ref{prop:wronglabel}]
    From Bayesian model averaging, the output distribution is
    \begin{align}
        &\PP(r_{t+1}\,|\,S_{t}^{\prime},\tilc_{t+1})\nonumber\\
        &\quad=\sum_{z\in\frakZ}\PP(r_{t+1}\,|\,\tilc_{t+1},z)\cdot \PP_{\calZ}(z\,|\,S_{t}^{\prime})\nonumber\\
        &\quad=\PP(r_{t+1}\,|\,\tilc_{t+1},z^{*})+\sum_{z\neq z^{*}}\big(\PP(r_{t+1}\,|\,\tilc_{t+1},z)-\PP(r_{t+1}\,|\,\tilc_{t+1},z^{*})\big)\cdot \PP_{\calZ}(z\,|\,S_{t}^{\prime})\nonumber\\
        &\quad=\PP(r_{t+1}\,|\,\tilc_{t+1},z^{*})+\sum_{z\neq z^{*}}\big(\PP(r_{t+1}\,|\,\tilc_{t+1},z)-\PP(r_{t+1}\,|\,\tilc_{t+1},z^{*})\big)\cdot \PP_{\calZ}(z^{*}\,|\,S_{t}^{\prime})\cdot \frac{\PP_{\calZ}(z)\cdot \PP(S_{t}^{\prime}\,|\,z)}{\PP_{\calZ}(z^{*})\cdot \PP(S_{t}^{\prime}\,|\,z^{*})},\label{eq:bmaerr}
    \end{align}
    where the first equality results from Bayesian model averaging, the last equality results from Bayes' theorem. Next, we upperbound the ratio $\PP(S_{t}^{\prime}\,|\,z)/\PP(S_{t}^{\prime}\,|\,z^{*})$ in the right-hand side of Eqn.~\eqref{eq:bmaerr}. We have that
    \begin{align*}
        \frac{1}{t}\log\frac{\PP(S_{t}^{\prime}\,|\,z)}{\PP(S_{t}^{\prime}\,|\,z^{*})}=\frac{1}{t}\sum_{i=1}^{t}\log \frac{\PP\big((\tilc_{i},r_{i}^{\prime})\,|\,z\big)}{\PP\big((\tilc_{i},r_{i}^{\prime})\,|\,z^{*}\big)}\leq -2\log c_{0}+\frac{1}{t}\sum_{i=1}^{t}\log \frac{\PP\big((\tilc_{i},r_{i})\,|\,z\big)}{\PP\big((\tilc_{i},r_{i})\,|\,z^{*}\big)},
    \end{align*}
    where the equality results from Assumption~\ref{assump:indep}, and the inequality results from Assumption~\ref{assump:positive}. Assumption~\ref{assump:positive} also implies that $|\log \PP((\tilc_{i},r_{i})\,|\,z)/\PP((\tilc_{i},r_{i})\,|\,z^{*}) |\leq (1+l)\log 1/c_{0}$. Hoeffding inequality shows that with probability at least $1-\delta$, we have
    \begin{align*}
        \frac{1}{t}\sum_{i=1}^{t}\log \frac{\PP\big((\tilc_{i},r_{i})\,|\,z\big)}{\PP\big((\tilc_{i},r_{i})\,|\,z^{*}\big)}+\KL_{\rm pair}\big(\PP(\cdot\,|\,z^{*})\,\|\,\PP(\cdot\,|\,z)\big)\leq \frac{(1+l)}{\sqrt{t}}\log\frac{1}{c_{0}}\cdot\log\frac{1}{\delta}.
    \end{align*}
    Thus, we have that with probability at least $1-\delta$, the following holds for all $z\neq z^{*}$
    \begin{align*}
        \frac{\PP(S_{t}^{\prime}\,|\,z)}{\PP(S_{t}^{\prime}\,|\,z^{*})}\leq \exp\bigg(-t\bigg(\KL_{\rm pair}\big(\PP(\cdot\,|\,z^{*})\,\|\,\PP(\cdot\,|\,z)\big)+2\log c_{0}-\frac{(1+l)}{\sqrt{t}}\log\frac{1}{c_{0}}\cdot\log\frac{|\frakZ|}{\delta}\bigg)\bigg).
    \end{align*}
    Combining this inequality with Eqn.~\eqref{eq:bmaerr}, we have that
    \begin{align}
        &\TV\big(\PP(\cdot\,|\,S_{t}^{\prime},\tilc_{t+1}),\PP(\cdot\,|\,\tilc_{t+1},z^{*})\big)\nonumber\\
        &\quad=\cO\bigg(\frac{1}{c_{1}}\exp\bigg(-t\bigg(\min_{z\neq z^{*}}\KL_{\rm pair}\big(\PP(\cdot\,|\,z^{*})\,\|\,\PP(\cdot\,|\,z)\big)+2\log c_{0}-\frac{(1+l)}{\sqrt{t}}\log\frac{1}{c_{0}}\cdot\log\frac{|\frakZ|}{\delta}\bigg)\bigg)\bigg).\label{eq:tvb}
    \end{align}
    Taking expectations with respect to the distribution of $S_{t}^{\prime},\tilc_{t+1}$ on the both sides in \eqref{eq:tvb}, we have that
    \begin{align}
        &\bbE_{\pt^{\prime}}\Big[\TV\big(\PP(\cdot\,|\,S_{t}^{\prime},\tilc_{t+1}),\PP(\cdot\,|\,\tilc_{t+1},z^{*})\big)\Big]\nonumber\\
        &\quad=\cO\bigg(\frac{1}{c_{1}}\exp\bigg(-t\bigg(\min_{z\neq z^{*}}\KL_{\rm pair}\big(\PP(\cdot\,|\,z^{*})\,\|\,\PP(\cdot\,|\,z)\big)+2\log c_{0}-\frac{(1+l)}{\sqrt{t}}\log\frac{1}{c_{0}}\cdot\log\frac{|\frakZ|}{\delta}\bigg)\bigg)\bigg)+\delta.\label{eq:tvb1}
    \end{align}
    We set $\delta=|\frakZ\exp(-a\sqrt{t}/2b)|$, where $a = \min_{z\neq z^{*}}\KL_{\rm pair}\big(\PP(\cdot\,|\,z^{*})\,\|\,\PP(\cdot\,|\,z)\big)+2\log c_{0}$, $b=-(1+l)\log{c_{0}}$. Then the right-hand side of \eqref{eq:tvb1} can be upper bounded as
    \begin{align*}
        &\bbE_{\pt^{\prime}}\Big[\TV\big(\PP(\cdot\,|\,S_{t}^{\prime},\tilc_{t+1}),\PP(\cdot\,|\,\tilc_{t+1},z^{*})\big)\Big]\nonumber\\
        &\quad=\cO\bigg(\exp\bigg(-\frac{\sqrt{t}}{2(1+l)\log 1/c_0}\bigg(\min_{z\neq z^{*}}\KL_{\rm pair}\big(\PP(\cdot\,|\,z^{*})\,\|\,\PP(\cdot\,|\,z)\big)+2\log c_{0}\bigg)\bigg)\bigg).
    \end{align*}
    Let $\bbE_{\pt^{\prime}}[\TV(\PP(\cdot\,|\,S_{t}^{\prime},\tilc_{t+1}),\PP_{\htheta}(\cdot\,|\,S_{t}^{\prime},\tilc_{t+1}))]\leq\Delta_{\rm pre}(N_{\rmp},T_{\rmp},\delta)$, where $\Delta_{\rm pre}(N_{\rmp},T_{\rmp},\delta)$ is the bound in Theorem~\ref{thm:pretrainmle}. Then we have that
    \begin{align*}
        &\bbE_{\pt^{\prime}}\Big[\KL\big(\PP(\cdot\,|\,\tilc_{t+1},z^{*})\|\PP_{\htheta}(\cdot\,|\,S_{t}^{\prime},\tilc_{t+1})\big)\Big]\\
        &\quad\leq \cO\Big(\bbE_{\pt^{\prime}}\Big[\TV\big(\PP_{\htheta}(\cdot\,|\,S_{t}^{\prime},\tilc_{t+1}),\PP(\cdot\,|\,\tilc_{t+1},z^{*})\big)\Big]\Big)\\
        &\quad=\cO\bigg(\!\Delta_{\rm pre}(N_{\rmp},T_{\rmp},\delta)\!+\!\exp\bigg(-\frac{\sqrt{t}}{2(1+l)\log 1/c_0}\bigg(\min_{z\neq z^{*}}\KL_{\rm pair}\big(\PP(\cdot\,|\,z^{*})\,\|\,\PP(\cdot\,|\,z)\big)+2\log c_{0}\bigg)\bigg)\!\bigg),
    \end{align*}
    where the first equality results from Assumption~\ref{assump:positive}. Thus, we conclude the proof of Proposition~\ref{prop:wronglabel}.
\end{proof}

\section{Proof of Supporting Propositions}\label{appendix:supp}

\subsection{Proof of Proposition \ref{prop:calculation_of_integral}}\label{app:calculation_of_integral}
\begin{proof}
	Let $a, b$ be two vectors in the $(d - 1)$-dimensional unit sphere $\SSS^{d - 1}$. We first define the following vector,
	\#\label{eq::pf_lem_calc_int_eq1}
	c = (a^\top b)\cdot b-\bigl(a - (a^\top b)\cdot b\bigr) \in\SSS^{d - 1}.
	\#
	By direct calculation, we have the following property of $c$ defined in \eqref{eq::pf_lem_calc_int_eq1},
	\#\label{eq::pf_lem_calc_int_eq2}
	c^\top b &= (a^\top b)\cdot\|b\|^2_2 - a^\top b +  (a^\top b)\cdot\|b\|^2_2 = a^\top b.
	\#
	By \eqref{eq::pf_lem_calc_int_eq1} and \eqref{eq::pf_lem_calc_int_eq2}, we have that
	\#\label{eq::pf_lem_calc_int_eq3}
	a + c = 2(a^\top b) \cdot b = 2(c^\top b) \cdot b = (a^\top b) \cdot b + (c^\top b) \cdot b.
	\#
	We now calculate the desired integration. Note that
	\#\label{eq::pf_lem_calc_int_eq4}
	\int_{\SSS^{d - 1}} a\cdot\exp(a^\top b)\ud a = b \cdot \int_{\SSS^{d - 1}} (a^\top b) \exp(a^\top b) \ud a + \int_{\SSS^{d - 1}} \bigl(a - (a^\top b)\cdot b\bigr)\cdot\exp(a^\top b) \ud a.
	\#
	For the second term on the right-hand side of \eqref{eq::pf_lem_calc_int_eq4}, it follows from \eqref{eq::pf_lem_calc_int_eq1} and  \eqref{eq::pf_lem_calc_int_eq2} and \eqref{eq::pf_lem_calc_int_eq3} that
	\#\label{eq::pf_lem_calc_int_eq5}
	\int_{\SSS^{d - 1}} \bigl(a - (a^\top b)\cdot b\bigr)\cdot\exp(a^\top b) \ud a &=-\int_{\SSS^{d - 1}} \bigl(c - (c^\top b)\cdot b\bigr)\cdot\exp(c^\top b) \ud c,
	\#
	where the equality follows from the fact that 
	$
	\ud c = 2\|b\|^2_2\ud a - \ud a = \ud a.
	$
	By replacing $c$ by $a$ on the right-hand side of \eqref{eq::pf_lem_calc_int_eq5}, we have
	\#\label{eq::pf_lem_calc_int_eq6}
	\int_{\SSS^{d - 1}} \bigl(a - (a^\top b)\cdot b\bigr)\cdot\exp(a^\top b) \ud a = -\int_{\SSS^{d - 1}} \bigl(a - (a^\top b)\cdot b\bigr)\cdot\exp(a^\top b) \ud a= 0
	\#
	Finally, by plugging \eqref{eq::pf_lem_calc_int_eq6} into \eqref{eq::pf_lem_calc_int_eq4}, we obtain that
	\$
	\int_{\SSS^{d - 1}} a\cdot\exp(a^\top b)\ud a = b \cdot \int_{\SSS^{d - 1}} (a^\top b) \exp(a^\top b) \ud a.
	\$
	Thus, by setting
	\$
	C_1 = \int_{\SSS^{d - 1}} (a^\top b) \exp(a^\top b) \ud a, \quad \forall b \in\SSS^{d - 1},
	\$
	we complete the proof of Proposition \ref{prop:calculation_of_integral}. Note that here $C_1$ is an absolute constant that does not depend on $b$ due to the symmetry on the unit sphere.
\end{proof}
\subsection{Proof of Proposition~\ref{prop:pacbayes}}\label{app:pacbayes}
\begin{proof}[Proof of Proposition~\ref{prop:pacbayes}]
    We note that $f(X)$ satisfies the condition in Lemma~\ref{lem:MCconcen} with $c_{i}=2b/N$ for $i\in[N]$. Then Lemma~\ref{lem:MCconcen} shows that
    \begin{align*}
        \bbE_{f\sim P_{0}}\Big[\bbE_{X}\Big(\exp\big[\lambda(f(X)-\bbE f(X))\big]\Big)\Big]\leq \exp\bigg(\frac{\lambda^{2}\cdot b^{2}\cdot t_{\rm min}}{2N}\bigg).
    \end{align*}
    Take $\lambda=\sqrt{2N\log 2/(b^{2} t_{\rm min}) }$. The Markov inequality shows that
    \begin{align*}
        P\bigg(\bbE_{f\sim P_{0}}\Big(\exp\big[\lambda(f(X)-\bbE f(X))\big]\Big)\geq \frac{2}{\delta}\bigg)\leq\delta
    \end{align*}
    for any $0<\delta<1$. We note that this probability inequality does not involve $P$. Take the function $g$ in Lemma~\ref{lem:DV} as $g(f)=\lambda(f(X)-\bbE f(X))$, then it shows that
    \begin{align*}
        \log\bbE_{P_{0}}\Big[\exp\big(g(X)\big)\Big]+\KL(P\,\|\,P_{0})\geq  \bbE_{P}\big[g(X)\big]
    \end{align*}
    for any $P$ simultaneously. Combining these inequalities, we have
    \begin{align*}
        \Bigl|\bbE_{P}\Bigl[\bbE_{X}\big[f(X)\big]-f(X)\Bigr]\Bigr|\leq \sqrt{\frac{b^{2}\cdot t_{\rm min}}{2\log 2 N}}\biggl[\KL(P\,\|\,P_{0})+\log\frac{4}{\delta}\biggr],\nonumber
    \end{align*}
    for any distribution $P$ on $\calF$ simultaneously with probability at least $1-\delta$. Thus, we conclude the proof of Proposition~\ref{prop:pacbayes}.
\end{proof}
\subsection{Proof of Proposition~\ref{prop:parafluc}}\label{app:parafluc}
\begin{proof}[Proof of Proposition~\ref{prop:parafluc} ]
        We analyze the error layer by layer in the neural network. Denote the outputs of each layer in the networks parameterized by $\theta$ and $\tiltheta$ as $X^{(t)}$ and $\tilX^{(t)}$, respectively. In the final layer, we have that
        \begin{align*}
            &\TV\big(P_{\theta}(\cdot\,|\,X),P_{\tiltheta}(\cdot\,|\,X)\big)\nonumber\\
            &\quad\leq 2\bigg\|\frac{1}{L\tau}\bbI_{L}^\top X^{(D)}A^{(D+1)}-\frac{1}{L\tau}\bbI_{L}^\top \tilX^{(D)}\tilA^{(D+1)}\bigg\|_{\infty}\nonumber\\
            &\quad\leq \frac{2}{\tau}\Big[\big\|A^{(D+1),\top}\big\|_{1,2}\cdot\big\|X^{(D),\top}-\tilX^{(D),\top}\big\|_{2,\infty}+\big\|A^{(D+1),\top}-\tilA^{(D+1),\top}\big\|_{1,2}\Big],
        \end{align*}
        where the first inequality results from Lemma~\ref{lem:smlip}, and the second inequality results from Lemma~\ref{lem:matvec} and that $\|X^{(D),\top}\|_{2,\infty}\leq 1$ due to the layer normalization. In the following, we build the recursion relationship between $\|X^{(t),\top}-\tilX^{(t),\top}\|_{2,\infty}$ for $t\in[D]$.
        \begin{align}
            &\|X^{(t+1),\top}-\tilX^{(t+1),\top}\|_{2,\infty}\nonumber\\
            &\quad\leq \big\|\FF(Y^{(t+1)},A^{(t+1)})^\top-\FF(\tilY^{(t+1)},\tilA^{(t+1)})^\top\big\|_{2,\infty}+|\gamma_{2}^{(t+1)}-\tilgamma_{2}^{(t+1)}|+ \big\|Y^{(t+1),\top}-\tilY^{(t+1),\top}\big\|_{2,\infty}\nonumber\\
            &\quad\leq |\gamma_{2}^{(t+1)}-\tilgamma_{2}^{(t+1)}|+ \big\|Y^{(t+1),\top}-\tilY^{(t+1),\top}\big\|_{2,\infty}+B_{A,1}\cdot B_{A,2}\cdot\|Y^{(t+1),\top}-\tilY^{(t+1),\top}\|_{2,\infty}\nonumber\\
            &\quad\qquad+B_{A,2}\cdot\|A_{1}^{(t+1)}-\tilA_{1}^{(t+1)}\|_{\rmF}+B_{A,1}\cdot\|A_{2}^{(t+1)}-\tilA_{2}^{(t+1)}\|_{\rmF},\label{ieq:6}
        \end{align}
        where the first inequality results from the triangle inequality and that $\Pi_{\rm norm}$ is not expansive, the second inequality results from the following proposition
        \begin{proposition}\label{prop:fflip}
            For any $X,\tilX\in\bbR^{L\times d}$, $A_{1},\tilA_{1}\in\bbR^{d\times d_{F}}$, and $A_{2},\tilA_{2}\in\bbR^{d_{F}\times d}$, we have that
            \begin{align*}
                &\big\|\FF(X,A)^\top-\FF(\tilX,\tilA)^\top\big\|_{2,\infty}\\
                &\quad\leq\|A_{1}\|_{\rmF}\cdot\|A_{2}\|_{\rmF}\cdot\|X^\top-\tilX^\top\|_{2,\infty}+\|A_{1}-\tilA_{1}\|_{\rmF}\cdot\|A_{2}\|_{\rmF}\cdot\|\tilX^\top\|_{2,\infty}\nonumber\\
                &\quad\qquad+\|\tilA_{1}\|_{\rmF}\cdot\|A_{2}-\tilA_{2}\|_{\rmF}\cdot\|\tilX^\top\|_{2,\infty}.
            \end{align*}
        \end{proposition}
        \begin{proof}[Proof of Proposition~\ref{prop:fflip}]
            See Appendix~\ref{app:fflip}.
        \end{proof}
        Next, we build the relationship between $\|Y^{(t+1),\top}-\tilY^{(t+1),\top}\|_{2,\infty}$ in the right-hand side of inequality~\eqref{ieq:6} and $\|X^{(t),\top}-\tilX^{(t),\top}\|_{2,\infty}$.
        \begin{align}
            &\|Y^{(t+1),\top}-\tilY^{(t+1),\top}\|_{2,\infty}\nonumber\\
            &\quad \leq \big\|\MHA(X^{(t)},W^{(t+1)})^\top-\MHA(\tilX^{(t)},\tilW^{(t+1)})^\top\big\|_{2,\infty}+|\gamma_{1}^{(t+1)}-\tilgamma_{1}^{(t+1)}|+ \big\|X^{(t),\top}-\tilX^{(t),\top}\big\|_{2,\infty}\nonumber\\
            &\quad\leq |\gamma_{1}^{(t+1)}-\tilgamma_{1}^{(t+1)}|+ \big\|X^{(t),\top}-\tilX^{(t),\top}\big\|_{2,\infty}\nonumber\\
            &\quad\qquad+h\cdot B_{V}\big(1+4 B_{Q}B_{K}\big)\|X^{(t),\top}-\tilde{X}^{(t),\top}\|_{2,\infty}+\sum_{i=1}^{h}\big\|W_{i}^{V,(t+1)}-\tilde{W}_{i}^{V,(t+1)}\|_{\rmF}\nonumber\\
            &\quad\qquad+2 B_{V}\cdot B_{K}\sum_{i=1}^{h}\|W_{i}^{Q,(t+1)}-\tilde{W}_{i}^{Q,(t+1)}\|_{\rmF}+2 B_{V}\cdot B_{Q}\sum_{i=1}^{h}\|W_{i}^{K,(t+1)}-\tilde{W}_{i}^{K,(t+1)}\|_{\rmF},\label{ieq:7}
        \end{align}
        where the first inequality results from the triangle inequality, and the second inequality results from Lemma~\ref{lem:mhalip}. Combining inequalities~\eqref{ieq:6} and \eqref{ieq:7}, we derive that
        \begin{align*}
            &\|X^{(t+1),\top}-\tilX^{(t+1),\top}\|_{2,\infty}\nonumber\\
            &\quad\leq (1+B_{A,1}\cdot B_{A,2})\big(1+hB_{V}(1+4 B_{Q}B_{K})\big)\|X^{(t),\top}-\tilde{X}^{(t),\top}\|_{2,\infty}+\beta_{t+1}+\iota_{t+1}+\kappa_{t+1}+\rho_{t+1}.
        \end{align*}
        This concludes the proof of Proposition~\ref{prop:parafluc}.
    \end{proof}
    
\subsection{Proof of Proposition~\ref{prop:parafluc2}}\label{app:parafluc2}
\begin{proof}[Proof of Proposition~\ref{prop:parafluc2} ]
        We analyze the error layer by layer in the neural network. Denote the outputs of each layer in the networks parameterized by $\theta$ and $\tiltheta$ as $X^{(t)}$ and $\tilX^{(t)}$, respectively. In the final layer, we have that
        \begin{align*}
            &\|f_{\theta}(X)-f_{\tilde{\theta}}(X)\|_{2}\nonumber\\
            &\quad\leq \big\|\tilA^{(D+1)}\big\|_{\rmF}\cdot\big\|X^{(D),\top}-\tilX^{(D),\top}\big\|_{2,\infty}+\big\|A^{(D+1)}-\tilA^{(D+1)}\big\|_{\rmF},
        \end{align*}
        where the inequality results from Lemma~\ref{lem:matvec} and that $\|X^{(D),\top}\|_{2,\infty}\leq 1$ due to the layer normalization.  The remaining proof just follows the procedures in the proof of Proposition~\ref{prop:parafluc}, and we have that
        \begin{align*}
            &\|f_{\theta}(X)-f_{\tilde{\theta}}(X)\|_{2}\\
            &\quad\leq \big\|A^{(D+1)}-\tilA^{(D+1)}\big\|_{\rmF}+\sum_{t=1}^{D}\alpha_{t}(\beta_{t}+\iota_{t}+\kappa_{t}+\rho_{t}).
        \end{align*}
        Thus, we conclude the proof of Proposition~\ref{prop:parafluc2}.
    \end{proof}
\subsection{Proof of Proposition~\ref{prop:fflip}}\label{app:fflip}
\begin{proof}[Proof of Proposition~\ref{prop:fflip}]
    We have that
    \begin{align*}
        &\big\|\FF(X,A)^\top-\FF(\tilX,\tilA)^\top\big\|_{2,\infty}\\
        &\quad\leq \max_{i\in[L]}\Big[\big\|\relu(X_{i,:}A_{1})A_{2}-\relu(\tilX_{i,:}A_{1})A_{2}\big\|_{2}+\big\|\relu(\tilX_{i,:}A_{1})A_{2}-\relu(\tilX_{i,:}\tilA_{1})\tilA_{2}\big\|_{2}\Big]\\
        &\quad\leq \max_{i\in[L]}\Big[\|A_{1}\|_{\rmF}\cdot\|A_{2}\|_{\rmF}\cdot\|X_{i,:}-\tilX_{i,:}\|_{2}+\big\|\relu(\tilX_{i,:}A_{1})A_{2}-\relu(\tilX_{i,:}\tilA_{1})A_{2}\big\|_{2}\\
        &\quad\qquad +\big\|\relu(\tilX_{i,:}\tilA_{1})A_{2}-\relu(\tilX_{i,:}\tilA_{1})\tilA_{2}\big\|_{2}\Big]\\
        &\quad\leq \max_{i\in[L]}\Big[\|A_{1}\|_{\rmF}\cdot\|A_{2}\|_{\rmF}\cdot\|X_{i,:}-\tilX_{i,:}\|_{2}+\|A_{1}-\tilA_{1}\|_{\rmF}\cdot\|A_{2}\|_{\rmF}\cdot\|\tilX_{i,:}\|_{2}\nonumber\\
        &\quad\qquad+\|\tilA_{1}\|_{\rmF}\cdot\|A_{2}-\tilA_{2}\|_{\rmF}\cdot\|\tilX_{i,:}\|_{2}\Big],
    \end{align*}
    where the first inequality results from the triangle inequality, the second and the last inequalities result from Lemma~\ref{lem:matvec} and that $\relu$ is not expansive. Thus, we conclude the proof of Proposition~\ref{prop:fflip}.
\end{proof}

\section{Technical Lemmas}
\begin{lemma}[\citet{caponnetto2007optimal}]
	\label{lem:cme-concen}
	Let $(\Omega, \nu)$ be a probability space and $\xi$ be a random variable on $\Omega$ taking value in a real separable Hilbert space $\cH$. We assume that there exists constants $B, \sigma > 0$ such that
	\begin{align*}
		\bigl\|\xi(w)\bigr\|_\cH \le B/2,\ \mathrm{a.s.}, \quad \EE\bigl[\norm{\xi}_\cH^2\bigr] \le \sigma^2.
	\end{align*}
	Then, it holds with probability at least $ 1- \delta$ that
	\begin{align*}
		\biggl\| L^{-1} \sum_{i = 1}^L \xi(\omega_i) - \EE[\xi] \biggr\| \le 2\biggl( \frac{B}{L} + \frac{\sigma}{\sqrt{L}} \biggr) \log \frac{2}{\delta}.
	\end{align*}
\end{lemma}
\begin{lemma}[Proposition 4.5 in \cite{duchi2019information}]\label{lem:fastpacbayes}
    Let $\calF$ be the collection of functions of $f:\bbR^{n}\rightarrow\bbR$. For any $f\in\calF$, we define
    \begin{align}
        \mu(f)= \bbE_{X}\big[f(X)\big], \quad \sigma^{2}(f)= \bbE_{X}\big[ (f(X)-\bbE_{X} [f(X)] )^{2}\big],\nonumber
    \end{align}
    where the expectation is taken with respect to a random variable $X\sim\nu$ on $(\bbR^{n},\calB(\bbR^{n}))$. 
    Assume that $|f(X)-\mu(f)|\leq b$ a.s. for some constant $b\in\bbR$ for all $f\in\calF$. Then for any $0<\lambda\leq 1/(2b)$, given a distribution $P_{0}$ on $\calF$, with probability at least $1-\delta$, we have
    \begin{align}
        \biggl|\bbE_{Q}\biggl[\bbE_{X}[f(X)]-\frac{1}{n}\sum_{i=1}^{n}f(X_{i})\biggr]\biggr|\leq \lambda \bbE_{Q}\big[\sigma^{2}(f)\big]+\frac{1}{n\lambda}\biggl[\KL(Q\,\|\,P_{0})+\log\frac{2}{\delta}\biggr],\nonumber
    \end{align}
    for any distribution $Q$ on $\calF$, where $X_{i}$ are i.i.d.\ samples of $\nu$. If the function class $\calF$ further satisfies $\sigma^{2}(f)\leq c \mu(f)$ for some constant $c\in\bbR$ for all $f\in\calF$, we have 
    \begin{align*}
        \biggl|\bbE_{Q}\biggl[\bbE_{X}\bigl[f(X)\bigr]-\frac{1}{n}\sum_{i=1}^{n}f(X_{i})\biggr]\biggr|\leq \lambda c\bbE_{Q}\big[\mu(f)\big]+\frac{1}{n\lambda}\biggl[\KL(Q\,\|\,P_{0})+\log\frac{2}{\delta}\biggr], 
    \end{align*}
    with probability at least $1-\delta$.
\end{lemma}
\begin{lemma}[Donsker--Varadhan representation in \cite{belghazi2018mutual}]\label{lem:DV}
    Let $P$ and $Q$ be distributions on a common space $\calX$. Then
    \begin{align}
        \KL(P\,\|\,Q)=\sup_{g\in\calG} \bigg\{\bbE_{P}\big[g(X)\big]-\log\bbE_{Q}\Big[\exp\big(g(X)\big)\Big]\bigg\},\nonumber
    \end{align}
    where $\calG=\{g:\calX\rightarrow\bbR \ | \ \bbE_{Q}[\exp(g(X))]<\infty\}$.
\end{lemma}

\begin{lemma}[Corollary 2.11 in \cite{paulin2015concentration}]\label{lem:MCconcen}
    Let $X=(X_{1},\cdots,X_{N})$ be a Markov chain, taking values in $\Lambda=\prod_{i=1}^{N}\Lambda_{i}$ with mixing time $t_{\rm mix}(\varepsilon)$ for $\varepsilon\in [0,1]$. Let
    \begin{align*}
        t_{\rm min}=\inf_{0\leq \varepsilon<1}t_{\rm mix}(\varepsilon)\cdot\bigg(\frac{2-\varepsilon}{1-\varepsilon}\bigg)^{2}.
    \end{align*}
    If function $f:\Lambda\rightarrow \bbR$ is such that
    $f(x)-f(y)\leq \sum_{i=1}^{N}c_{i}\bbI_{x_{i}\neq y_{i}}$
    for every $x,y\in\Lambda$, then for any $\lambda\in\bbR$,
    \begin{align*}
        \log\bbE\Big(\exp\big[\lambda(f(X)-\bbE f(X))\big]\Big)\leq \frac{\lambda^{2}\cdot\|c\|_{2}^{2}\cdot t_{\rm min}}{8}.
    \end{align*}
    For any $t\geq 0$, we have
    \begin{align*}
        P\Big(\big|f(X)-\bbE f(X)\big|\geq t\Big)\leq 2\exp\bigg(\frac{-2t^{2}}{\|c\|_{2}^{2}\cdot t_{\rm min}}\bigg).
    \end{align*}
\end{lemma}

\begin{lemma}[Lemma 25 in \cite{agarwal2020flambe}]\label{lem:tvbound}
        For any two conditional probability densities $P(\cdot\,|\,X), P^{\prime}(\cdot\,|\,X)$ and any distribution $\nu\in\Delta(\calX)$,we have
        \begin{align}
            \bbE_{\nu}\Big[\TV\big(P(\cdot\,|\,X), P^{\prime}(\cdot\,|\,X)\big)^{2}\Big]\!\leq\! -2\log\bigg(\bbE_{X\sim\nu,Y\sim P(\cdot\,|\,X)}\bigg[\exp\bigg(\!-\!\frac{1}{2}\log\frac{P(Y\,|\,X)}{P^{\prime}(Y\,|\,X)}\bigg)\bigg]\bigg).\nonumber
        \end{align}
    \end{lemma}
    
\begin{lemma}[Corollary A.7 in \cite{edelman2021inductive} ]\label{lem:smlip}
    For any $x,y\in\bbR^{d}$, we have
    \begin{align}
        \|\softmax(x)-\softmax(y)\|_{1}\leq 2\|x-y\|_{\infty}.\nonumber
    \end{align}
\end{lemma}

\begin{lemma}[Lemma 17 in \cite{zhang2022relational} ]\label{lem:matvec}
    Given any two conjugate numbers $u,v\in [1,\infty]$, i.e., $\frac{1}{u}+\frac{1}{v}=1$, and $1\leq p\leq \infty$, for any $A\in\bbR^{r\times c}$ and $x\in \bbR^{c}$, we have
    \begin{align}
        \|Ax\|_{p}\leq \|A\|_{p,u}\|x\|_{v}\quad\mbox{and}\quad  \|Ax\|_{p}\leq \|A^{\top}\|_{u,p}\|x\|_{v}\nonumber.
    \end{align}
\end{lemma}

\begin{lemma}[Propositions 20 and 21 in \cite{zhang2022relational}]\label{lem:mhalip}
    For any $X,\tilde{X}\in\bbR^{L\times d}$, and any $W_{i}^{Q},\tilW_{i}^{Q},W_{i}^{K},\tilW_{i}^{K}\in\bbR^{d\times d_{h}},W_{i}^{V},\tilW_{i}^{V}\in\bbR^{d\times d}$ for $i\in [h]$ , if $\|X^{\top}\|_{p,\infty},\|\tilde{X}^{\top}\|_{2,\infty}\leq B_{X}$, $\|W_{i}^{Q}\|_{\rmF},\|\tilW_{i}^{Q}\|_{\rmF}\leq B_{Q}$, $\|W_{i}^{K}\|_{\rmF},\|\tilW_{i}^{K}\|_{\rmF}\leq B_{K}$, $\|W_{i}^{V}\|_{\rmF},\|\tilW_{i}^{V}\|_{\rmF}\leq B_{V}$ for $i\in[h]$, then we have 
    \begin{align*}
        &\Big\|\big(\MHA(X,W)-\MHA(\tilde{X},\tilW)\big)^{\top}\Big\|_{2,\infty}\nonumber\\
        &\quad\leq h\cdot B_{V}\big(1+4B_{X}^{2}\cdot B_{Q}B_{K}\big)\|X^{\top}-\tilde{X}^{\top}\|_{2,\infty}+B_{X}\sum_{i=1}^{h}\big\|W_{i}^{V}-\tilde{W}_{i}^{V}\|_{\rmF}\\
        &\quad\qquad+2B_{X}^{3}\cdot B_{V}\cdot B_{K}\sum_{i=1}^{h}\|W_{i}^{Q}-\tilde{W}_{i}^{Q}\|_{\rmF}+2B_{X}^{3}\cdot B_{V}\cdot B_{Q}\sum_{i=1}^{h}\|W_{i}^{K}-\tilde{W}_{i}^{K}\|_{\rmF}.
    \end{align*}
\end{lemma}

\begin{lemma}[Lemma A.6 in \cite{elbrachter2021deep}]\label{lem:approx}
    For $a,b\in\bbR$ with $a<b$, let
    \begin{align*}
        \calS_{[a,b]}=\Big\{f\in\calS^{\infty}([a,b],\bbR)\,|\, \big\|f^{(n)}(x)\big\|\leq n! \text{ for all }n\in\bbN\Big\}.
    \end{align*}
    There exists a constant $C>0$ such that for all $a,b\in\bbR$ with $a<b$, $f\in\calS_{[a,b]}$, and $\varepsilon\in (0,1/2)$, there is a fully connect network $\Psi_{f}$ such that
    \begin{align*}
        \|f-\Psi_{f}\|_{\infty}\leq\varepsilon,
    \end{align*}
    with the depth of the network as $D(\Psi_{f})\leq C\max\{2,b-a\}(\log \varepsilon^{-1})^{2}+\log(\lceil\max\{|a|,|b|\}\rceil)+\log(\lceil 1/(b-a)\rceil)$, the width of the network as $W(\Psi_{f})\leq 16$, and the maximal weight in the network as $B(\Psi_{f})\leq 1$.
    
\end{lemma}

\begin{lemma}
    \label{lem:tv-kl}
Let $b = \sup_x \log (p(x) / q(x))$.
We have that
\begin{align}\label{eq:tv-kl0}
    \KL(p\,\|\, q)  \le 2(3 + b) \cdot \TV(p, q).
\end{align}
\end{lemma}
\begin{proof}
    We let $f(t) = \log t$ and $g(t) = |1/t - 1|$. Then, for $ 0 \le t \le \exp(b)$, we have that
    \begin{align*}
        \sup _{ 0\le t \le \exp(b)}\frac{f(t)}{g(t)} = \sup _{ 0\le t \le \exp(b)} \frac{\log t}{|1/t - 1 |} = \sup _{ 1\le t \le \exp(b)} \frac{t\log t}{t - 1} \le 2(b+3).
    \end{align*}
    Note that $\KL(p\,\|\, q) = \EE_p[f(p(x)/q(x))]$ and $\TV(p, q) = \EE_p[g(p(x)/q(x))]$, which concludes the proof.
\end{proof}
\end{document}